\RequirePackage[l2tabu,orthodox]{nag}
\documentclass
[letterpaper,12pt,]
{article}

\usepackage{etex}
\usepackage{verbatim}
\usepackage{xspace,enumerate}
\usepackage[dvipsnames]{xcolor}
\usepackage[T1]{fontenc}
\usepackage[full]{textcomp}
\usepackage[american]{babel}
\usepackage{mathtools}
\usepackage{amsthm}
\usepackage[
letterpaper,
top=1in,
bottom=1in,
left=1in,
right=1in]{geometry}
\usepackage{newpxtext} 
\usepackage{textcomp} 
\usepackage[varg,bigdelims]{newpxmath}
\usepackage[scr=rsfso]{mathalfa}
\usepackage{bm} 
\let\mathbb\varmathbb
\usepackage{microtype}
\usepackage[pagebackref,colorlinks=true,urlcolor=blue,linkcolor=blue,citecolor=OliveGreen]{hyperref}
\usepackage[capitalise,nameinlink]{cleveref}
\crefname{lemma}{Lemma}{Lemmas}
\crefname{fact}{Fact}{Facts}
\crefname{theorem}{Theorem}{Theorems}
\crefname{corollary}{Corollary}{Corollaries}
\crefname{claim}{Claim}{Claims}
\crefname{example}{Example}{Examples}
\crefname{algorithm}{Algorithm}{Algorithms}
\crefname{problem}{Problem}{Problems}
\crefname{definition}{Definition}{Definitions}
\crefname{exercise}{Exercise}{Exercises}
\usepackage{amsthm}

\newtheorem{theorem}{Theorem}[section]
\newtheorem*{theorem*}{Theorem}
\newtheorem{lemma}[theorem]{Lemma}
\newtheorem*{lemma*}{Lemma}
\newtheorem{fact}[theorem]{Fact}
\newtheorem*{fact*}{Fact}

\newtheorem*{proposition*}{Proposition}
\newtheorem{corollary}[theorem]{Corollary}
\newtheorem*{corollary*}{Corollary}

\newtheorem*{hypothesis*}{Hypothesis}

\newtheorem*{conjecture*}{Conjecture}
\theoremstyle{definition}
\newtheorem{definition}[theorem]{Definition}
\newtheorem*{definition*}{Definition}

\newtheorem*{construction*}{Construction}

\newtheorem*{example*}{Example}

\newtheorem*{question*}{Question}
\newtheorem{assumption}[theorem]{Assumption}
\newtheorem*{assumption*}{Assumption}
\newtheorem{problem}[theorem]{Problem}
\newtheorem*{problem*}{Problem}

\newtheorem*{openquestion*}{Open Question}

\newtheorem*{property*}{Property}

\theoremstyle{remark}

\newtheorem*{claim*}{Claim}
\newtheorem{remark}[theorem]{Remark}
\newtheorem*{remark*}{Remark}

\newtheorem*{observation*}{Observation}

\usepackage{paralist}
\let\originalleft\left
\let\originalright\right
\renewcommand{\left}{\mathopen{}\mathclose\bgroup\originalleft}
\renewcommand{\right}{\aftergroup\egroup\originalright}
\usepackage{turnstile}
\usepackage{mdframed}
\usepackage{tikz}
\usepackage{caption}
\DeclareCaptionType{Algorithm}
\usepackage{newfloat}
\usepackage{algorithm}
\usepackage{algorithmic}
\usepackage{xparse}
\usepackage{amsthm} 
\makeatletter
\let\latexparagraph\paragraph
\RenewDocumentCommand{\paragraph}{som}{%
	\IfBooleanTF{#1}
	{\latexparagraph*{#3}}
	{\IfNoValueTF{#2}
		{\latexparagraph{\maybe@addperiod{#3}}}
		{\latexparagraph[#2]{\maybe@addperiod{#3}}}%
	}%
}
\newcommand{\maybe@addperiod}[1]{%
	#1\@addpunct{.}%
}
\makeatother

\newenvironment{algorithmbox}{\begin{mdframed}[nobreak=true]
		\begin{algorithm}}{\end{algorithm}\end{mdframed}}

\newcommand{\Authornote}[2]{}
\newcommand{\Authornotecolored}[3]{}
\newcommand{\Authorcomment}[2]{}
\newcommand{\Authorfnote}[2]{}
\newcommand{\Tnote}{\Authornote{T}}

\newcommand{\Dnote}{\Authornote{D}}

\usepackage{boxedminipage}
\newcommand{\paren}[1]{(#1)}
\newcommand{\Paren}[1]{\left(#1\right)}

\newcommand{\brac}[1]{[#1]}
\newcommand{\Brac}[1]{\left[#1\right]}

\newcommand{\bracbb}[1]{\llbracket#1\rrbracket}
\newcommand{\Bracbb}[1]{\left\llbracket#1\right\rrbracket}

\newcommand{\abs}[1]{\lvert#1\rvert}
\newcommand{\Abs}[1]{\left\lvert#1\right\rvert}

\newcommand{\card}[1]{\lvert#1\rvert}
\newcommand{\Card}[1]{\left\lvert#1\right\rvert}

\newcommand{\set}[1]{\{#1\}}
\newcommand{\Set}[1]{\left\{#1\right\}}

\newcommand{\norm}[1]{\lVert#1\rVert}
\newcommand{\Norm}[1]{\left\lVert#1\right\rVert}

\newcommand{\snorm}[1]{\norm{#1}^2}
\newcommand{\Snorm}[1]{\Norm{#1}^2}

\newcommand{\normo}[1]{\norm{#1}_1}
\newcommand{\Normo}[1]{\Norm{#1}_1}

\newcommand{\Normi}[1]{\Norm{#1}_\infty}

\newcommand{\iprod}[1]{\langle#1\rangle}
\newcommand{\Iprod}[1]{\left\langle#1\right\rangle}

\newcommand{\Esymb}{\mathbb{E}}
\newcommand{\Psymb}{\mathbb{P}}

\DeclareMathOperator*{\E}{\Esymb}

\DeclareMathOperator*{\ProbOp}{\Psymb}
\renewcommand{\Pr}{\ProbOp}
\newcommand{\given}{\mathrel{}\middle\vert\mathrel{}}
\newcommand{\suchthat}{\;\middle\vert\;}

\renewcommand{\ij}{{ij}}

\newcommand{\vbig}{\vphantom{\bigoplus}}

\newcommand{\defeq}{\stackrel{\mathrm{def}}=}
\newcommand{\seteq}{\mathrel{\mathop:}=}
\newcommand{\from}{\colon}

\newcommand\bdot\bullet

\DeclareMathOperator{\Ind}{\mathbf 1}

\DeclareMathOperator{\argmin}{argmin}

\DeclareMathOperator{\supp}{supp}

\DeclareMathOperator{\sign}{sign}

\newcommand{\N}{\mathbb N}
\newcommand{\R}{\mathbb R}

\newcommand{\Q}{\mathbb Q}

\newcommand{\cA}{\mathcal A}
\newcommand{\cB}{\mathcal B}
\newcommand{\cC}{\mathcal C}

\newcommand{\cE}{\mathcal E}

\newcommand{\cK}{\mathcal K}

\newcommand{\cM}{\mathcal M}
\newcommand{\cN}{\mathcal N}
\newcommand{\cO}{\mathcal O}
\newcommand{\cP}{\mathcal P}

\newcommand{\cR}{\mathcal R}
\newcommand{\cS}{\mathcal S}
\newcommand{\cT}{\mathcal T}
\newcommand{\cU}{\mathcal U}

\newcommand{\cX}{\mathcal X}

\newcommand{\cZ}{\mathcal Z}

\newcommand{\bbP}{\mathbb P}

\renewcommand{\leq}{\leqslant}
\renewcommand{\le}{\leqslant}
\renewcommand{\geq}{\geqslant}
\renewcommand{\ge}{\geqslant}
\let\epsilon=\varepsilon
\numberwithin{equation}{section}
\newcommand\MYcurrentlabel{xxx}
\newcommand{\MYstore}[2]{%
	\global\expandafter \def \csname MYMEMORY #1 \endcsname{#2}%
}
\newcommand{\MYload}[1]{%
	\csname MYMEMORY #1 \endcsname%
}
\newcommand{\MYnewlabel}[1]{%
	\renewcommand\MYcurrentlabel{#1}%
	\MYoldlabel{#1}%
}
\newcommand{\MYdummylabel}[1]{}
\newcommand{\torestate}[1]{%
	\let\MYoldlabel\label%
	\let\label\MYnewlabel%
	#1%
	\MYstore{\MYcurrentlabel}{#1}%
	\let\label\MYoldlabel%
}
\newcommand{\restatetheorem}[1]{%
	\let\MYoldlabel\label
	\let\label\MYdummylabel
	\begin{theorem*}[Restatement of \cref{#1}]
		\MYload{#1}
	\end{theorem*}
	\let\label\MYoldlabel
}
\newcommand{\restatelemma}[1]{%
	\let\MYoldlabel\label
	\let\label\MYdummylabel
	\begin{lemma*}[Restatement of \cref{#1}]
		\MYload{#1}
	\end{lemma*}
	\let\label\MYoldlabel
}
\newcommand{\restateprop}[1]{%
	\let\MYoldlabel\label
	\let\label\MYdummylabel
	\begin{proposition*}[Restatement of \cref{#1}]
		\MYload{#1}
	\end{proposition*}
	\let\label\MYoldlabel
}
\newcommand{\restatefact}[1]{%
	\let\MYoldlabel\label
	\let\label\MYdummylabel
	\begin{fact*}[Restatement of \cref{#1}]
		\MYload{#1}
	\end{fact*}
	\let\label\MYoldlabel
}
\newcommand{\restate}[1]{%
	\let\MYoldlabel\label
	\let\label\MYdummylabel
	\MYload{#1}
	\let\label\MYoldlabel
}

\newcommand{\e}{\epsilon}
\newcommand{\eps}{\epsilon}

\allowdisplaybreaks
\newcommand*{\Id}{\mathrm{Id}}
\newcommand*{\Lowner}{L\"owner\xspace}

\newcommand*{\transpose}[1]{{#1}{}^{\mkern-1.5mu\mathsf{T}}}
\newcommand*{\dyad}[1]{#1#1{}^{\mkern-1.5mu\mathsf{T}}}

\providecommand{\todo}{{\color{red}{\textbf{TODO }}}}

\providecommand{\toexpand}{{\color{blue}{\textbf{TOEXPAND }}}}

\newcommand{\ind}[1]{\mathbf{1}_{\Brac{#1}}}

\newcommand{\betastar}{\beta^*}
\newcommand{\betahat}{\hat{\beta}}

\newcommand{\kldiv}[2]{D_{KL}\Paren{#1 \| #2}}

\newif\ificml

\icmlfalse

\title{Consistent regression when oblivious outliers overwhelm\thanks{
    This project has received funding from the European Research Council (ERC) under the European Union's Horizon 2020 research and innovation programme (grant agreement No 815464).
  }
}
\author{
   Tommaso d'Orsi\thanks{ETH Z\"urich.}
	 \and
	 Gleb Novikov\thanks{ETH Z\"urich.}
	 \and
	 David Steurer\thanks{ETH Z\"urich.}
}

\date{}

\begin{document}


\maketitle

\begin{abstract}


We consider a robust linear regression model \(y=X\beta^* + \eta\), where an adversary oblivious to the design \(X\in \R^{n\times d}\) may choose \(\eta\) to corrupt all but an \(\alpha\) fraction of the observations \(y\) in an arbitrary way.
Prior to our work, even for Gaussian \(X\), no estimator for \(\beta^*\) was known to be consistent in this model except for quadratic sample size \(n \gtrsim (d/\alpha)^2\) or for logarithmic inlier fraction \(\alpha\ge 1/\log n\).
We show that consistent estimation is possible with nearly linear sample size and inverse-polynomial inlier fraction.
Concretely, we show that the Huber loss estimator is consistent for every sample size \(n= \omega(d/\alpha^2)\) 
and achieves an error rate of $O(d/\alpha^2n)^{1/2}$.
Both bounds are optimal (up to constant factors).
Our results extend to designs far beyond the Gaussian case and only require the column span of \(X\) to not contain approximately sparse vectors 
(similar to the kind of assumption commonly made about the kernel space for compressed sensing).
We provide two technically similar proofs.
One proof is phrased in terms of strong convexity, extending work of \cite{tsakonas2014convergence}, and particularly short.
The other proof highlights a connection between the Huber loss estimator and high-dimensional median computations. 
In the special case of Gaussian designs, this connection leads us to a strikingly simple algorithm based on computing coordinate-wise medians that achieves optimal guarantees in nearly-linear time, and that can exploit sparsity of $\beta^*$.
The model studied here also captures heavy-tailed noise distributions that may not even have a first moment.

\end{abstract}

\clearpage

\microtypesetup{protrusion=false}
\tableofcontents{}
\microtypesetup{protrusion=true}

\clearpage


\section{Introduction}\label{sec:introduction}
Linear regression is a fundamental task in statistics:
given observations $(x_1,y_1),\ldots,(x_n,y_n)\in\R^{d+1}$ following a linear model \(y_i = \iprod{x_i,\betastar}+\eta_i\), where \(\betastar\in \R^d\) is the unknown parameter of interest and \(\eta_1,\ldots,\eta_n\) is noise, the goal is to recover \(\betastar\) as accurately as possible. 

In the most basic setting, the noise values are drawn independently from a Gaussian distribution with mean $0$ and variance $\sigma^2$.
Here, the classical least-squares estimator \(\bm {\hat\beta}\) achieves an optimal error bound $\frac{1}{n}\norm{X\paren{\betastar-\bm {\hat\beta}}}^2\lesssim \sigma^2\cdot d/n$ with high probability, where the design \(X\) is has rows \(x_1,\ldots,x_n\).
Unfortuantely, this guarantee is fragile and the estimator may experience arbitrarily large error in the presence of a small number of benign outlier noise values.

In many modern applications, including economics \cite{rousseeuw2005robust}, image recognition \cite{wright2008robust}, and sensor networks \cite{haupt2008compressed}, there is a desire to cope with such outliers stemming from extreme events, gross errors, skewed and corrupted measurements.
It is therefore paramount to design estimators robust to noise distributions that may have substantial probability mass on outlier values.

In this paper, we aim to identify the weakest possible assumptions on the noise distribution such that for a wide range of measurement matrices $X$, we can efficiently recover the parameter vector $\betastar$ with vanishing error.

The design of learning algorithms capable of succeeding on data sets contaminated by adversarial noise has been a central topic in robust statistics (e.g. see \cite{diakonikolas2019robust,charikar2017learning} and their follow-ups for some recent developments).
In the context of regression with adaptive adversarial outliers (i.e. depending on the instance) several results are known \cite{taoRIP2005,taoRIP, colt/KlivansKM18,  conf/nips/KarmalkarKK19, caramanis1, caramanis2, karmalkar2018compressed, paristech, yau20}.
However, it turns out that for adaptive adversaries, vanishing error bounds are only possible if the fraction of outliers is vanishing.

In order to make vanishing error possible in the presence of large fractions of outliers, we consider weaker adversary models that are oblivious to the design \(X\).
Different assumptions can be used to model oblivious adversarial corruptions.
\cite{sun2019adaptive} assume  the noise distribution satisfies $\E\Brac{\bm \eta_i\given x_i}=0$ and $ \E\Brac{\Abs{\bm \eta_i}^{1+\delta}}< \infty$ for some $0\leq \delta\leq 1$, and show that if $X$ has constant condition number, then (a modification of) the Huber loss estimator \cite{huber1964} is consistent for\footnote{We hide absolute constant multiplicative factors using the standard notations $\lesssim, O(\cdot)$. Similarly, we hide multiplicative factors at most logarithmic in $n$ using the notation $\tilde{O}$. 
}  $n\geq \tilde{O}((\Normi{X}\cdot d)^{(1+\delta)/2\delta})$ (an estimator is consistent if the error tends to zero as the number of observation grows, $\tfrac 1 n\norm{X\paren{ {\hat{\bm\beta}}-\betastar}}^2\rightarrow 0$).

Without constraint on moments, a useful model is that of assuming the noise vector $\eta\in \R^n$ to be an arbitrary fixed vector  with $\alpha\cdot n$ coordinates  bounded by $1$ in absolute value. 
This model\footnote{
  Several models for consistent robust linear regression appear in the literature.
  Our model is strong enough to subsume the models we are aware of.
  See \cref{sec:error-convergence-model-assumptions} for a detailed comparison.
} also captures random vectors $\bm \eta\in \R^n$ independent of the measurement matrix $X$ and conveniently allows us to think of the $\alpha$ fraction of samples with small noise as the set of uncorrupted samples.  
In these settings, the problem has been mostly studied in the context of Gaussian design $\bm x_1,\ldots,\bm x_n\sim N(0,\Sigma)$. 
\cite{bathia_crr} provided an estimator achieving error $\tilde{O}(d/(\alpha^2\cdot n))$ for any $\alpha$ larger than some fixed constant. 
This result was then extended in \cite{SuggalaBR019}, where the authors proposed a near-linear time algorithm computing a $\tilde{O}(d/(\alpha^2\cdot n))$-close estimate for any\footnote{More precisely, their condition is $\alpha \gtrsim \tfrac{1}{\log n}$ for consistent estimation and $\alpha \gtrsim \tfrac{1}{\log \log n}$ to get the error bound $\tilde{O}(\tfrac{d}{\alpha^2 n})$ .}  $\alpha\gtrsim 1/\log\log n$. 
That is, allowing the number of uncorrupted samples to be $o(n)$. 
Considering even smaller fractions of inliers,  \cite{tsakonas2014convergence} showed that with high probability the Huber loss estimator is consistent for $n\geq \tilde{O}(d^2/\alpha^2)$, thus requiring sample size quadratic in the ambient dimension.


Prior to this work, little was known for more general settings  when the design matrix $X$ is non-Gaussian.
From an asymptotic viewpoint, i.e., when $d$ and $\alpha$ are fixed and $n\to \infty$, a similar model was studied 30 years ago in a seminal work by Pollard \cite{pollard}, albeit under stronger assumptions on the noise vector.
Under mild constraints on $X$, it was shown that the least absolute deviation (LAD) estimator is consistent.

So, the outlined state-of-the-art provides an incomplete picture of the statistical and computational complexity of the problem. The question of what conditions we need to enforce on the  measurement matrix $X$ and the noise vector $\bm \eta$ in order to  efficiently and consistently recover $\betastar$ remains largely unanswered. 
In high-dimensional settings, no estimator has been shown to be consistent 
when the fraction of uncontaminated samples $\alpha$ is smaller than $1/\log n$ and  the number of samples $n$ is smaller than $d^2/\alpha^2$, even in the simple settings of spherical Gaussian design. Furthermore, even less is known on how we can regress consistently when the design matrix is non-Gaussian.

In this work, we provide a more comprehensive picture of the problem. Concretely, we analyze the Huber loss estimator in  non-asymptotic, high dimensional setting where the fraction of inliers may depend  (even polynomially) on the number of samples and ambient dimension. Under \textit{mild} assumptions on the design matrix and the noise vector, we show that such algorithm achieves \textit{optimal} error guarantees and sample complexity.

Furthermore, a by-product of our analysis is an strikingly simple linear-time estimator based on computing coordinate-wise medians, that achieves nearly optimal guarantees  for standard Gaussian design, even in the regime where the parameter vector $\betastar$ is $k$-sparse (i.e. $\betastar$ has at most $k$ nonzero entries).

\subsection{Results about Huber-loss estimator}\label{sec:results}

We provide here guarantees on the error convergence of the Huber-loss estimator, 
defined as a minimizer of the \emph{Huber loss} \(f\from\R^d\to\R_{\ge 0}\),
\[
f(\beta)=\tfrac 1n \sum_{i=1}^n \Phi[(X\beta -y)_i]\,,
\]
where \(\Phi\from \R \to \R_{\ge 0}\) is the \emph{Huber penalty},\footnote{Here, we choose $2$ as transition point between quadratic and linear penalty for simplicity.
  See \cref{sec:tightness-noise-assumptions} for a discussion about different choices for this transition point.}
\[
\Phi[t]\defeq
\begin{cases}
	\tfrac 12 t^2 & \text{if }\abs{t}\le 2\,,\\
	2\abs{t}-2 & \text{otherwise.}
\end{cases}
\]

\paragraph{Gaussian design}

The following theorem states our the Huber-loss estimator in the case of Gaussian designs.
Previous quantitative guarantees for consistent robust linear regression focus on this setting \cite{tsakonas2014convergence,bathia_crr,SuggalaBR019}.

\begin{theorem}[Guarantees for Huber-loss estimator with Gaussian design]
	\label{thm:huber-loss-gaussian-results}
	Let \(\eta\in\R^n\) be a deterministic vector.
  Let \(\bm X\) be a random\footnote{As a convention, we use boldface to denote random variables.} \(n\)-by-\(d\) matrix with iid standard Gaussian entries \(\bm X_{ij}\sim N(0,1)\).
  
  Suppose \(n\ge  C \cdot d/\alpha^2\), where \(\alpha\) is the fraction of entries in \(\eta\) of magnitude at most \(1\), and $C > 0$ is large enough absolute constant.
  
	Then, with probability at least \(1-2^{-d}\) over $\bm X$, for every $\betastar \in \R^d$, given $\bm X$ and \(\bm y=\bm X \betastar + \eta\), the Huber-loss estimator $\bm\betahat$ satisfies
	\begin{align*}
    \Norm{\betastar - \bm{\hat{\beta}}}^2
    \le O\Paren{\frac{d}{\alpha^2 n}}\,.
	\end{align*}
\end{theorem}

The above result improves over previous quantitative analyses of the Huber-loss estimator that require quadratic sample size \(n \gtrsim d^2/\alpha^2\) to be consistent \cite{tsakonas2014convergence}.
Other estimators developed for this model \cite{DBLP:conf/nips/Bhatia0KK17,SuggalaBR019} achieve a sample-size bound nearly-linear in \(d\) at the cost of an exponential dependence on \(1/\alpha\).
These results require for consistent estimation a logarithmic bound on the inlier fraction \(\alpha\gtrsim 1/\log d\) to achieve sample-size bound nearly-linear in \(d\) .
In constrast our sample-size bound is nearly-linear in \(d\) even for any sub-polynomial inlier fraction \(\alpha=1/d^{o(1)}\).
In fact, our sample-size bound and estimation-error bound is statistically optimal up to constant factors.\footnote{
  In the case that \(\eta\sim N(0,\sigma^2\cdot\Id)\), it's well known that the optimal Bayesian estimator achieves expected error \(\sigma^2\cdot d/n\).
  For \(\sigma\ge 1\), the vector \(\eta\) has a \(\Theta(1/\sigma)\) fraction of entries of magnitude at most \(1\) with high probability.
}

The proof of the above theorem also applies to approximate minimizers of the Huber loss and it shows that such approximations can be computed in polynomial time.


We remark that related to (one of) our analyses of the Huber-loss estimator, we develop a fast algorithm based on (one-dimensional) median computations that achieves estimation guarantees comparable to the ones above but in linear time \(O(nd)\).
A drawback of this fast algorithm is that its guarantees depend (mildly) on the norm of \(\betastar\).

Several results \cite{taoRIP2005, taoRIP, karmalkar2018compressed,DBLP:conf/soda/DiakonikolasKS19, paristech} considered  settings where the noise vector is adaptively chosen by an adversary.
In this setting, it is possible to obtain a unique estimate only if the fraction of outliers is smaller than \(1/2\).
In contrast, \cref{thm:huber-loss-gaussian-results} implies consistency even when the fraction of corruptions tends to $1$ but applies to settings where the noise vector $\eta$ is fixed \textit{before} sampling $\bm X$ and thus it is oblivious to the data.

\paragraph{Deterministic design}
The previous theorem makes the strong assumption that the design is Gaussian.
However, it turns out that our proof extends to a much broader class of designs with the property that their columns spans are well-spread (in the sense that they don't contain vectors whose \(\ell_2\)-mass is concentrated on a small number of coordinates, see \cite{DBLP:conf/approx/GuruswamiLW08}).
In order to formulate this more general results it is convenient to move the randomness from the design to the noise vector and consider deterministic designs \(X\in\R^{n\times d}\) with probabilistic \(n\)-dimensional noise vector \(\bm\eta\),
\begin{equation}
  \label{eq:deterministc-design-model}
  \bm y = X \betastar + \bm\eta\,.
\end{equation}
Here, we assume that \(\bm \eta\) has independent, symmetrically distributed entries satisfying \(\Pr\set{\abs{\bm\eta_i}\le 1}\ge \alpha\) for all $i\in [n]$.

This model turns out to generalize the one considered in the previous theorem.
Indeed, given data following the previous model with Gaussian design and deterministic noise, 
we can generate data following the above model randomly subsampling the given data and 
multiplying with random signs (see \cref{sec:error-convergence-model-assumptions} for more details).

\begin{theorem}[Guarantees for Huber-loss estimator with general design]
  \label{thm:huber-loss-informal}
  Let $X\in \R^{n\times d}$ be a deterministic matrix and let $\bm \eta$ be an $n$-dimensional random vector with independent, symmetrically distributed entries and $\alpha=\min_{i\in [n]}\Pr\set{\abs{\bm \eta_i}\leq 1}$.
	
	
  Suppose that for every vector \(v\) in the column span of \(X\) and every subset \(S\subseteq [n]\) with \(\card{S}\le C\cdot d/\alpha^2\),
  \begin{align}\label{eq:rip-property}
    \norm{v_S} \le 0.9 \cdot \norm{v}\,,
  \end{align}
  where \(v_S\) denotes the restriction of \(v\) to the coordinates in~\(S\), and $C > 0$ is large enough absolute constant.
	
	
	Then, with probability at least \(1-2^{-d}\) over \(\bm \eta\),
  for every \(\betastar \in\R^d\), given $ X$ and \(\bm y= X \betastar + \bm \eta\), 
  the Huber-loss estimator $\bm\betahat$ satisfies
	\begin{align*}
		\tfrac 1 n \Norm{X(\betastar - \bm{\hat{\beta}})}^2
		\le O\Paren{\frac{d}{\alpha^2 n}}\,.
	\end{align*}
	
	
\end{theorem}

In particular, \cref{thm:huber-loss-informal} implies that under condition \cref{eq:rip-property} and mild noise assumptions,
the Huber loss estimator is consistent   for $n \ge \omega\Paren{d/\alpha^2}$.



As we only assume the  column span of $X$ to be well-spread,
the result applies to  a substantially broader class of design matrices $X$ than Gaussian, 
naturally including those studied in 
\cite{tsakonas2014convergence, bathia_crr, SuggalaBR019}. 
Spread subspaces are related to $\ell_1$-vs-$\ell_2$ distortion\footnote{Our analysis of \cref{thm:huber-loss-informal} also applies to design matrices whose column span has bounded distortion, see \cref{sec:error-convergence-model-assumptions} for more details.},
and have some resemblance with restricted isometry properties (RIP).
Indeed both RIP and distortion assumptions have been successfully used in compressed sensing \cite{taoRIP2005,taoRIP, kashin2007remark, donoho2006compressed}
but, to the best of our knowledge, they were never observed to play a fundamental role in the context of robust linear regression. This is a key difference between our analysis and that of previous works. 
Understanding how crucial this well-spread property is and how to leverage it allows us to simultaneously obtain nearly optimal error guarantees, while also relaxing the design matrix assumptions. 
It is important to remark that a weaker version of property \cref{eq:rip-property} is necessary as otherwise it may be \textit{information theoretically impossible} to solve the problem (see \cref{lem:lower-bound}).

We derive both \cref{thm:huber-loss-gaussian-results} and \cref{thm:huber-loss-informal} using the same proof techniques explained in \cref{sec:techniques}.

\begin{remark}[Small failure probability]\label{remark:error-probability}
	For both \cref{thm:huber-loss-gaussian-results} and \cref{thm:huber-loss-informal}  our proof also gives that for any $\delta \in (0,1)$,  the Huber loss estimator achieves error $O\Paren{\frac{d+\log(1/\delta)}{\alpha^2 n}}$ with probability at least $1-\delta$ as long as $n \gtrsim  \frac{d+\ln(1/\delta)}{\alpha^2}$, and, in \cref{thm:huber-loss-informal}, 
	the well-spread property is satisfied for all sets $S\subseteq [n]$ of size 
	$\Card{S}\le O\Paren{\frac{d+\log(1/\delta)}{\alpha^2}}$.
\end{remark}

\subsection{Results about fast algorithms}\label{sec:results-fast-algorithms}

The Huber loss estimator has been extensively applied to robust regression problems \cite{tanSuWitten, tsakonas2014convergence, vandeGeer}. 
However, one possible  drawback of such algorithm (as well as other standard approaches such as  $L_1$-minimization \cite{pollard, karmalkar2018compressed, tractran})
is the non-linear running time. In real-world applications with large, high dimensional datasets, an algorithm running in linear time $O(nd)$ may make the difference between feasible and unfeasible. 

In the special case of Gaussian design, previous results \cite{SuggalaBR019} already obtained estimators computable in linear time. However these algorithms require a logarithmic bound on the fraction of inliers $\alpha\gtrsim 1/\log n$. We present here a strikingly simple algorithm that achieves similar guarantees as the ones shown in \cref{thm:huber-loss-gaussian-results} and runs in \textit{linear time}:  for each coordinate $j \in [d]$ compute the  median 
$\bm\betahat_j$ of $\bm y_1/\bm X_{1j},\ldots,\bm y_n/\bm X_{nj}$ 
subtract the resulting estimation $\bm X\bm\betahat$ and repeat, 
logarithmically many times,  with fresh samples.

\begin{theorem}[Guarantees for fast estimator with Gaussian design]
	\label{thm:median-algorithm-informal}
	Let \(\eta\in\R^n\) and \(\beta^*\in\R^d\) be deterministic vectors.
	Let \(\bm X\) be a random \(n\)-by-\(d\) matrix with iid standard Gaussian entries 
	\(\bm X_{ij}\sim N(0,1)\).
	
	Let \(\alpha \) be the fraction of entries in \(\eta\) of magnitude at most \(1\), and let 
	$\Delta\ge  10 + \norm{\beta^*}$.
	Suppose  that 
	\[
	n\ge  C\cdot \frac{d}{\alpha^2}\cdot \ln\Delta \cdot (\ln d+\ln\ln\Delta)\,,
	\]
	where $C$ is a large enough absolute constant.
	
	Then, there exists an algorithm that given $\Delta$, $\bm X$ and \(\bm y=\bm X \betastar + \eta\) as input,  
	in time\footnote{By time we mean number of arithmetic operations and comparisons of entries of $\bm y$ and $\bm X$. 
		We do not
		take bit complexity into account.}  $O\Paren{nd}$ finds a vector $\bm {\hat{\beta}}\in \R^d$ such that
	\begin{align*}
		\Norm{\betastar - \bm{\hat{\beta}}}^2
		\le O\Paren{\frac{d}{\alpha^2 n}\cdot \log d}\,,
	\end{align*}
	with probability at least $1-d^{-10}$. 
\end{theorem}
The algorithm in \cref{thm:median-algorithm-informal}  requires knowledge of an upper bound $\Delta$ on the norm of the parameter vector.  The sample complexity of the estimator has logarithmic dependency on this upper bound. This phenomenon is a consequence of the iterative nature of the algorithm and also appears in other results \cite{SuggalaBR019}.

\cref{thm:median-algorithm-informal} also works for non-spherical settings Gaussian design matrix and provides nearly optimal error convergence with nearly optimal sample complexity, albeit with running time $\tilde{O}(nd^2)$. The algorithm doesn't require prior knowledge of the covariance matrix $\Sigma$.
In these settings, even though  time complexity is not linear in $d$, it is linear in $n$, 
and if $n$ is considerably larger than $d$, the algorithm may be very efficient.

\paragraph{Sparse linear regression} For spherical Gaussian design, the median-based algorithm introduced above can naturally be extended to the sparse settings, yielding the following theorem.
\begin{theorem}[Guarantees of fast estimator  for sparse regression with Gaussian design]
	\label{thm:median-sparse}
	Let \(\eta\in\R^n\) and \(\beta^*\in\R^d\) 
	be deterministic vectors, and assume that $\beta^*$ has at most $k\le d$ nonzero entries.
	Let \(\bm X\) be a random \(n\)-by-\(d\) matrix with iid standard Gaussian entries 
	\(\bm X_{ij}\sim N(0,1)\).
	
	Let \(\alpha \) be the fraction of entries in \(\eta\) of magnitude at most \(1\), and let 
	$\Delta\ge  10 + \norm{\beta^*}$.
	Suppose that 
	\[
	n\ge  C\cdot \frac{k}{\alpha^2}\cdot \ln \Delta\cdot(\ln d+\ln\ln \Delta)\,,
	\]
	where $C$ is a large enough absolute constant.
	
	Then, there exists an algorithm that given $k$, $\Delta$, $\bm X$ and \(\bm y=\bm X \betastar + \eta\) as input,   
	in time $O\Paren{nd}$ 	finds a vector $\bm {\hat{\beta}}\in \R^d$ such that
	\begin{align*}
		\Norm{\betastar - \bm{\hat{\beta}}}^2
		\le O\Paren{\frac{k}{\alpha^2 n}\cdot \log d}\,,
	\end{align*}
	with probability at least $1-d^{-10}$.
\end{theorem}



\section{Techniques}\label{sec:techniques}

Recall our linear regression model,
\begin{equation}
  \label{eq:linear-model-techniques}
  \bm y = X \betastar + \bm \eta\,,
\end{equation}
where we observe (a realization of) the random vector \(\bm y\), the matrix \(X\in\R^{n\times d}\) is a known design, the vector \(\betastar\in \R^n\) is the unknown parameter of interest, and the noise vector \(\bm \eta\) has independent, symmetrically distributed\footnote{The distributions of the coordinates are not known to the algorithm designer and can be non-identical.} coordinates with\footnote{The value of \(\alpha\) need not be known to the algorithm designer and only affects the error guarantees of the algorithms.} \(\alpha=\min_{i\in[n]}\Pr\{\abs{\bm \eta_i}\le 1\}\).

To simplify notation in our proofs, we assume \(\tfrac 1n \transpose X X=\Id\).
(For general \(X\), we can ensure this property by orthogonalizing and scaling the columns of \(X\).)

We consider the \emph{Huber loss estimator} \(\bm{\hat\beta}\), defined as a minimizer of the \emph{Huber loss} \(\bm f\),
\[
  \bm f(\beta)\seteq\tfrac 1n \sum_{i=1}^n \Phi[(X\beta -\bm y)_i]\,,
\]
where \(\Phi\from \R \to \R_{\ge 0}\) is the \emph{Huber penalty},\footnote{Here, in order to streamline the presentation, we choose \(\{\pm 2\}\) as the transition points between quadratic and linear penalty.
Changing these points to \(\set{\pm 2\delta}\) is achieved by scaling \(t\mapsto\delta^2\Phi(t/\delta)\).}
\[
  \Phi[t]=
  \begin{cases}
    \tfrac 12 t^2 & \text{if }\abs{t}\le 2\,,\\
    2\abs{t}-2 & \text{otherwise.}
  \end{cases}
\]

\subsection{Statistical guarantees from strong convexity}
\label{sec:techniques-strong-convexity}

In order to prove statistical guarantees for this estimator, we follow a well-known approach that applies to a wide range of estimators based on convex optimization (see \cite{DBLP:conf/nips/NegahbanRWY09} for a more general exposition), which also earlier analyses of the Huber loss estimator \cite{tsakonas2014convergence} employ.
This approach has two ingredients: (1) an upper bound on the norm of the gradient of the loss function \(\bm f\) at the desired parameter \(\betastar\) and (2) a lower bound on the strong-convexity curvature parameter of \(\bm f\) within a ball centered at \(\betastar\).
Taken together, these ingredients allow us to construct a global lower bound for \(\bm f\) that implies that all (approximate) minimizers of \(\bm f\) are close to \(\betastar\).
(See \cref{thm:error-bound-from-strong-convexity} for the formal statement.)

An important feature of this approach is that it only requires strong convexity to hold locally around \(\betastar\).
(Due to its linear parts, the Huber loss function doesn't satisfy strong convexity globally.)
It turns out that the radius of strong convexity we can prove is the main factor determining the strength of the statistical guarantee we obtain.
Indeed, the reason why previous analyses\footnote{We remark that the results in \cite{tsakonas2014convergence} are phrased asymptotically, i.e., fixed \(d\) and \(n\to \infty\).
  Therefore, a radius bound independent of \(n\) is enough for them.
However, their proof is quantiative and yields a radius bound of \(1/\sqrt d\) a we will discuss.} of the Huber loss estimator \cite{tsakonas2014convergence} require quadratic sample size \(n\gtrsim (d/\alpha)^2\) to ensure consistency is that they can establish strong convexity only within inverse-polynomial radius \(\Omega (1/\sqrt d)\) even for Gaussian \(X\sim N(0,1)^{n\times d}\).
In contrast, our analysis gives consistency for any super-linear sample size \(n= \omega(d/\alpha^2)\) for Gaussian \(X\) because we can establish strong convexity within constant radius.

Compared to the strong-convexity bound, which we disucss next, the gradient bound is straightforward to prove.
The gradient of the Huber loss at \(\betastar\) for response vector \(\bm y = X\betastar+\bm\eta\) takes the following form,
\[
\nabla\bm f (\betastar)=\tfrac 1n \sum_{i=1}^n \Phi'[\bm \eta_i] \cdot x_i\quad\text{with}\quad\Phi'[t]= \sign(t)\cdot \min\{\abs{t},2\}\,,
\]
where \(x_1,\ldots,x_n\in\R^d\) form the rows of \(X\).
Since \(\bm \eta_1,\ldots,\bm \eta_n\) are independent and symmetrically distributed, the random vectors in the above sum are independent and centered.
We can bound the expected gradient norm,
\[
  \E\Norm{\tfrac 1n \sum_{i=1}^n\Phi'[\bm \eta_i]}^2
  =\tfrac 1{n^2}\sum_{i=1}^n\E\Abs{\Phi'[\bm \eta_i]}\cdot \norm{x_i}^2
  \le \tfrac{2}{n^2}\sum_{i=1}^n\norm{x_i}^2
  = \frac{2d}{n}\,.
\]
The last step uses our assumption \(\tfrac 1n \transpose XX=\Id\).
Finally, Bernstein's inequality for random vectors (see e.g. \cite{DBLP:journals/tit/Gross11}) allows us to conclude that the gradient norm is close to its expectation with high probability (see the \cref{thm:gradient-bound} for details).

\paragraph{Proving local strong convexity for Huber loss}

For response vector \(\bm y=X\betastar+\bm\eta\) and arbitrary \(u\in\R^d\), the Hessian\footnote{The second derivative of the Huber penalty doesn't exit at the transition points \(\set{\pm 2}\) between its quadratic and linear parts.
  Nevertheless, the second derivative exists as an $L_1$-function in the sense that \(\Phi'[b]-\Phi'[a]=\int_{a}^b\bracbb{\abs{t}\le 2}\,\mathrm d t\) for all \(a,b\in\R\).
This property is enough for our purposes.} of the Huber loss at \(\betastar+u\) has the following form,
\[
  H\bm f(\betastar+u)=\tfrac 1n \sum_{i=1}^n \Phi''[(Xu)_i-\bm\eta_i]\cdot \dyad{x_i}\quad\text{with}\quad\Phi''[t]=\bracbb{\card{t}\le2}\,.
\]
Here, \(\bracbb{\cdot}\) is the Iverson bracket (0/1 indicator).
To prove strong convexity within radius \(R\), we are to lower bound the smallest eigenvalue of this random matrix uniformly over all vectors \(u\in\R^d\) with \(\norm{u}\le R\).

We do not attempt to exploit any cancellations between \(X u\) and \(\bm \eta\) and work with the following lower bound \(\bm M(u)\) for the Hessian,
\begin{equation}
  \label{eq:hessian-lower-bound}
  H\bm f(\betastar+u)\succeq \bm M(u)\seteq \tfrac 1n \sum_{i=1}^n \bracbb{\abs{\iprod{x_i,u}}\le 1}\cdot \bracbb{\abs{\bm \eta_i}\le 1}\cdot \dyad{x_i}\,.
\end{equation}
Here, \(\succeq\) denotes the \Lowner order.

It's instructive to first consider \(u=0\).
Here, the above lower bound for the Hessian satisfies,
\[
  \E \Brac{\bm M(0)}
  =\tfrac 1n\sum_{i=1}^n \Pr\{\abs{\bm\eta_i}\le 1\}\cdot \dyad{x_i}
  \succeq \alpha \Id\,.
\]
Using standard (matrix) concentration inequalities, we can also argue that this random matrix is close to its expectation with high-probability if \(n\ge \tilde O(d/\alpha)\) \Dnote{check this bound} under some mild assumption on \(X\) (e.g., that the row norms are balanced so that \(\norm{x_1},\ldots,\norm{x_n}\le O(\sqrt d)\)).

The main remaining challenge is dealing with the quantification over \(u\).
Earlier analyses \cite{tsakonas2014convergence} observe that the Hessian lower bound \(\bm M(\cdot)\) is constant over balls of small enough radius.
Concretely, for all \(u\in\R^d\) with \(\norm{u}\le 1/\max_i \norm{x_i}\), we have
\[
\bm M(u) = \bm M(0)\,,
\]
because \(\abs{\iprod{x_i,u}}\le \norm{x_i}\cdot \norm{u}\le 1\) by Cauchy-Schwarz.
Thus, strong convexity with curvature parameter \(\alpha\) within radius \(1/\max_i \norm{x_i}\) follows from the aforementioned concentration argument for \(\bm M(0)\).
However, since \(\max_i \norm{x_i}\ge \sqrt d\), this argument cannot give a better radius bound than \(1/\sqrt d\), which leads to a quadratic sample-size bound \(n \gtrsim d^2/\alpha^2\) as mentioned before.

For balls of larger radius, the lower bound \(\bm M(\cdot)\) can vary significantly.
For illustration, let us consider the case \(\bm \eta=0\) and let us denote the Hessian lower bound by \(M(\cdot)\) for this case.
(The deterministic choice of \(\bm \eta=0\) would satisfy all of our assumptions about \(\bm \eta\).)
As we will see, a uniform lower bound on the eigenvalues of \(M(\cdot)\) over a ball of radius \(R>0\) implies that the column span of \(X\) is well-spread in the sense that every vector \(v\) in this subspace has a constant fraction of its \(\ell_2\) on entries with squared magnitude at most a \(1/{ R^2}\) factor times the average squared entry of \(v\).
(Since we aim for \(R>0\) to be a small constant, the number \(1/R^2\) is a large constant.)
Concretely,
\begin{align}
  \notag
  \min_{\norm{u}\le R} \lambda_{\min} (M(u))
  &\le \min_{\norm{u}= R} \tfrac 1 {R^2} \iprod{u,M(u)u}\\
  \notag
  & =  \min_{\norm{u}= R} \tfrac 1 {R^2} \cdot \tfrac 1 n \sum_{i=1}^n\bracbb{\iprod{x_i,u}^2\le 1}\cdot \iprod{x_i,u}^2 \\
  \nonumber
  & = \min_{v\in \mathop{\mathrm{col.span}}(X)} \tfrac 1 {\norm{v}^2} \sum_{i=1}^n \Bracbb{R^2 \cdot v_i^2 \le \tfrac 1 {n}\norm{v}^2} \cdot v_i^2\\
  \label{eq:spread-derivation-techniques}
  & \eqcolon \kappa_R\,.
\end{align}
(The last step uses our assumption \(\transpose XX=\Id\).)

It turns out that the above quantity \(\kappa_R\) in \cref{eq:spread-derivation-techniques} indeed captures up to constant factors the radius and curvature parameter of strong convexity of the Huber loss function around \(\beta^*\) for \(\bm \eta=0\).
\Dnote{it might be nice to add an explanation to this claim at some point; showing that small \(\kappa_R\) implies that the huber loss function is not strongly convex for \(\bm \eta = 0\).}
In this sense, the well-spreadness of the column span of \(X\) is required for the current approach of analyzing the Huber-loss estimator based on strong convexity.
The quantity \(\kappa_R\) in \cref{eq:spread-derivation-techniques} is closely related to previously studied notions of well-spreadness for subspaces \cite{DBLP:conf/approx/GuruswamiLW08,DBLP:journals/combinatorica/GuruswamiLR10} in the context of compressed sensing and error-correction over the reals.

Finally, we use a covering argument to show that a well-spread subspace remains well-spread even when restricted to a random fraction of the coordinates (namely the coordinates satisfying \(\abs{\bm \eta_i}\le 1\)).
This fact turns out to imply the desired lower bound on the local strong convexity parameter.
Concretely, if the column space of \(X\) is well-spread in the sense of \cref{eq:spread-derivation-techniques} with parameter \(\kappa_R\) for some \(R\ge \tilde O_{\kappa_R}(\tfrac{d}{\alpha n} )^{1/2}\), we show that the Huber loss function is locally \(\Omega(\alpha\cdot\kappa_R)\)-strong convex at \(\betastar\) within radius \(\Omega(R)\).
(See \cref{thm:convexty-well-spread}.)
Recall that we are interested in the regime \(n\ge d /\alpha^2\) (otherwise, consistent estimation is impossible) and small \(\alpha\), e.g., \(\alpha=d^{-0.1}\).
In this case, Gaussian \(X\) satisfies \(\kappa_R\ge 0.1\) even for constant \(R\).

\paragraph{Final error bound}

The aforementioned general framework for analyzing estimators via strong convexity (see \cref{thm:error-bound-from-strong-convexity}) allows us to bound the error \(\norm{\bm{\hat\beta}-\betastar}\) by the norm of the gradient \(\norm{\nabla \bm f (\betastar)}\) divided by the strong-convexity parameter, assuming that this upper bound is smaller than the strong-convexity radius.

Consequently, for the case that our design \(X\) satisfies \(\kappa_R\ge 0.1\) (corresponding to the setting of \cref{thm:huber-loss-informal}), the previously discussed gradient bound and strong-convexity bound together imply that, with probability over \(\bm \eta\), the error bound satisfies
\[
  \norm{\bm{\hat\beta}-\betastar}\le \underbrace{\vbig O\Paren{\sqrt{\frac d n}}}_{\text{gradient bound}} \cdot \underbrace{O\Paren{\vphantom{\sqrt{\frac d n}}\frac 1 {\alpha}}}_{\text{strong-convexity bound}} = O\Paren{\frac d {\alpha^2 n}}^{1/2}\,,
\]
assuming \(R\gtrsim \sqrt{d/\alpha^2 n}\).
(This lower bound on \(R\) is required by \cref{thm:error-bound-from-strong-convexity} and is stronger than the lower bound on \(R\) required by \cref{thm:convexty-well-spread}.)

\subsection{Huber-loss estimator and high-dimensional medians}\label{sec:techniques-connection-huber-median}
We discuss here some connections between high-dimensional median computations and efficient estimators such as  Huber loss or the LAD estimator. This connection leads to a better understanding of \textit{why} these estimators are not susceptible to heavy-tailed noise. Through this analysis we also obtain guarantees similar to the ones shown in \cref{thm:huber-loss-informal}.

Recall our linear regression model $\bm y = X \betastar + \bm \eta$ as in \cref{eq:linear-model-techniques}. The noise vector $\bm \eta$ has independent, symmetrically distributed coordinates with $\alpha=\min_{i\in[n]}\bbP \Set{\Abs{\bm \eta_i}\leq1}$. We further assume the noise entries to satisfy
\begin{align*}
	 \forall t\in [0,1]\,, \quad \bbP \Paren{\Abs{\bm \eta_i}\leq t}\geq \Omega(\alpha \cdot t)\,.
\end{align*}
This can be assumed without loss of generality as, for example, we may simply add a  Gaussian vector $\bm w\sim N(0, \Id_n)$ (independent of $\bm y$) to $\bm y$ (after this operation parameter $\alpha$ changes only by a constant factor).


\paragraph{The one dimensional case: median algorithm} To understand how to design an efficient algorithm robust to $\Paren{1-\sqrt{d/n}}\cdot n$ corruptions, it is instructive to look into the simple settings of one dimensional  Gaussian design $\bm X\sim N(0,\Id_n)$. 
Given samples $(\bm y_1,\bm X_1),\ldots,(\bm y_n,\bm X_n)$ for any $i \in [n]$ such that $\Abs{\bm X_i} \ge 1/2$, consider
\begin{align*}
	\bm y_i/ \bm X_i = \betastar+\bm \eta_i/\bm X_i\,.
\end{align*}
By \textit{obliviousness} the random variables $\bm \eta_i' = \bm \eta_i/\bm X_i$ 
are symmetric about $0$ and for any $0\le t\le 1$, still satisfy
$\Pr(-t \le \bm \eta'_i\le t) \ge \Omega(\alpha\cdot t)$. 
Surprisingly, this simple observation is enough to obtain an optimal robust algorithm. 
Standard tail bounds show that with probability 
$1-\exp\Set{-\Omega\Paren{\alpha^2\cdot \eps^2\cdot n}}$ 
the median ${\hat{\bm \beta}}$ of $\bm y_1/\bm X_1,\ldots,\bm y_n/\bm X_n$ falls in the interval $\Brac{-\eps+\betastar,+\eps+\betastar}$ for any $\eps\in \Brac{0,1}$. 
Hence, setting $\eps \gtrsim1/\sqrt{\alpha^2\cdot n}$ we immediately get that with probability at least $0.999$, $\Snorm{\betastar-{\hat{\bm\beta}}}\leq \eps^2 \leq O(1/(\alpha^2\cdot n))$.

\paragraph{The high-dimensional case: from the median to the  Huber loss} In the one dimensional case, studying the median of the samples $\bm y_1/\bm X_1,\ldots\bm y_n/\bm X_n$ turns out to be enough to obtain optimal guarantees. The next logical step is to try to construct a similar argument in high dimensional settings. 
However, the main problem here is that high dimensional analogs of the median are usually computationally inefficient (e.g. Tukey median \cite{tukey}) and so  this doesn't seem to be a good strategy to design efficient algorithms. Still in our case one such function provides fundamental insight. 



We start by considering the sign pattern of $X\betastar$, we do not fix {any} property of $X$ yet. 
Indeed, note that the median satisfies $\sum_{i \in [n]}\sign\Paren{\bm y_i/X_i-\hat{\bm \beta}}\approx0$ 
and so $\sum_{i \in [n]}\sign\Paren{\bm y_i-\hat{\bm\beta}X_i}\sign(X_i)\approx0$. So a natural generalization to high dimensions is the following candidate estimator
\begin{align}\label{eq:sign-candidate-estimator}
	 \hat{\bm \beta}= \argmin_{\beta\in \R^d}\max_{u \in \R^d}\Abs{\tfrac{1}{n}\iprod{\sign\Paren{\bm y-X\beta},\sign(Xu)}}\,.
\end{align}
Such an estimator may be inefficient to compute, but nonetheless it is instructive to reason about it. We may assume $X,\betastar$ are fixed, so that the randomness of the observations $\bm y_1,\ldots,\bm y_n$ only depends on $\bm{\eta}$.
Since for each $i \in [n]$, the distribution of $\bm \eta_i$ has median zero and as there are at most $n^{O\Paren{d}}$ sign patterns in $\Set{\sign(Xu)\given u \in \R^d}$, standard $\eps$-net arguments show that with high probability
\begin{equation}\label{eq:tail-bound-sign-function}
	\max_{u \in \R^d}\tfrac{1}{n}\Abs{\iprod{\sign\Paren{\bm y-X\hat{\bm\beta}},\sign(Xu)}}\leq \tilde{O}\Paren{\sqrt{d/n}}\,,
\end{equation}
and hence
\[
\max_{u \in \R^d}\tfrac{1}{n}\Abs{\iprod{\sign\Paren{\bm\eta+X\Paren{\betastar-\hat{\bm\beta}}},\sign(Xu)}}\leq \tilde{O}\Paren{\sqrt{d/n}}\,.
\]
Consider $\bm g(z) = \tfrac{1}{n}\iprod{\sign\Paren{\bm\eta+Xz},\sign(Xz)}\leq \tilde{O}\Paren{d/n}$ for $z \in \R^d$. 
Now the central observation is that for any $z\in \R^d$,
\begin{align*}
	\E_{\bm \eta}  \bm g(z) &= \tfrac{1}{n}\sum_{i \in [n]}\E_{\bm \eta}  
	{\sign\Paren{\bm \eta_i+\iprod{X_i,z}}\cdot\sign\Paren{\iprod{X_i,z}}} \\
	&\ge 
	\tfrac{1}{n}\sum_{i \in [n]} 
	\bbP \Paren{0 \ge \sign\Paren{\iprod{X_i,z}}\cdot {\bm \eta_i}\ge -\Abs{\iprod{X_i,z}}}\\
	&\geq \tfrac{1}{n}\sum_{i \in [n]} \Omega\Paren{\alpha} \cdot \min \Set{1,\abs{\iprod{X_i,z}}}\,.
\end{align*}
By triangle inequality $\E \bm g(z) \leq \Abs{ \bm g(z)} + \Abs{\bm g(z)-\E \bm g(z)} $ and using a similar argument as in \cref{eq:tail-bound-sign-function}, with high probability, for any $z\in \R^d$,
\[
 \Abs{\bm g(z)-\E \bm g(z)} \leq \tilde{O}\Paren{\sqrt{d/n}}\,.
\]

Denote with $\bm z:= \betastar- \hat{\bm\beta}\in \R^d$.
Consider $\bm g(z)$,
thinking of $z\in \R^d$ as a \textit{fixed} vector.
This allows us to easily study $\E_{\bm \eta} \bm g(z)$. 
On the other hand, since our bounds are based on $\eps$-net argument, 
we don't have to worry about the dependency of $\bm z$ on $\bm\eta$.

So { without} any constraint on the measurement $X$ we derived the following inequality:
\[
	\tfrac{1}{n}\sum_{i \in [n]}\min\Set{1,\Abs{\iprod{X_i,\bm z}}}
	\leq \tilde{O}\Paren{\sqrt{d/(\alpha^2\cdot n)}}\,.
\]
Now, our well-spread condition \cref{eq:rip-property} will allow us to relate 	
$\tfrac{1}{n}\sum_{i \in [n]}\min\Set{1,\Abs{\iprod{X_i,\bm z}}}$ 
with $\tfrac{1}{n}\sum_{i \in [n]}\iprod{X_i,\bm  z}^2$ and thus obtain a bound of the form
\begin{equation}\label{eq:general-bound-informal}
	\frac{1}{n}\Snorm{X\Paren{\betastar-{\hat{\bm \beta}}}}\leq \tilde{O}\Paren{ d/(\alpha^2 n) }\,.
\end{equation}

So far we glossed over the fact that  \cref{eq:sign-candidate-estimator} may be hard to compute, however it is easy to see that we can replace such estimator with some well-known efficient estimators and keep a similar proof structure. For instance, one could expect the LAD estimator
\begin{equation}\label{eq:lad-estimator}
	\hat{\bm\beta} = \min_{\beta\in \R^d}\Normo{\bm y-X\beta}
\end{equation}
to obtain comparable guarantees. 
For fixed $d$ and $\alpha$ and $n$ tending to infinity
this is indeed the case,
as we know by \cite{pollard} that such estimator recovers $\betastar$. 
The Huber loss function also turns out to be a good proxy for \cref{eq:sign-candidate-estimator}.
Let $\bm g(u):= \tfrac{1}{n}\underset{i \in [n]}{\sum}\iprod{\Phi'_h(\bm \eta_i+\iprod{X_i,u}), Xu}$ where $\Phi_h:\R\rightarrow\R_{\geq 0}$ is the Huber penalty function and $\bm z=\betastar-\hat{\bm\beta}$. Exploiting \textit{only} first order optimality conditions on $ \hat{\bm\beta}$ one can show
\begin{align*}
	\E \bm g(z)\leq \Abs{\bm g(z)-\E \bm g(z)} \leq \tilde{O}\Paren{\sqrt{d/n}}\,,
\end{align*}
using a similar argument as the one mentioned for \cref{eq:tail-bound-sign-function}. Following a similar proof structure as the one sketched above, we can obtain a bound similar to \cref{eq:general-bound-informal}.
Note that this approach crucially exploits the fact that the noise $\bm \eta$ has median zero but does not rely on symmetry and so can successfully obtain a good estimate of $X\betastar$ under \textit{weaker} noise assumptions.

\subsection{Fast algorithms for Gaussian design}\label{sec:techniques-fast-algorithms}
The one dimensional median approach introduced above can be directly extended to high dimensional settings. This essentially amounts to repeating the procedure for each coordinate, thus resulting in an extremely simple and efficient algorithm. More concretely:

\begin{algorithm}[H]
	\caption{Multivariate linear regression iteration via median}
	\label{alg:linear-regression-via-median-informal}
	\begin{algorithmic}
		\STATE \textbf{Input: }$(y,X)$ where $y\in \R^n$, $X\in \R^{n\times d}$. 
		\FOR{\textbf{all} $j\in [d]$ }
		\FOR{\textbf{all} $i\in [n]$ }
		\STATE 	Compute $ z_{ij}=\frac{y}{ X_{ij}}$.
		\ENDFOR
		\STATE Let $\hat{ \beta}_j$ be the median of $\Set{ z_{ij}}_{i\in[n]}$.	
		\ENDFOR
		\STATE \textbf{Return} $\hat{ \beta}:= \transpose{\Paren{\hat{ \beta}_1,\ldots,\hat{ \beta}_d}}$.
	\end{algorithmic}
\end{algorithm}

If $\bm X_1,\ldots, \bm X_n\sim N(0,\Id_d)$, the analysis of the one dimensional case shows that with high probability, for each $j\in [d]$, the algorithm returns $\hat{\bm\beta}_j$ satisfying $(\betastar_j- \hat{\bm\beta}_j)^2\leq O\Paren{\frac{1+\snorm{\betastar}}{\alpha^2}\cdot \log d}$. 
Summing up all the coordinate-wise errors,  \cref{alg:linear-regression-via-median-informal} returns a $O\Paren{\frac{d(1+\snorm{\betastar})}{\alpha^2}\cdot \log d}$-close estimation. This is better than a trivial estimate, but for large $\Norm{\betastar}$ it is far from the $O(d\cdot \log d/(\alpha^2\cdot n))$ error guarantees we aim for. However, using bootstrapping we can indeed improve the accuracy of the estimate.  It suffices to iterate $\log \Norm{\betastar}$ many times.

\begin{algorithm}[H]
	\caption{Multivariate linear regression via median}
	\begin{algorithmic}\label{alg:linear-regression-via-bootstrapping-informal}
		\STATE \textbf{Input: }$(y, X, \Delta)$ where $X\in \R^{n\times d}$, $y\in \R^n$ and $\Delta$ is an upper bound to $\Norm{\betastar}$.
		\STATE Randomly partition the samples $y_1,\ldots,y_n$ in 
		$t := \Theta(\log \Delta)$ sets $\bm \cS_1,\ldots, \bm\cS_{t}$, 
		such that all $\bm\cS_1,\ldots, \bm\cS_{t-1}$ have sizes
		$\Theta\Paren{\frac{n}{\log \Delta}}$ and $\bm\cS_{t}$ has size $\lfloor{n/2}\rfloor$.
		\FOR{\textbf{all} $i\in [t]$ }
		\STATE 		Run \cref{alg:linear-regression-via-median-informal} on input 
		\begin{align*}
			\Paren{y_{\bm\cS_i} -  X_{\bm \cS_i}\Paren{\underset{j<i-1}{\sum}\hat{\bm\beta}^{(j)}},  X_{\bm \cS_i}}\,,
		\end{align*}
		and let  $ \hat{\bm\beta}^{(i)}$ be the resulting estimator.
		\ENDFOR
		\STATE \textbf{Return} $\hat{ \beta}:= \transpose{\Paren{\hat{ \beta}_1,\ldots,\hat{ \beta}_d}}$.
\end{algorithmic}
\end{algorithm}

As mentioned in \cref{sec:results-fast-algorithms}, \cref{alg:linear-regression-via-bootstrapping-informal} requires knowledge of an upper bound $\Delta$ on the norm of $\betastar$. The algorithm only obtains meaningful guarantees for 
\begin{align*}
	\label{eq:median-sample-complexity}
	n\gtrsim{\frac{d}{\alpha^2}\log\Delta\Paren{\log d+\log\log\Delta}}
\end{align*}
 and as such works with nearly optimal (up to poly-logarithmic terms) sample complexity whenever $\Norm{\betastar}$ is polynomial in $d/\alpha^2$.

In these settings, since each iteration $i$ requires $O\Paren{\Card{\bm \cS_i}\cdot d}$ steps, \cref{alg:linear-regression-via-bootstrapping-informal}  runs in linear time $O(n\cdot d)$ and outputs a vector $\hat{\bm \beta}$ satisfying 
\begin{align*}
	\Snorm{\hat{\bm \beta}-\betastar}\leq O\Paren{\frac{d}{\alpha^2\cdot n}\cdot \log d}\,,
\end{align*}
with high probability.

\begin{remark}[On learning the norm of $\betastar$]\label{remark:learn-norm-betasta}
	As was noticed in \cite{SuggalaBR019}, one can obtain a rough estimate of the norm of $\eta$ by projecting $\bm y$ onto the orthogonal complement of the columns span of $\bm X_{[n/2]}$.
	Since the ordinary least square estimator obtains an estimate with error $O(\sqrt{d}\Norm{\eta}/n)$ with high probability,  if $\eta$ is polynomial in the number of samples, we obtain a vector $\bm \betahat_{LS}$ such that $\Norm{\betahat-\bm\betastar_{LS}}\leq n^{O(1)}$. 
	The median algorithm can then be applied on  $\Paren{\bm y = \bm X_{[n]\setminus [n/2]}(\betastar-\bm\betahat_{LS})+\eta,\; \bm X_{[n]\setminus [n/2]},\; \Norm{\betahat-\bm\betastar_{LS}}}$. Note that since $\bm X_{[n/2]}$ and $\bm X_{[n]\setminus [n/2]}$ are independent, $\betastar-\bm\betahat_{LS}$ is independent of $\bm X_{[n]\setminus [n/2]}$.
\end{remark}

\section{Huber-loss estimation guarantees from strong convexity}

In this section, we prove statistical guarantees for the Huber loss estimator by establishing strong convexity bounds for the underlying objective function.




The following theorem allows us to show that the global minimizer of the Huber loss is close to the underlying parameter \(\betastar\).
To be able to apply the theorem, it remains to prove (1) a bound on the gradient of the Huber loss at \(\betastar\) and (2) a lower bound on the curvature of the Huber loss within a sufficiently large radius around \(\betastar\).

\begin{theorem}[Error bound from strong convexity, adapted from \cite{DBLP:conf/nips/NegahbanRWY09,tsakonas2014convergence}]
	\label{thm:error-bound-from-strong-convexity}
  Let \(f\from \R^d\to\R\) be convex differentiable function and let \(\beta^*\in \R^d\).
  Suppose \(f\) is locally \(\kappa\)-strongly convex at \(\beta^*\) within radius \(R>0\):
  \begin{equation}\label{eq:local-strong-convexity}
    \forall u\in \R^d\,,\norm{u}\le R\,.\quad
    f(\betastar+u) \ge f(\betastar) + \iprod{\nabla f(\betastar), u} + \frac{\kappa}{2}\norm{u}^2\,.
  \end{equation}
  
   If \(\norm{\nabla f(\beta^*)} < \frac{1}{2}\cdot R \kappa\), then every vector \(\beta\in\R^d\) such that \(f(\beta)\le f(\betastar)\) satisfies
  \begin{equation} \label{eq:convexity-errorbound}
  \norm{\beta-\betastar} \le 2\cdot \norm{\nabla f(\beta^*)}/\kappa\,.
  \end{equation}

  Furthermore, if
  If \(\norm{\nabla f(\beta^*)} < 0.49\cdot R \kappa\), then every vector \(\beta\in\R^d\) such that \(f(\beta)\le f(\betastar) + \eps\), where \(\eps= 0.01\cdot \norm{\nabla f(\beta^*)}^2/\kappa\), satisfies
  \begin{equation}
     \norm{\beta-\betastar} \le 2.01\cdot \norm{\nabla f(\beta^*)}/\kappa\,.
  \end{equation} 
\end{theorem}

\begin{proof}
  Let $\beta\in\R^d$ be any vector that satisfies \(f(\beta) \le f(\betastar) + \eps\).
  Write \(\beta=\betastar+t\cdot u\) such that \(\norm{\Delta}\le R\) and \(t=\max\set{1,\norm{\beta-\betastar}/R}\).
  Since \(\beta'=\betastar+u\) lies on the line segment joining \(\betastar\) and \(\beta\), the convexity of \(f\) implies that \(f(\beta')\le \max\set{f(\beta),f(\betastar)}\le f(\betastar) + \eps \).
  By local strong convexity and Cauchy--Schwarz,
  \[
    \eps \ge f(\beta') - f(\betastar) \ge - \norm{\nabla f (\betastar)} \cdot \norm{u} + \frac{\kappa}{2}\cdot \norm{u}^2\,.
  \]
  
  If $\eps = 0$, we get $\norm{u} \le 2\cdot \norm{\nabla f(\beta^*)}/\kappa <R$. By our choice of \(t\), this bound on \(\norm{u}\) implies \(t=1\) and we get the desired bound.
  
  If \(\e=0.01\cdot \norm{\nabla f(\betastar)}^2/\kappa\) and \(\norm{\nabla f(\beta^*)} < 0.49\cdot R \kappa\), by the quadratic formula,
  \[
  \norm{u}\le \frac{\norm{\nabla f (\betastar)} + \sqrt{\norm{\nabla f (\betastar)}^2  + 2 \kappa \e}}{\kappa} \le \frac{2.01\norm{\nabla f (\betastar)}}{\kappa} < 2.01\cdot 0.49\cdot R < R\,.
  \]
  Again, by our choice of \(t\), this bound on \(\norm{u}\) implies \(t=1\).
  We can conclude \(\norm{\beta-\betastar}=\norm{u}\le 2.01 \norm{\nabla f (\betastar)}/\kappa\) as desired.
\end{proof}

We remark that the notion of local strong convexity \cref{eq:local-strong-convexity} differs from the usual notion of strong convexity in that one evaluation point for the function is fixed to be \(\betastar\).
(For the usual notion of strong convexity both evaluation points for the function may vary within some convex region.)
However, it is possible to adapt our proof of local strong convexity to establish also (regular) strong convexity inside a ball centered at \(\betastar\).

A more general form of the above theorem suitable for the analysis of regularized M-estimators appears in \cite{DBLP:conf/nips/NegahbanRWY09} (see also \cite{wainwright_2019}).
Earlier analyses of the Huber-loss estimator also use this theorem implicitly \cite{tsakonas2014convergence}.
(See the the discussion in \cref{sec:techniques-strong-convexity}.)

The following theorem gives an upper bound on the gradient of the Huber loss at \(\betastar\) for probabilistic error vectors \(\bm \eta\).
\begin{theorem}[Gradient bound for Huber loss]
  \label{thm:gradient-bound}
  Let \(X\in\R^{n\times d}\) with \(\transpose X X=\Id\) and \(\betastar\in\R^d\).
  Let \(\bm \eta\) be an \(n\)-dimensional random vector with independent symmetrically-distributed entries.
  
  Then for any $\delta \in (0,1)$, with probability at least $1-\delta/2$, the Huber loss function \(\bm f(\beta)= \tfrac 1 n \sum_{i=1}^n \Phi[(X\beta -\bm y)_i]\) for \(\bm y=X \betastar + \bm \eta\) satisfies
	\[
    \norm{\nabla \bm f(\beta^*)} \le 8\sqrt{\frac{d + \ln(2/\delta)}{n}}\,.
	\]
\end{theorem}

\begin{proof}
	Let $\bm z$ be the \(n\)-dimensional random vector with entries $\bm z_i= \Phi'(\bm \eta_i)$, where \(\Phi'(t)=\sign(t)\cdot \min\set{2,\abs{t}}\).
  Then, \(\nabla \bm f(\betastar) = \tfrac 1n \sum_{i=1}^n \bm z_i \cdot x_i\).
	Since $\bm \eta_i$ is symmetric, \(\E \bm z_i=0\).
  By the Hoeffding bound, every unit vector \(u\in\R^d\) satisfies with probability at least $1-2\exp\paren{-2\Paren{d+\ln(2/\delta)}}$.
	\[
    n\cdot \Abs{\iprod{\nabla  \bm f(\beta^*),u}} = \Abs{\iprod{\bm z, X  u}}\leq 4 \sqrt{{d+\ln(2/\delta)}\,} \cdot\Norm{X u} = 4\sqrt{\Paren{d+\ln(2/\delta)} n\,}\,.
	\]
	 Hence, by union bound over a \(1/2\)-covering of the $d$-dimensional unit ball of size at most \(5^d\), we have with probability at least $1-2\exp\paren{-2\ln(2/\delta)} \ge 1 - \delta/2$,
	\begin{align*}
	\underset{\norm{u} \le 1} {\max}\Abs{\iprod{\nabla  \bm f(\beta^*),u}} 
	&\le 
	4\sqrt{\Paren{d+\ln(2/\delta)}/ n\,} + \underset{\norm{u}\le 1/2}{\max}\Abs{\iprod{\nabla  \bm f(\beta^*),u}} 
	\\&= 
	4\sqrt{\Paren{d+\ln(2/\delta)}/n\,}+ \tfrac 12 \underset{\norm{u}\le 1}{\max}\Abs{\iprod{\nabla  \bm f(\beta^*),u}}\,.
	\end{align*}
	Since $u=\frac{1}{\norm{\nabla \bm f(\beta^*)}}\nabla \bm f(\beta^*)$ satisfies $\iprod{\nabla  \bm f(\beta^*),u } = \norm{\nabla \bm f(\beta^*)}$, we get the desired bound.
\end{proof}

\Dnote{MAYBE: add somewhere what kind of deterministic guarantees we achieve for the huber loss. i think that would just boil down to writing out the gradient norm and the eigenvalue lower bound for the hessian}

\paragraph{Proof of local strong convexity}

The following lemma represents the second-order behavior of the Huber penalty as an integral.
To prove local strong convexity for the Huber-loss function, we will lower bound this integral summed over all sample points.

\begin{lemma}[Second-order behavior of Huber penalty]
  \label{lem:second-order-huber-penalty}
  For all $h,\eta\in \R$,
  \begin{equation}
    \Phi(\eta+h)-\Phi(\eta) - \Phi'(\eta)\cdot h
    = h^2\cdot\int_{0}^1 (1-t)\cdot  \Ind_{\abs{\eta+t\cdot h}\le 2} \,\mathrm d t
    \ge \tfrac 12 \Ind_{\abs{\eta}\le 1}\cdot \Ind_{\abs{h}\le 1}
    \,.
  \end{equation}
\end{lemma}

\begin{proof}
  A direct consequence of Taylor's theorem and the integral form of the remainder of the Tayler approximation.
  Concretely, consider the function \(g\from \R\to\R\) with \(g(t)=\Phi(\eta+t\cdot h)\).
  The first derivative of \(g\) at \(0\) is \(g'(0)=\Phi'(\eta)\cdot h\).
  The function \(g''(t)=h^2\cdot \Ind_{\abs{t}\le 2}\) is the second derivative of \(g\) as an $L_1$ function (so that \(\int_a^b g''(t)\,\mathrm d t=g'(b)-g'(a)\) for all \(a,b\in \R\)).
  Then, the lemma follows from the following integral form of the remainder of the first-order Taylor expansion of \(g\) at \(0\),
  \[
    g(1)-g(0) - g'(0) = \int_{0}^1 (1-t)\cdot  g''(t)\, \mathrm d t\,.
  \]
  Finally, we lower bound the above right-hand side by \(\ge \tfrac 12 \Ind_{\abs{\eta}\le 1}\cdot \Ind_{\abs{h}\le 1}\) using \(\int_0^1(1-t)\mathrm d t=\tfrac 12\) and the fact \(g''(t)\ge \Ind_{\abs{\eta}\le 1}\cdot \Ind_{\abs{h}\le 1}\) for all \(t\in[0,1]\).
\end{proof}

\begin{theorem}[Strong convexity of Huber loss]\label{thm:convexty-well-spread}
  Let \(X\in\R^{n\times d}\) with \(\transpose X X=\Id\) and \(\betastar\in\R^d\).
  Let \(\bm \eta\) be an \(n\)-dimensional random vector with independent entries such that \(\alpha=\min_{i}\Pr\set{\abs{\bm\eta_i}\le 1}\).
  Let $\kappa > 0$ and $\delta \in (0,1)$.
  Suppose that every vector \(v\) in the column span of \(X\) satisfies
  \begin{equation}
    \label{eq:convexty-theorem-well-spread}
    \sum_{i=1}^n \Bracbb{r^2\cdot v_i^2\le \tfrac 1{n}\norm{v}^2}\cdot v_i^2 \ge \kappa\cdot \norm{v}^2
    \,,
  \end{equation}
  with
  \[
  \sqrt{\frac{50\cdot \Paren{d\cdot\ln\Paren{\frac{100}{\alpha\kappa}} + \ln(2/\delta)}}{\kappa^2\alpha n}}
  \le  r\le 1\,.
  \]
  
  Then, with probability at least \(1-\delta/2\), the Huber loss function \(\bm f(\beta)= \tfrac 1 n \sum_{i=1}^n \Phi[(X\beta -\bm y)_i]\) for \(\bm y=X \betastar + \bm \eta\) is locally \(\kappa \alpha\)-strongly convex at \(\betastar\) within radius \(r/2\) (in the sense of \cref{eq:local-strong-convexity}).
\end{theorem}

\begin{proof}
  By \cref{lem:second-order-huber-penalty}, for every \(u\in\R^d\),
  \begin{align}
    \notag
    \bm f(\betastar +u) - \bm f(\betastar)-\iprod{\nabla \bm f (\betastar),u}
    & = \tfrac 1n \sum_{i=1}^n \Phi(\iprod{x_i,u}-\bm \eta_i) - \Phi(-\bm \eta_i) - \Phi'(-\bm \eta_i)\cdot\iprod{x_i,u}\\
    \label{eq:second-order-lower-bound}
    & \ge \tfrac 1{2n} \sum_{i=1}^n\iprod{x_i,u}^2 \cdot \Ind_{\abs{\iprod{x_i,u}}\le 1} \cdot \Ind_{\abs{\bm\eta_i}\le 1}\,.
  \end{align}
  
  It remains show that with high probability over the realization of \(\bm \eta\), the right-hand side \cref{eq:second-order-lower-bound} is bounded from below uniformly over all \(u\in \R^d\) in a ball.

  To this end, we will show, using a covering argument, that with probability at least \(1-2^{-\Omega(d)}\) over \(\bm \eta\), for every unit vector \(u\in \R^d\), the vector \(v=Xu\) satisfies the following inequality,
  \begin{equation}
    \label{eq:bound-convexity-proof}
    \tfrac 1 n\sum_{i=1}^n \bracbb{v_i^2\le 4/r^2} \cdot v_i^2 \cdot \bracbb{\abs{\bm\eta_i}\le 1} \ge \alpha \kappa/2\,.
  \end{equation}
  (Since \(\tfrac 1n\transpose X X = \Id\) and \(\norm{u}=1\), the vector \(v\) has average squared entry \(1\).)

  Let \(N_\e\) be an \(\e\)-covering of the unit sphere in \(\R^d\) of size \(\card{N_\e}\le (3/\e)^d\) for a parameter \(\e\) to be determined later.
  Let \(u\in\R^d\) be an arbitrary unit vector and let \(v=Xu\).
  Choose \(u'\in N_\e\) such that \(\norm{u-u'}\le \e\) and let \(v'=Xu'\).
  We establish the following lower bound on the left-hand side of \cref{eq:bound-convexity-proof} in terms of a similar expression for \(v'\),
  \begin{align}
    \label{eq:net-expression}
    & \tfrac 1 n\sum_{i=1}^n \bracbb{(v'_i)^2\le 1/r^2} \cdot (v_i')^2 \cdot \bracbb{\abs{\bm\eta_i}\le 1}\\
    \label{eq:first-step-proof}
    & \le \e^2 + \tfrac 1 n \sum_{i=1}^n \bracbb{v_i^2\le 4/r^2}\cdot  \bracbb{(v'_i)^2\le 1/r^2}  \cdot (v_i')^2 \cdot \bracbb{\abs{\bm\eta_i}\le 1}\\
    \label{eq:second-step-proof}
    & \le  2\e+  \e^2 + \tfrac 1 n \sum_{i=1}^n \bracbb{v_i^2\le 4/r^2}\cdot  \bracbb{(v'_i)^2\le 1/r^2}  \cdot v_i^2 \cdot \bracbb{\abs{\bm\eta_i}\le 1}
  \end{align}
  The first step \cref{eq:first-step-proof} uses that each term in the first sum that doesn't appear in the second sum corresponds to a coordinate \(i\) with \(\abs{v_i'}\le 1/r\) and \(\abs{v_i}\ge 2/r\), which means that \((v_i-v_i')^2\ge  1/r^2\).
  Since each term has value at most \(1/r^2\), the sum of those terms is bounded by \(\norm{v-v'}^2\le \e^2 n\).
  For the second step \cref{eq:second-step-proof}, let \((w'_i)^2\) be the terms of the second sum and \(w_i^2\) the terms of the third sum.
  Then, the difference of the two sums is equal to \(\iprod{w-w',w+w'}\le \norm{w-w'}\cdot \norm{w+w'}\).
  We have \(\norm{w-w'}\le \norm{v-v'}\le \e \sqrt n\) and \(\norm{w+w'}\le \norm{v}+\norm{v'}=2\sqrt n\).

  It remains to lower bound the expression \cref{eq:net-expression} over all \(u'\in N_\e\).
  Let \(\bm z_i=\alpha_i-\bracbb{\abs{\bm \eta_i}\le 1}\), where \(\alpha_i=\Pr\set{\abs{\bm \eta_i}\le 1}\ge \alpha\).
  The random variables \(\bm z_1,\ldots,\bm z_n\) are independent, centered, and satisfy \(\abs{\bm z_i}\le 1\).
  Let \(c_i= \bracbb{(v'_i)^2\le 1/r^2}\cdot (v'_i)^2\).
  By Bernstein inequality, for all \(t\ge 1\),
  \[
    \Pr\Set{\sum_{i=1}^n c_i\cdot \bm z_i
      \ge t\cdot \sqrt{\sum_{i\in [n]} \alpha_i c_i^2} + t^2/r^2} \le e^{-t^2/4}\,.
  \]
  Since \(c_i^2\le c_i/r^2\), we have \(\sum_{i\in [n]} \alpha_i c_i^2 \le \frac{1}{r^2}\sum_{i\in [n]} \alpha_i c_i\). 
  Denote $b = \tfrac 1 n\sum_{i\in [n]} \alpha_i c_i$. Note that $b\ge \alpha \kappa$.

  Choosing $\eps = 0.03\alpha \kappa$, \(t=2\sqrt{d \ln (3/\e) + \ln(2/\delta)}\), by the union bound over \(N_\e\), it follows that with probability at least \(1-\delta/2\), for every \(u'\in N_\e\), the vector \(v'=Xu'\) satisfies,
  \begin{align*}
    \tfrac 1 n\sum_{i=1}^n \bracbb{(v'_i)^2\le 1/r^2} \cdot (v_i')^2 \cdot \bracbb{\abs{\bm\eta_i}\le 1} 
    &\ge 
   b - \sqrt{\frac{t^2b}{r^2 n}} - \frac{t^2}{r^2 n}
    \\&\ge
    b - \sqrt{0.1}\cdot \sqrt{b}\cdot \sqrt{\alpha \kappa} - 0.08\cdot \alpha \kappa
     \\&\ge
     0.6\alpha \kappa\,.
  \end{align*}
  As discussed before, this event implies that for all unit vectors \(u\in\R^d\), the vector \(v=X u\) satisfies
  \[
     \tfrac 1 n\sum_{i=1}^n \bracbb{v_i^2\le 4/r^2} \cdot v_i^2 \cdot \bracbb{\abs{\bm\eta_i}\le 1} 
    \ge 0.6\alpha \kappa - 2\e - \e^2 
    \ge 0.5 \alpha \kappa\,. 
  \]
\end{proof}

\paragraph{Putting things together}

In this paragraph, we proof \cref{thm:huber-loss-informal} by combining previous results in this section.

We start with the definition of well-spread property:

\begin{definition}
	Let $V\subseteq \R^n$ be a vector space. $V$ is called \emph{$(m,\rho)$-spread}, if for every $v\in V$ and every subset $S\subseteq [n]$ with $\Card{S} \ge n-m$,
	\[
	\norm{v_S} \ge \rho \norm{v}\,.
	\]
\end{definition}

\begin{theorem}\label{thm:huber-loss-formal}
		Let $X\in \R^{n\times d}$ be a deterministic matrix and let $\bm \eta$ be an $n$-dimensional random vector with independent, symmetrically distributed entries and $\alpha=\min_{i\in [n]}\Pr\set{\abs{\bm \eta_i}\leq 1}$.

		Let $\rho \in (0,1)$ and $\delta \in (0,1)$ and suppose  that column span of \(X\) is \((m,\rho)\)-spread\footnote{\textbf{The statement of this theorem in the submitted short version of this paper unfortunately contained a typo and stated an incorrect bound for $\Card{S}=m$.}}  for 
		\[
		m = \frac{100\cdot \Paren{d\cdot \ln(10/\rho) +\ln(2/\delta)}}{\rho^4 \cdot \alpha^2}\,.
		\]
		
		Then, with probability at least \(1-\delta\) over \(\bm \eta\),
		for every \(\betastar \in\R^d\), given $ X$ and \(\bm y= X \betastar + \bm \eta\), 
		the Huber-loss estimator $\bm\betahat$ satisfies
		\begin{align*}
		\tfrac 1 n \Norm{X(\betastar - \bm{\hat{\beta}})}^2
		\le 1000\cdot \frac{d + \ln(2/\delta)}{\rho^4\cdot \alpha^2\cdot n}\,.
		\end{align*}
\end{theorem}

\begin{proof}
	Note that if column span of $X$ is $(m, \rho)$-spread, then 
	\cref{eq:convexty-theorem-well-spread} holds for all $v$ from column span of $X$ 
	with $r^2 = m/n$ and $\kappa = \rho^2$. 
	Indeed, the set $\Set{i\in [n]\suchthat v_i^2 > \frac{1}{r^2 n}\norm{v}^2}$ has size at most $r^2 n = m$, 
	so $\sum_{i\in[n]\setminus S} v_i^2 \ge \rho^2 \norm{v}^2 = \kappa \norm{v}^2$.
	Hence for $m =  \frac{100\Paren{d\ln(10/\rho) +\ln(2/\delta)}}{\rho^4 \alpha^2}$, 
	the conditions of \cref{thm:convexty-well-spread} are satisfied and $\bm f$ is 
	$\rho^2\alpha$-strongly convex in the ball of radius $\sqrt{m/n}$ with probability at least $1-\delta/2$.
	
	By \cref{thm:gradient-bound}, with probability at least $1-\delta/2$, 
	\[
	\norm{\nabla \bm f(\beta^*)} \le 8\sqrt{\frac{d + \ln(2/\delta)}{n}}\,.
	\]
	Hence with probability at least $1-\delta$,  $\norm{\nabla \bm f(\beta^*)} < 0.49 \rho^2\alpha\sqrt{m/n} $.  
	Therefore, by \cref{thm:error-bound-from-strong-convexity}, with probability at least $1-\delta$,
	\[
	\tfrac 1 n \Norm{X(\betastar - \bm{\hat{\beta}})}^2 \le 2.1^2 \cdot \frac{\norm{\nabla \bm f(\beta^*)}^2}{\rho^4\cdot \alpha^2} \le 1000\cdot \frac{d + \ln(2/\delta)}{\rho^4\cdot \alpha^2\cdot n}\,.
	\]
\end{proof}

\begin{proof}[Proof of \cref{thm:huber-loss-informal}]
	\cref{thm:huber-loss-informal} follows from \cref{thm:huber-loss-formal} 
	with $\rho = \sqrt{1-0.81} = \sqrt{0.19}$ and $\delta = 2^{-d}$. 
	Note that in this case $m \le 20000 \cdot d/\alpha^2$.
\end{proof}

\subsection{Huber-loss estimator for Gaussian design and deterministic noise}

In this section we provide a proof of \cref{thm:huber-loss-gaussian-results}. We will use the same strategy as in the previous section: show that the gradient at $\beta^*$ is bounded by $O\Paren{\sqrt{d/n}}$, then show that Huber loss is locally strongly convex at $\beta^*$ in a ball of radius $\Omega(1)$, and then use  \cref{thm:error-bound-from-strong-convexity} to obtain the desired bound.

\paragraph{Gradient bound}
\begin{theorem}[Gradient bound, Gaussian design]
	\label{thm:gradient-bound-gaussian}
	Let \(\bm X\sim N(0,1)^{n\times d}\) and \(\betastar\in\R^d\).
	Let \( \eta \in \R^n\) be a deterministic vector.
	
	Then for every $\delta \in (0,1)$, with probability at least $1-\delta/2$, the Huber loss function \(\bm f(\beta)= \tfrac 1 n \sum_{i=1}^n \Phi[(\bm X\beta -\bm y)_i]\) for \(\bm y=\bm X \betastar + \eta\) satisfies
	\[
	\norm{\nabla \bm f(\beta^*)}^2 \le 
	3\sqrt{\frac{d + 2\ln(2/\delta)}{n}}\,.
	\]
\end{theorem}
\begin{proof}
	The distribution of $\nabla\bm f(\beta^*)$ is $N\Paren{0,\frac{1}{n^2}\underset{i \in [n]}{\sum}\Paren{\Phi'\paren{\eta_i}}^2\cdot \Id_d}$. Hence by \cref{fact:chi-squared-tail-bounds}, with probability at least $1-\delta/2$,
	\[
	\norm{\nabla\bm f(\beta^*)}^2 \le \frac{4}{n}\Paren{d + 2\ln(2/\delta) + 2\sqrt{d\ln(2/\delta)}} \le \frac{8d + 12\ln(2/\delta)}{n}\,.
	\]
\end{proof}

\paragraph{Strong convexity}
\begin{theorem}[Strong convexity, Gaussian design]\label{thm:convexty-gaussian}
	Let \(\bm X\sim N(0,1)^{n\times d}\) and \(\betastar\in\R^d\).
Let \( \eta \in \R^n\) be a deterministic vector with $\alpha n$ entries of magnitude at most $1$. Suppose that for some $\delta \in (0,1)$,
\[
n \ge 200\cdot {\frac{d+2\ln(4/\delta)}{\alpha^2}}\,.
\]

Then with probability at least $1-\delta/2$, the Huber loss function \(\bm f(\beta)= \tfrac 1 n \sum_{i=1}^n \Phi[(\bm X\beta -\bm y)_i]\) for \(\bm y=\bm X \betastar + \eta\)
is locally  $\alpha$-strongly convex at $\beta^*$ within radius $1/6$ (in the sense of \cref{eq:local-strong-convexity}).
\end{theorem}
\begin{proof}
  By \cref{lem:second-order-huber-penalty}, for every \(u\in\R^d\),
\begin{align*}
\notag
\bm f(\betastar +u) - \bm f(\betastar)-\iprod{\nabla \bm f (\betastar),u}
& = \tfrac 1n \sum_{i=1}^n \Phi(\iprod{\bm x_i,u}- \eta_i) - \Phi(- \eta_i) - \Phi'(- \eta_i)\cdot\iprod{\bm x_i,u}\\
& \ge \tfrac 1{2n} \sum_{i=1}^n\iprod{\bm x_i,u}^2 \cdot \Ind_{\abs{\iprod{\bm x_i,u}}\le 1} \cdot \Ind_{\abs{\eta_i}\le 1}\,.
\end{align*}
Consider the set $\cC = \Set{i\in [n]\suchthat \abs{\eta_i}\le 1}$. Since $\card{\cC}= \alpha n$ and $\eta$ is deterministic, $\bm X_{\cC}\sim N(0,1)^{\alpha n\times d}$.

By  \cref{fact:k-sparse-norm-gaussian}, for $k=\alpha n/200$,
with probability at least $1-\delta/4$, for any set $\cK\subseteq \cC$ of size $k$ and every \(u\in\R^d\),
\begin{align*}
\sum_{i\in\cK}\iprod{\bm x_i,u}^2
&\le \norm{u}^2 \cdot  \Paren{\sqrt{d}+ \sqrt{k} +  \sqrt{2k\ln\Paren{\frac{e\alpha n}{k}}} + \sqrt{2\ln(4/\delta)}}^2
\\ &\le
 \norm{u}^2\cdot\Paren{\sqrt{\alpha n/200} + \sqrt{0.01\alpha n\ln\Paren{200e}} + 0.1\cdot \sqrt{\alpha n}}^2
 \\&\le
 0.18\norm{u}^2\cdot \alpha n
\,.
\end{align*}
Now if $\cK$ is the set of top $k$ entries of $\bm Xu$ for $u\in \R^d$ such that $\norm{u}\le 1/6$, then we get that the average squared coordinate in $\cK$ is at most $1$. Hence
\[
\sum_{i=1}^n\iprod{\bm x_i,u}^2 \cdot \Ind_{\abs{\iprod{\bm x_i,u}}\le 1} \cdot \Ind_{\abs{\eta_i}\le 1} \ge 
\sum_{i\in\cC\setminus\cK}\iprod{\bm x_i,u}^2 
\ge
 \norm{\bm X_{\cC}u}^2 - 0.18\cdot\norm{u}^2\alpha n\,.
\]

Since $\bm X_{\cC}$ is a Gaussian matrix, for all $u\in \R^d$, with probability at least $1-\delta/4$,
\[
\norm{\bm X_{\cC}u}^2\ge \norm{u}^2  
\Paren{\sqrt{\alpha n} - \sqrt{d} -\sqrt{2\log(4/\delta)}}^2 \ge 
\norm{u}^2  
\Paren{\sqrt{\alpha n} - 0.1\sqrt{\alpha n}}^2
\ge 
0.81\norm{u}^2 \alpha n
\,.
\]
Hence with probability at least $1-\delta/2$, for all $u\in \R^d$ such that $\norm{u}\le 1/6$,
\[
\bm f(\betastar +u) - \bm f(\betastar)-\iprod{\nabla \bm f (\betastar),u}
\ge 0.5\alpha\,,
\]
and $\bm f$ is locally strongly convex with parameter $\alpha$ at $\beta^*$ in the ball of radius $1/6$.
\end{proof}

\paragraph{Putting everything together}
The following theorem implies \cref{thm:huber-loss-gaussian-results}.
\begin{theorem}\label{thm:huber-loss-gaussian-technical}
		Let \(\eta\in\R^n\) be a deterministic vector.
	Let \(\bm X\) be a random\footnote{As a convention, we use boldface to denote random variables.} \(n\)-by-\(d\) matrix with iid standard Gaussian entries \(\bm X_{ij}\sim N(0,1)\).
	
	Suppose that for some $\delta\in(0,1)$,
	\[
	n\ge 2000\cdot {\frac{d+2\ln(2/\delta)}{\alpha^2}}\,,
	\]
	where \(\alpha\) is the fraction of entries in \(\eta\) of magnitude at most \(1\).
	
	Then, with probability at least \(1-\delta\), for every $\betastar \in \R^d$, given $\bm X$ and \(\bm y=\bm X \betastar + \eta\), the Huber-loss estimator $\bm\betahat$ satisfies
	\begin{align*}
	\Norm{\betastar - \bm{\hat{\beta}}}^2
	\le 100\cdot {\frac{d+2\ln(2/\delta)}{\alpha^2 n}}\,.
	\end{align*}
\end{theorem} 
\begin{proof}
	Using bounds from  \cref{thm:gradient-bound-gaussian} and \cref{thm:convexty-gaussian}, we can apply \cref{thm:error-bound-from-strong-convexity}. 
	Indeed, 
	 with probability at least $1-\delta$,
	\[
	R = 1/6 > 2\cdot 3\sqrt{\frac{d+2\ln(2/\delta)}{\alpha^2 n}} \ge  2\cdot \frac{\norm{\nabla \bm f(\beta^*)}}{\alpha} \,,
	\]
	Hence
	\[
	\Norm{\betastar - \bm{\hat{\beta}}}^2 \le 
	4\cdot \frac{\norm{\nabla \bm f(\beta^*)}^2 }{\alpha^2} \le 100\cdot {\frac{d+2\ln(2/\delta)}{\alpha^2 n}}\,.
	\]
\end{proof}
\section{Robust regression in linear time}\label{sec:median-algorithm}

In this section we prove \cref{thm:median-algorithm-informal} and \cref{thm:median-sparse}. We consider the linear model $\bm y = \bm X\betastar+\bm \eta$ where $\bm X \in \R^{n\times d}$ has i.i.d entries $\bm X_\ij\sim N(0,1)$ and the  noise vector $\eta$  satisfies the following assumption.

\begin{assumption}\label{assumption:noise-median}
	$\eta\in \R^n$ is a fixed vector such that  for some $\alpha = \alpha(n,d) \in (0,1)$,
	$\Card{\Set{i\in [n] \suchthat \abs{\bm \eta_i} \le 1}} \ge \alpha n$. 
	We denote $\cT:= \Set{i\in [n] \suchthat \abs{ \eta_i} \le 1}$.
\end{assumption}

We start showing how to obtain a linear time algorithm for Gaussian design in one dimension. Then generalize it to high dimensional settings. 
We add the following (linear time) preprocessing step 
\begin{align*}\label{eq:median-preprocess}
	\forall i \in [n], \qquad& \bm y'_i=\bm \sigma_i\cdot \bm y_i+\bm w_i, \\
	 &\bm X'_i = \bm \sigma_i \cdot \bm X_i,\qquad \bm w_i \sim N(0,1), \bm\sigma_i\sim U\Set{-1,1}\,,\tag{PRE}
\end{align*}
where $\bm w_1,\ldots,\bm w_n,\bm \sigma_1,\ldots, \bm \sigma_n, \bm X$ 
are mutually independent.
For simplicity, when the context is clear we denote $\bm\sigma_i \eta_i+\bm w_i$  by $\bm \eta_i$ and $\bm y',\bm X'$ with $\bm y,\bm X$. Note that this preprocessing step takes time linear in $nd$.
\cref{assumption:noise-median} implies that after this preprocessing step, $\bm \eta$ satisfies the following assumption:
\begin{assumption}\label{assumption:noise-median-symmetry}
For all $i\in \cT$ and  for any $t\in \brac{0,1}$,
\[
\bbP \Paren{0\leq \bm \eta_j\leq t} = \bbP \Paren{-t\leq \bm \eta_j\leq 0}\geq t/10\,.
\]
\end{assumption}

\subsection{Warm up: one-dimensional settings}
For the one-dimensional settings, the only property of the design $n\times 1$ matrix $\bm X\sim N(0,\Id_n)$ we are going to use is anti-concentration.

\begin{fact}\label{fact:gaussian-anticoncentration}
	Let $\bm X\sim N(0,\Id_n)$. Then for any $c\in [0,1]$ and $i \in [n]$,
	\begin{align*}
		\bbP\Paren{\Abs{\bm X_i}\geq c}\geq \Omega(1)\,.
	\end{align*}
\end{fact}

As shown below, our estimator simply computes a median of the samples.
\begin{algorithm}[H]
  \caption{Univariate Linear Regression via Median}
  \label{alg:univariate-linear-regression-via-median}
	\mbox{}\\
	\textbf{Input:} $(y,X)$, where $y,X\in \R^n$.
	\begin{enumerate}[1.]
		\item[0.] Preprocess $y,X$ as in \cref{eq:median-preprocess} and let  $(\bm y', \bm X')$ be the resulting pair.
		\noindent\rule{14cm}{0.21pt}
		\item Let $\bm \cM = \Set{i\in [n] \suchthat \abs{\bm X'_i} \ge 1/2}$. 
		For $i \in \bm \cM$, compute $\bm z_i=\frac{\bm y'_i}{\bm X'_i}$.
		\item Return the median $\hat{\bm \beta}$ of $\Set{\bm z_i}_{i\in\bm \cM}$.
	\end{enumerate}
\end{algorithm}
\begin{remark}[Running time]
	Preprocessing takes linear time.
	Finding $\cM$ requires linear time, similarly we can compute all $\bm z_i$ in  $O(n)$. The median can then be found in linear time using \textit{quickselect} \cite{Hoare} with pivot chosen running the \textit{median of medians} algorithm \cite{blum1973time}. Thus the overall running time is $O(n)$.
\end{remark}

The guarantees of the algorithm are proved in the following theorem.

\begin{theorem}\label{thm:median-univariate}
	Let $\bm  y = \bm  X\betastar+\eta$ for arbitrary  
	$\betastar\in \R$,  $\bm  X\sim N(0,1)^{n \times d}$ and $\eta\in \R^n$  
	satisfying  \cref{assumption:noise-median} with parameter $\alpha$.
	Let $\bm \betahat$ be the estimator computed by Algorithm \ref{alg:univariate-linear-regression-via-median} given $(\bm  X, \bm  y)$ as input. Then for any positive $\tau \leq \alpha^2\cdot n$,
	\begin{equation*}
	\Snorm{\betastar-\bm \betahat}\le \frac{\tau}{\alpha^2\cdot n}
	\end{equation*}
	with probability at least $1-2\exp\Set{-\Omega(\tau)}$.
\end{theorem}

To prove \cref{thm:median-univariate} we will use the following bound on the median, which we prove in \cref{sec:missing_proofs}.

\begin{lemma}\label{lem:median-meta}
	Let $\cS\subseteq [n]$ be a set of size $\gamma n$ and let $\bm z_1,\ldots, \bm z_n\in \R$ be mutually independent  random variables satisfying 
	\begin{enumerate}
		\item For all $i\in [n]$, $\bbP \Paren{\bm z_i \ge 0} = \bbP \Paren{\bm z_i \le 0}$.
		\item For some $\eps\geq 0$, for all $i\in \cS$,  $\bbP \Paren{\bm z_i\in \Brac{0,\eps}}=
		\bbP \Paren{\bm z_i \in \Brac{-\eps,0}}\geq q$.
	\end{enumerate}
	Then with probability at least $1-2\exp\Set{-\Omega\Paren{q^2\gamma^2 n}}$ 
	the median $\bm{\hat{z}}$ satisfies
	\[
	\Abs{\bm{\hat{z}}}\leq \eps\,.
	\]
\end{lemma}

\begin{proof}[Proof of \cref{thm:median-univariate}]
	Due to the preprocessing step the resulting noise  $\bm \eta$ satisfies 
	\cref{assumption:noise-median-symmetry}.  
	Let $\bm  \cM\subseteq [n]$ be the set of entries such that $\abs{\bm  X'_i} \geq \frac{1}{2}$.    Since $\cT$ and $\bm\cM$ are independent, 
	by Chernoff bound,
	$\Card{\cT\cap\bm  \cM}\geq \Omega(\alpha n)$ 
	with probability at least $1-2\exp \Brac{- \Omega(\alpha n)}\ge 1-2\exp\Brac{-\Omega(\tau)}$.
	Now observe that for all $\eps\in (0,1)$ and for all $i\in\cT\cap\bm  \cM$, by \cref{assumption:noise-median-symmetry},
\[
	\bbP\Paren{\Abs{\bm z_i-\betastar}\leq \eps} =
	\bbP\Paren{\Abs{\frac{\bm \eta_i}{\bm X'_i}}\leq \eps}
	\geq \bbP \Paren{\Abs{\bm \eta_i} \leq \eps/2}
	\geq \frac{\eps}{20}\,.
\]
	By \cref{lem:median-meta}, we get the desired bound for $\tau = \eps^2 \alpha^2 n$.
\end{proof}

\subsection{High-dimensional settings}\label{sec:median-high-dimension}
The median approach can also be applied in higher dimensions. In these settings we need to assume that an upper bound $\Delta$ on $\norm{\betastar}$ is known. 

\begin{remark}
	As shown in \cite{SuggalaBR019},  under the model of \cref{thm:huber-loss-gaussian-results} the classical least square estimator obtain an estimate $\hat{\bm  \beta}$ with error $\frac{d}{n}\cdot \Norm{\eta}$. Thus under the additional assumption that the noise magnitude is polynomial in $n$ it is easy to obtain a good enough estimate of the parameter vector.
\end{remark}

We first prove in \cref{sec:high-dimensional-iteration} how to obtain an estimate of the form $\Snorm{\betastar-\hat{\beta}}\leq \frac{\snorm{\betastar}}{2}$. Then in \cref{sec:high-dimensional-bootstrapping} we obtain \cref{thm:median-algorithm-informal}, using bootstrapping.
In \cref{sec:high-dimensional-iteration-sparse} we generalize the results to sparse parameter vector $\betastar$, proving \cref{thm:median-sparse}. Finally we extend  the result of \cref{sec:high-dimensional-bootstrapping} to non-spherical Gaussians in \cref{sec:high-dimensional-nonspherical}.

\subsubsection{High-dimensional Estimation via median algorithm}\label{sec:high-dimensional-iteration}

To get an estimate of the form $\Snorm{\betastar-\hat{\beta}}\leq  \frac{\snorm{\betastar}}{2}$, we use the algorithm below:
\begin{algorithm}[H]
  \caption{Multivariate Linear Regression Iteration via Median}
  \label{alg:linear-regression-via-median}
	\mbox{}\\
	\textbf{Input:} $(y,X)$ where $y\in \R^n$, $X\in \R^{n\times d}$. 
	\begin{enumerate}[1.]
		\item[0.] Preprocess $y,X$ as in \cref{eq:median-preprocess} and let  $(\bm y', \bm X')$ be the resulting pair.
		\noindent\rule{14cm}{0.21pt}
		\item For all $j\in [d]$ run \cref{alg:univariate-linear-regression-via-median} on input $(\bm y,\bm X'_j)$, where $\bm X'_j$ is a $j$-th column of $\bm X'$ (without additional preprocessing). 
		Let $\hat{\bm \beta}_j$ be the resulting estimate.
		\item Return $\hat{\bm \beta}:= \transpose{\Paren{\hat{\bm \beta}_1,\ldots,\hat{\bm \beta}_d}}$.
	\end{enumerate}
\end{algorithm}
\begin{remark}[Running time]
	Preprocessing takes linear time. Then the algorithm simply executes \cref{alg:univariate-linear-regression-via-median} $d$ times, so it runs in $O(nd)$ time.
\end{remark}

The performance of the algorithm is captured by the following theorem.

\begin{theorem}\label{thm:median-iteration-technical}
	Let $\bm y =\bm X\betastar+\eta$ for arbitrary  $\betastar\in \R^d$,  $\bm X\sim N(0,1)^{n \times d}$ and $\eta\in \R^n$  satisfying \cref{assumption:noise-median} with parameter $\alpha$.
	Let $\bm\betahat$ be the estimator computed by 
	\cref{alg:linear-regression-via-median}
	given $(\bm  y, \bm  X)$ as input. 
	Then for  any positive $\tau \le \alpha^2n$,
	\begin{align*}
	\Snorm{\betastar-\bm\betahat}\leq 
	\frac{d\cdot \tau}{\alpha^2\cdot n}\Paren{1+\Snorm{\betastar}}
	\end{align*}
	with probability at least $1-2\exp\Brac{\ln d- \Omega\Paren{\tau}}$.
	\begin{proof}
		We first show that the algorithm obtains
		a good estimate for each coordinate. 
		Then it suffices to sum the coordinate-wise errors.
		For $j \in [d]$, let $\bm\cM_j\subseteq [n]$ be the set of entries such that $\abs{\bm X_{ij}} \geq \frac{1}{2}$. 
		Observe that since $\bm \cM_j$ doesn't depend on $\cT$,  
		by Chernoff bound,  
		$\cT\cap \bm\cM_j \ge \Omega(\alpha n)$ 
		with probability at least $1-2\exp\Brac{-\Omega(\alpha n)}\ge 1-2\exp\Brac{-\Omega(\tau)}$.
		Now for all $i \in [n]$ let 
		\begin{align*}
		\bm z_{ij}:= \frac{1}{\bm X'_{ij}}\Paren{\bm\sigma_i\eta_i+\bm w_i + 
			\underset{l \neq j}{\sum}\bm X'_{il}\betastar_l}\,.
		\end{align*}
		Note that $\bbP\Paren{\bm z_{ij} \ge 0} = \bbP\Paren{\bm z_{ij} \le 0}$. Now let $\bar{\beta}\in \R^d$ be the vector such that for $j \in [d]\setminus \Set{i}$, $\bar{\beta}_j = \betastar_j$ and $\bar{\beta}_i=0$. By properties of Gaussian distribution, for all $i\in [n]$,
		\[
		\bm w_i + \underset{l\neq j}{\sum}\bm X'_{il}\betastar_l
		\sim N(0, 1+\snorm{\bar{\beta}})\,.
		\]
		Hence for each $i\in \cT\cap \bm\cM_j$, 
		for all $0 \le t \le \sqrt{1+ \snorm{\bar{\beta}}}$,
		\[
		\bbP\Paren{\abs{\bm z_{ij}}\le t }\geq \Omega\Paren{\frac{t}{\sqrt{1+ \snorm{\bar{\beta}}}}}\,.
		\]
		By \cref{lem:median-meta}, median $\bm {\hat{z}}_j$ of $\bm z_{ij}$ satisfies
		\[
		\bm {\hat{z}}_j^2 \le {\frac{\tau}{\alpha^2\cdot n}\Paren{1+\Snorm{\bar{\beta}}}} 
		\le
		 {\frac{\tau}{\alpha^2\cdot n}\Paren{1+\Snorm{\betastar}}}
		\]
		with probability at least $1-2\exp\Brac{-\Omega(\tau)}$. 
		Since $\bm\betahat_j = \bm {\hat{z}}_j+ \betastar_j$,
		applying union bound over all coordinates $j\in [d]$,  we get the desired bound.
	\end{proof}
\end{theorem}

\subsubsection{Nearly optimal estimation via bootstrapping}\label{sec:high-dimensional-bootstrapping}
Here we show how through multiple executions of \cref{alg:linear-regression-via-median} we can indeed obtain error $\Snorm{\betastar-\bm\betahat}\leq \tilde{O}\Paren{\frac{d}{\alpha^2\cdot n}}$. As already discussed, assume that we  know some upper bound on $\norm{\beta}$, which we denote by $\Delta$. Consider the following procedure:

\begin{algorithm}[H]
  \caption{Multivariate Linear Regression via Median}\label{alg:linear-regression-via-bootstrapping}
	\mbox{}\\
	\textbf{Input:} $(y, X,\Delta)$ where $X\in \R^{n\times d}$, $y\in \R^n$, and $\Delta \ge 3$.
	\begin{enumerate}[1.]
		\item Randomly partition the samples $y_1,\ldots,y_n$ in 
		$t := \lceil\ln\Delta\rceil$ sets $\bm \cS_1,\ldots, \bm\cS_{t}$, 
		such that all $\bm\cS_1,\ldots, \bm\cS_{t-1}$ have sizes
		$\Theta\Paren{\frac{n}{\log\Delta}}$ and $\bm\cS_{t}$ has size $\lfloor{n/2}\rfloor$. 
		\item 	Denote $\bm \betahat^{(0)}=0\in \R^{d}$. 
		For $i \in [t]$, run \cref{alg:linear-regression-via-median} on input 
		$\Paren{y_{\bm\cS_i} -  X_{\bm \cS_i}\bm\betahat^{(i-1)},  X_{\bm \cS_i}}$,
		 and let  $\bm \betahat^{(i)}$ be the resulting estimator.
		\item Return $\hat{\bm \beta}:=\underset{i \in [t]}{\sum}\bm \betahat^{(i)}.$
	\end{enumerate}
\end{algorithm}
\begin{remark}[Running time]
	Splitting the samples into $t$ sets requires time $O(n)$. For each set $\cS_i$, the algorithm simply executes \cref{alg:linear-regression-via-median}, so all in all the algorithm takes $O\Paren{\Theta\Paren{\frac{n}{\log \Delta}}d\cdot \log \Delta} = O\Paren{nd}$ time.
\end{remark}

The theorem below proves correctness of the algorithm.

\begin{theorem}\label{thm:median-algorithm-main}
	Let $\bm y = \bm X\betastar+\eta$ for  $\betastar\in \R^d$,  $\bm X\sim N(0,1)^{n \times d}$ and $\eta\in \R^n$  satisfying Assumption \ref{assumption:noise-median} with parameter $\alpha$. 
	Suppose that $\Delta \ge 3\Paren{1+\norm\betastar}$,
	and that for some positive  $\eps\le 1/2$,
	$n \ge C \cdot \frac{d\ln \Delta}{\alpha^2}\cdot \Paren{\ln\Paren{d/\eps} + \ln\ln\Delta}$ for sufficiently large absolute constant $C>0$.
	Let $\bm \betahat$ be the estimator computed by \cref{alg:linear-regression-via-bootstrapping} given $(\bm y,\bm X, \Delta)$ as input.
	Then, 	with probability at least $1-\eps$,
	\begin{align*}
	\frac{1}{n}\Snorm{\bm X\Paren{\betastar- \bm\betahat}}\leq 
    	O\Paren{\frac{d\cdot \log \Paren{d/\eps}}{\alpha^2\cdot n}}\,.
	\end{align*}
	\begin{proof}
		Since $n \ge C\cdot \frac{d\ln \Delta}{\alpha^2}\cdot \Paren{\ln d + \ln\ln\Delta}$, 
		by \cref{thm:median-iteration-technical},
		for each $i\in [t-1]$,
		\[
	    \Snorm{\betastar - \sum_{j=1}^{i}\bm\betahat^{(j)}}\leq \frac{1+\Snorm{\betastar - \sum_{j=1}^{i-1}\bm\betahat^{(j)}}}{10}\,,
		\]
    	with probability at least $1-2\exp\Brac{\ln d -10\ln\Paren{d/\eps} - 10\ln\ln\Delta}$. 
    	By union bound over $i\in[t-1]$, with probability at least $1-2\eps^{10}$,
    	\[
    	\Snorm{\betastar - \sum_{j=1}^{t-1}\bm\betahat^{(j)}}\le 100\,.
    	\]
    	 Hence  by \cref{thm:median-iteration-technical}, with probability  at least $1-4\eps^{10}$,
    	\[
    	\Snorm{\betastar - \sum_{j=1}^{t}\bm\betahat^{(j)}}\le 
    	O\Paren{\frac{d\cdot \log \Paren{d/\eps}}{\alpha^2\cdot n}}\,.
    	\]

    	 By \cref{fact:gaussian-sample-covariance}, with probability at least $1-\varepsilon^{10}$,
    	 \[
    	\frac{1}{n}\Snorm{\bm X\Paren{\betastar- \bm\betahat}} = 
    	\transpose{\Paren{\betastar- \bm\betahat}}\Paren{\frac{1}{n}\transpose{\bm X}\bm X}
    	\Paren{\betastar- \bm\betahat} \le 
    	1.1\cdot \Snorm{\betastar- \bm\betahat}\,.
    	 \]
	\end{proof}
\end{theorem}

\subsubsection{Nearly optimal sparse estimation}\label{sec:high-dimensional-iteration-sparse}
A slight modification of \cref{alg:linear-regression-via-median} can be use in the sparse settings.

\begin{algorithm}[H]
  \caption{Multivariate Sparse Linear Regression Iteration via Median}
  \label{alg:linear-regression-via-median-sparse}
	\mbox{}\\
	Input: $(y, X)$, where $X\in \R^{n\times d}$, $y\in \R^n$.
	\begin{enumerate}[1.]
		\item Run \cref{alg:linear-regression-via-median}, let $\bm \beta'$ be the resulting estimator.
		\item 
		Denote by $\bm a_k$ the value of the $k$-th largest 
		(by absolute value) coordinate of $\bm \beta'$.
		For each $j \in [d]$, 
		let $\bm \betahat_j = \bm \beta'_j$ 	if $\abs{\bm \beta'_j} \ge \bm a_k$,
		and $\bm \betahat_j  = 0$ otherwise.
		\item Return $\hat{\bm  \beta}:= \transpose{\Paren{\hat{\bm \beta}_1,\ldots,\hat{\bm \beta}_d}}$.
	\end{enumerate}
\end{algorithm}
\begin{remark}[Running time]
	Running time of \cref{alg:linear-regression-via-median} is $O(nd)$. 
	Similar to median, 
	$a_k$ can be computed in time $O(d)$ (for example, using procedure from \cite{blum1973time}).
\end{remark}

The next theorem is the sparse analog of \cref{thm:median-iteration-technical}.

\begin{theorem}\label{thm:median-iteration-technical-sparse}
	Let $\bm y = \bm X\betastar+\eta$ for $k$-sparse  $\betastar\in \R^d$,  
	$\bm X\sim N(0,1)^{n \times d}$ and 
	$\eta\in \R^n$  satisfying Assumption \ref{assumption:noise-median} with parameter $\alpha$. 
	 Let $\bm \betahat$ be the estimator computed by
	  \cref{alg:linear-regression-via-median-sparse} given $(\bm y, \bm X)$ as input.
	 Then for any positive $\tau \le \alpha^2 n$,
	 \[
	\Snorm{\betastar-\bm\betahat}\leq 
		  O\Paren{ \frac{k\cdot\tau}{\alpha^2\cdot n}\Paren{1+\Snorm{\betastar}}}
	\]
	with probability at least $1-2\exp\Brac{\ln d - \Omega(\tau)}$.
	\begin{proof}
	    The reasoning of \cref{thm:median-iteration-technical} 
		shows that for each coordinate $[j]$, the median $\bm {\hat{z}}_j$ satisfies 
		\[
		\abs{\bm {\hat{z}}_j}\le
		\sqrt{\frac{\tau}{\alpha^2\cdot n}\Paren{1+\Snorm{\betastar}}}
		\]
		with probability at least 
		$1-2\exp\Brac{-\Omega\Paren{\tau}}$. 
		By union bound over $j\in[d]$,  with probability at least $1-2\exp\Brac{\ln d - \Omega(\tau)}$, for any $j\in[d]$,
		 \[
		\abs{\bm\beta'_j - \betastar_j} \le
		{\sqrt{\frac{\tau}{\alpha^2\cdot n}\Paren{1+\Snorm{\betastar}}}}\,.
		\]
		If $\bm \beta'_j < a_k$, then there should be some $i\notin\supp\set{\beta^*}$ such that $\abs{\bm \beta'_i} \ge \abs{\bm \beta'_j}$. Hence for such $j$,
		\[
		\abs{\betastar_j} \le 
		\abs{\bm\beta'_j - \betastar_j} +\abs{\bm\beta'_j} \le
		\abs{\bm\beta'_j - \betastar_j}    + \abs{\bm \beta'_i - \betastar_i} \le 
		O\Paren{\sqrt{\frac{\tau}{\alpha^2\cdot n}\Paren{1+\Snorm{\betastar}}}}\,.
		\]
		Note that since random variables $\bm X_{ij}$ 
		are independent and absolutely continuous with positive density,
		$\bm\beta'_j \neq \bm\beta'_m$ for $m\neq j$  with probability $1$. 
		Hence $\Card{\Set{j\in [d]\suchthat \abs{\bm \beta'_j} \ge a_k}} = k$. It follows that
		  \begin{align*}
		  \Snorm{\betastar-\bm \betahat}
		  &\leq 
		  \sum_{j\in \supp\set{\betastar}} \ind{\abs{\bm \beta'_j} < a_k}\cdot \Paren{\betastar_j}^2 +
		  \sum_{j=1}^d \ind{\abs{\bm \beta'_j} \ge a_k}\cdot 
		  \Paren{\betastar_j - \bm\beta'_j}^2
		  \\&\leq
		  \sum_{j\in \supp\set{\betastar}}
		  O\Paren{ \frac{\tau}{\alpha^2\cdot n}\Paren{1+\Snorm{\betastar}}}
		  +
		  \sum_{j=1}^d \ind{\abs{\bm \beta'_j} \ge a_k}\cdot  
		  O\Paren{ \frac{\tau}{\alpha^2\cdot n}\Paren{1+\Snorm{\betastar}}}
		 \\&\leq
		 		  O\Paren{ \frac{k\cdot\tau}{\alpha^2\cdot n}\Paren{1+\Snorm{\betastar}}}\,.
		  \end{align*}
	\end{proof}
\end{theorem}

Again through bootstrapping we can obtain a nearly optimal estimate. However, for the first few iterations we need a different subroutine. Instead of taking the top-$k$ entries, we will zeros all entries smaller some specific value.
\begin{algorithm}[H]
  \caption{Multivariate Sparse Linear Regression Iteration via Median}\label{alg:linear-regression-via-median-sparse-subroutine}
	\mbox{}\\
	\textbf{Input:} $(y, X, \Delta)$, where $X\in \R^{n\times d}$, $y\in \R^n$, $\Delta > 0$.
	\begin{enumerate}[1.]
		\item Run \cref{alg:linear-regression-via-median}, let $\bm \beta'$ be the resulting estimator.
		\item 
		For each $j \in [d]$, 
		let $\bm \betahat_j = \bm \beta'_j$ 	if $\abs{\bm \beta'_j} \ge 
		\frac{1}{100\sqrt{k}}\Delta$,
		and $\bm \betahat_j  = 0$ otherwise.
		\item Return $\hat{\bm  \beta}:= \transpose{\Paren{\hat{\bm \beta}_1,\ldots,\hat{\bm \beta}_d}}$.
	\end{enumerate}
\end{algorithm}
\begin{remark}[Running time]
	The running time of this algorithm is the same as the running time of \cref{alg:linear-regression-via-median}, i.e. $O(nd)$.
\end{remark}

The following theorem proves correctness of \cref{alg:linear-regression-via-median-sparse-subroutine}.
\begin{theorem}\label{thm:median-iteration-technical-sparse-subroutine}
	Let $\bm y = \bm X\betastar+\eta$ for $k$-sparse  $\betastar\in \R^d$,  
	$\bm X\sim N(0,1)^{n \times d}$ and 
	$\eta\in \R^n$  satisfying Assumption \ref{assumption:noise-median} with parameter $\alpha$. 
	Suppose that 
	$\norm{\betastar}\leq \Delta$. Let $\bm \betahat$ be the estimator computed by
	\cref{alg:linear-regression-via-median-sparse-subroutine} given $(\bm y, \bm X, \Delta)$ as input.
	Then, with probability at least  
	$1- 2\exp\Brac{
		\ln d-\Omega\Paren{\frac{\alpha^2\cdot n\cdot \Delta^2}{k\Paren{1+\Delta^2}}}}$,
	$\supp\set{\bm \betahat}\subseteq \supp\set{\betastar}$ and
	\[
	\Norm{\betastar-\bm\betahat}\leq \frac{\Delta}{10}\,.
	\]
	\begin{proof}
		Fix a coordinate $j \in [d]$. The reasoning of \cref{thm:median-iteration-technical} 
		shows that the median $\bm {\hat{z}}_j$ satisfies 
		\[
		\bm {\hat{z}}_j^2\le
		\frac{\tau}{\alpha^2\cdot n}\Paren{1+\Snorm{\betastar}}
		\le		
		\frac{\tau}{\alpha^2\cdot n}\Paren{1+\Delta^2}
		\]
		with probability at least 
		$1-\exp\Brac{-\Omega\Paren{\tau}} - \exp\Brac{-\Omega\Paren{\alpha n}}$. 
		If $\betastar_i=0$ then with probability at least 
		$1-2\exp\Brac{-\Omega\Paren{\frac{\alpha^2\cdot n\cdot \Delta^2}{k\Paren{1+\Delta^2}}}}$
		we have 
		$\Abs{\bm {\hat{z}}_j}\leq \frac{\Delta}{100\sqrt{k}}$, so $\bm\betahat_j = 0$.
		Conversely if $\betastar_i\neq 0$ then with probability
		$1-2\exp\Brac{-\Omega\Paren{\frac{\alpha^2\cdot n\cdot \Delta^2}{k\Paren{1+\Delta^2}}}}$
		the error is at most $2\cdot \frac{\Delta}{100\sqrt{k}}$.
		Combining the two and repeating the argument for all $i \in [d]$, we get that by union bound,
		with probability at least 
		$1- 2\exp\Brac{
			\ln d-\Omega\Paren{\frac{\alpha^2\cdot n\cdot \Delta^2}{k\Paren{1+\Delta^2}}}}$,
		$\Snorm{\betastar-\bm \betahat}\leq 
		k \cdot 4 \cdot \frac{\Delta^2}{10000\cdot k}\leq \frac{\Delta^2}{100}$.
\end{proof}
\end{theorem}

Now, combining \cref{alg:linear-regression-via-median-sparse} and \cref{alg:linear-regression-via-median-sparse-subroutine} we can introduce the full algorithm.

\begin{algorithm}[H]
  \caption{Multivariate Sparse Linear Regression via Median}
	\label{alg:linear-regression-via-bootstrapping-sparse}
	\mbox{}\\
	\textbf{Input: }$(y, X,\Delta)$ where $X\in \R^{n\times d}$, $y\in \R^n$, and $\Delta \ge 3$.
	\begin{enumerate}[1.]
		\item Randomly partition the samples $y_1,\ldots,y_n$ in 
		$t := \lceil\ln\Delta\rceil$ sets $\bm \cS_1,\ldots, \bm\cS_{t}$, 
		such that all $\bm\cS_1,\ldots, \bm\cS_{t-1}$ have sizes
		$\Theta\Paren{\frac{n}{\log\Delta}}$ and $\bm\cS_{t}$ has size $\lfloor{n/2}\rfloor$. 
		\item 	Denote $\bm \betahat^{(0)}=0\in \R^{d}$ and $\Delta_0 = \Delta$. 
		For $i \in [t-1]$, run \cref{alg:linear-regression-via-median-sparse-subroutine} on input 
		\[
		\Paren{y_{\bm\cS_i} -  X_{\bm \cS_i}\bm\betahat^{(i-1)},  X_{\bm \cS_i}, \Delta_{i-1}}\,.
		\]
		Let $\bm \betahat^{(i)}$ be the resulting estimator and $\Delta_{i} = \Delta_{i-1}/2$.
		\item Run \cref{alg:linear-regression-via-median-sparse} on input 
		$\Paren{y_{\bm\cS_t} -  X_{\bm \cS_t}\bm\betahat^{(t-1)},  X_{\bm \cS_t}}$,
		and let $\bm \betahat^{(t)}$ be the resulting estimator.
		\item Return $\hat{\bm \beta}:=\underset{i \in [t]}{\sum}\bm \betahat^{(i)}.$
	\end{enumerate}
\end{algorithm}
\begin{remark}[Running time]
	Splitting the samples into $t$ sets requires time $O(n)$. For each set $\cS_i$, 
	the algorithm simply executes
	either \cref{alg:linear-regression-via-median-sparse-subroutine} or
	\cref{alg:linear-regression-via-median-sparse}, 
	so all in all the algorithm takes 
	$O\Paren{\Theta\Paren{\frac{n}{\log \Delta}}d\log \Delta + O(nd)} = O\Paren{nd}$ time.
\end{remark}

Finally, \cref{thm:median-algorithm-informal} follows from the result below.

\begin{theorem}\label{thm:median-algorithm-main-sparse}
	Let $\bm y = \bm X\betastar+\eta$ for  $k$-sparse $\betastar\in \R^d$,  
	$\bm X\sim N(0,1)^{n \times d}$ and $\eta\in \R^n$  satisfying 
	Assumption \ref{assumption:noise-median} with parameter $\alpha$. 
	Suppose that $\Delta \ge 3\Paren{1+\norm\betastar}$,
	and that for some positive  $\eps<1/2$,
	$n \ge C \cdot \frac{k\ln \Delta}{\alpha^2}\cdot \Paren{\ln\Paren{d/\eps} + \ln\ln\Delta}$ for sufficiently large absolute constant $C>0$.
	Let $\bm \betahat$ be the estimator computed by \cref{alg:linear-regression-via-bootstrapping-sparse} given $(\bm y,\bm X, \Delta)$ as input.
	Then, 	with probability at least $1-\eps$,
	\begin{align*}
	\frac{1}{n}\Snorm{\bm X\Paren{\betastar- \bm\betahat}}
	\le
	O\Paren{\frac{k\cdot \log \Paren{d/\eps}}{\alpha^2\cdot n}}\,.
	\end{align*}
	\begin{proof}
		Since $n \ge C\cdot \frac{k\ln \Delta}{\alpha^2}\cdot \Paren{\ln d + \ln\ln\Delta}$, 
		by \cref{thm:median-iteration-technical-sparse-subroutine} and union bound over $i\in[t-1]$,
	    with probability at least $1-2\exp\Brac{\ln d + \ln t -10\ln\Paren{d/\eps} - 10\ln\ln\Delta}$,
		for each $i\in [t-1]$,
		\[
		\Norm{\betastar - \sum_{j=1}^{i}\bm\betahat^{(j)}}\leq \frac{\Delta}{10^{i}}\,.
		\]
		Hence
		\[
		\Snorm{\betastar - \sum_{j=1}^{t-1}\bm\betahat^{(j)}}\le 100
		\]
		with probability $1-2\eps^{10}$. Therefore,  by \cref{thm:median-iteration-technical-sparse},
		\[
		\Snorm{\betastar - \sum_{j=1}^{t}\bm\betahat^{(j)}}\le 
		O\Paren{\frac{k\cdot \log \Paren{d/\eps}}{\alpha^2\cdot n}}
		\]
		with probability $1-4\eps^{10}$.
		
		 Since $\bm\betahat^{(t)}$ is $k$-sparse and with probability $1-2\eps^{10}$, $\supp\Set{\sum_{j=1}^{t-1}\bm\betahat^{(j)}} \subseteq \supp\Set{\betastar}$, 
		 vector $\betastar-\bm\betahat$ is $2k$-sparse. 
		 By \cref{lem:gaussian-sample-covariance-sparse}, with probability at least $1-\varepsilon^{10}$,
		\[
		\frac{1}{n}\Snorm{\bm X\Paren{\betastar- \bm\betahat}} = 
		\transpose{\Paren{\betastar- \bm\betahat}}\Paren{\frac{1}{n}\transpose{\bm X}\bm X}
		\Paren{\betastar- \bm\betahat} \le 
		1.1\cdot \Snorm{\betastar- \bm\betahat}\,.
		\]
	\end{proof}
\end{theorem}

\subsubsection{Estimation for non-spherical Gaussians}\label{sec:high-dimensional-nonspherical}
We further extend the results to non-spherical Gaussian design. In this section we assume $n \ge d$.
We use the algorithm below. 
We will assume to have in input an estimate of the covariance matrix of the rows of $\bm X$:
$\bm x_1,\ldots, \bm x_n\sim N(0, \Sigma)$. 
For example, if number of samples is large enough, 
sample covariance matrix is a good estimator of $\Sigma$. 
For more details, see \cref{sec:estimate-covariance-matrix}.
\begin{algorithm}[H]
  \caption{Multivariate Linear Regression Iteration via Median for Non-Spherical Design}
  \label{alg:linear-regression-via-median-nonspherical}
	\mbox{}\\
		\textbf{Input: }$\Paren{y,X, \hat{\Sigma}}$, where $X\in \R^{n\times d}$, $y\in \R^n$, $\hat{\Sigma}$ is a positive definite symmetric matrix.
	\begin{enumerate}[1.]
		\item Compute $ \tilde{X}=X \hat{\Sigma}^{-1/2}$.
		\item Run \cref{alg:linear-regression-via-median} on input $(y, \tilde{X})$ and let $\bm \beta'$ be the resulting estimator. 
		\item Return $\hat{\bm  \beta} := \hat{\Sigma}^{-1/2}\bm  \beta'$.
	\end{enumerate}
\end{algorithm}
\begin{remark}[Running time]
	Since $n \ge d$, computing $\tilde{X}$ requires $O(nT(d)/d)$, 
	where $T(d)$ is a time required for multiplication of two $d\times d$ matrices. 
	\cref{alg:linear-regression-via-median}  runs in time $O(nd)$, so the running time is $O(nT(d)/d)$.
\end{remark}

The performance of the algorithm is captured by the following theorem.

\begin{theorem}\label{thm:median-iteration-technical-non-spherical}
	Let $\bm y =\bm X\betastar + \eta$, such that
	rows of $\bm X$ are iid $\bm x_i\sim N(0, \Sigma)$, 
	$\eta$ satisfies  \cref{assumption:noise-median} with parameter $\alpha$, 
	and $\hat{\Sigma}\in \R^{d\times d}$ is a symmetric matrix independent of $\bm X$ such that
	$\norm{\Sigma^{1/2} \hat{\Sigma}^{-1/2} - \Id_d} \le \delta$ 
	for some $\delta > 0$. 
	Suppose that for some $N \ge n+d$, $\delta \le \frac{1}{100\sqrt{\ln N}}$ and 
	$\alpha^2 n\ge 100\ln N$. 
	Let $\bm \betahat$ be the estimator computed by
	 \cref{alg:linear-regression-via-median-nonspherical} given  
	$(\bm y, \bm X, \hat{\Sigma})$ as input.  
	Then, with probability at least $1-O\Paren{N^{-5}}$, 
	\begin{align*}
		\Norm{\hat\Sigma^{1/2}\Paren{\betastar-\bm\betahat}}^2\leq 
		O\Paren{\frac{d\cdot \log N}{\alpha^2\cdot n}\Paren{1+\Snorm{\hat\Sigma^{1/2}\betastar}}
	+\delta^2\cdot \snorm{\hat\Sigma^{1/2}\beta^*}}\,.
	\end{align*}
	
	\begin{proof}
		We first show that the algorithm obtains a good estimate for each coordinate. Then it suffices to add together the coordinate-wise errors. By assumptions on $\hat\Sigma$,
		\[
		\tilde{\bm  X} = \bm  X\hat{\Sigma}^{-1/2} = \bm  G\Sigma^{1/2}\hat{\Sigma}^{-1/2} = \bm  G+ \bm  G E\,,
		\]
		where $\bm  G\sim N(0,1)^{n\times d}$ and $E$ is a  matrix such that $\Norm{E} \le \delta$ for some  $\delta \le \frac{1}{100\sqrt{\ln N}}$. 
		Since $E$ is independent of $\bm  X$ and each column of $E$ has norm at most $\delta$, 
		for all $i\in[n]$ and $j\in [d]$,
		$\abs{\tilde{\bm  X}_{ij} - \bm  G_{ij}} \le 40\delta \sqrt{\ln N}$ 
		with probability $1-O\Paren{N^{-10}}$.
		For simplicity, we still write $\bm y, \tilde{\bm  X}$ after the preprocessing step.
		Fix $j \in [d]$, let $\bm \cM_j\subseteq [n]$ be the set of entries such that $\abs{\tilde{\bm  X}_{ij}} \geq 1/2$.
		With probability $1-O\Paren{N^{-10}}$,
		$\Set{i\in [n]\suchthat \abs{\bm G_{ij}} \ge 1}$ is a subset of $\bm \cM_j$.
		Hence by Chernoff bound, 
		$\Card{\cT\cap\bm \cM_j}\geq \Omega(\alpha n)$ with probability at least $1-O\Paren{N^{-10}}$.
		Now for all $i \in \bm \cM_j$ let 
		\begin{align*}
		\bm q_{ij} &:= \frac{1}{\tilde{\bm  X}_{ij}}
		\Paren{\bm \eta_i+ \underset{l\neq j}{\sum}
			\tilde{\bm  X}_{il}\Paren{\hat\Sigma^{1/2}\betastar}_l}
		\\&=
		\frac{1}{\tilde{\bm  X}_{ij}}\Paren{\bm \eta_i+ \underset{l\neq j}{\sum}\bm  G_{il}\Paren{\hat\Sigma^{1/2}\betastar}_l + 
		 \underset{l\neq j}{\sum}\sum_{m\neq j} \bm G_{im} E_{ml}\Paren{\hat\Sigma^{1/2}\betastar}_l
	    +\sum_{l\neq j} \bm G_{ij}E_{jl}\Paren{\hat\Sigma^{1/2}\betastar}_l}
		\,.
		\end{align*}
		 Note that for any $i\in \bm \cM_j$, with probability $1-O\Paren{N^{-10}}$, $\sign\Paren{\tilde{\bm  X}_{ij}} = \sign\Paren{\bm G_{ij}}$. Hence
		 \[
		 \bm z_{ij} := \frac{1}{\tilde{\bm  X}_{ij}}\Paren{\bm\sigma_i \eta_i+ \bm w_i +
		 	\underset{l\neq j}{\sum}\bm G_{il}\Paren{\hat\Sigma^{1/2}\betastar}_l + 
		 	\underset{l\neq j}{\sum}\sum_{m\neq j} \bm G_{im} E_{ml}
		 	\Paren{\hat\Sigma^{1/2}\betastar}_l}
		 \]
		 is symmetric about zero.
		
		Now let $\bar{\beta}\in \R^d$ be the vector such that for $l \in [d]\setminus \Set{j}$, $\bar{\beta}_l = \Paren{\hat\Sigma^{1/2}\betastar}_l$ and $\bar{\beta}_i=0$. Note that $\norm{\bar{\beta}}\leq \norm{\hat\Sigma^{1/2}\betastar}$. 
		By properties of Gaussian distribution,
		\[
		  \bm w_i  + \underset{l\neq j}{\sum}\bm G_{il}\Paren{\hat\Sigma^{1/2}\betastar}_l + 
		\underset{l\neq j}{\sum}\sum_{m\neq j} \bm G_{im}E_{ml}
		\Paren{\hat\Sigma^{1/2}\betastar}_l
	     \sim N(0, \sigma^2)\,,
		\]
		where $1 + \snorm{\bar{\beta}} \le \sigma^2 \le 1+ \Paren{1+\delta^2}\snorm{\bar{\beta}}$.
			Hence for each $i\in \cT\cap \bm\cM_j$, 
		 for all $0 \le t \le \sqrt{1+ \snorm{\bar{\beta}}}$,
		\[
		\bbP\Paren{\abs{\bm z_{ij}}\le t }\geq \Omega\Paren{\frac{t}{\sqrt{1+ \snorm{\bar{\beta}}}}}\,.
		\]
		By \cref{lem:median-meta}, median $\bm {\hat{z}}_j$ of 
		$\Set{\bm z_{ij}}_{i\in \bm\cM_j}$ satisfies
		\[
		\bm {\hat{z}}_j^2 \le 
		O\Paren{\frac{\log N}{\alpha^2\cdot n}\Paren{1+\Snorm{\bar{\beta}}}} 
		\le
		O\Paren{\frac{\log N}{\alpha^2\cdot n}\Paren{1+\Snorm{\hat\Sigma^{1/2}\betastar}}}
		\]
		with probability at least $1-O\Paren{N^{-10}}$. 
		
		For any $i\in \bm \cM_j$, the event
		\[
		\Abs{\frac{\bm G_{ij}}{\tilde{\bm  X}_{ij}}\sum_{l\neq j} E_{jl}\Paren{\hat\Sigma^{1/2}\betastar}_l}
		\le 
		O\Paren{\delta\cdot \norm{\hat\Sigma^{1/2}\beta^*}}
		\]
		occurs with probability $1-O\Paren{N^{-10}}$. Moreover, since $\norm{E}\le \delta$,
		with probability $1-O\Paren{N^{-10}}$, for all $i_1,\ldots,i_d \in \bm\cM_j$,
		\[
		\sum_{j=1}^d\Paren{\frac{\bm G_{i_jj}}{\tilde{\bm  X}_{i_jj}}\sum_{l\neq j} E_{jl}
					\Paren{\hat\Sigma^{1/2}\betastar}_l}^2 \le 
		O\Paren{\delta^2\cdot \snorm{\hat\Sigma^{1/2}\beta^*}}\,.
		\] 
		Therefore, with probability $1-O\Paren{N^{-9}}$,
		medians $\hat{\bm q}_{j}$ 
		of $\Set{\bm q_{ij}}_{i\in \bm\cM_j}$ satisfiy
		\[
		\sum_{j=1}^d\hat{\bm q}_j^2 \le
		O\Paren{\frac{d\cdot \log n}{\alpha^2\cdot n}\Paren{1+\Snorm{\hat\Sigma^{1/2}\betastar}}
        +\delta^2\cdot \snorm{\hat\Sigma^{1/2}\beta^*}}\,.
		\]
		Since $\bm y_i/\tilde{\bm X}_{ij} = \Paren{\hat\Sigma^{1/2}\betastar}_j + \bm q_{ij}$,
		\[
		\sum_{j=1}^d\Paren{\bm \beta'_j - \Paren{\hat\Sigma^{1/2}\betastar}_j }^2 \le
		O\Paren{\frac{d\cdot \log n}{\alpha^2\cdot n}\Paren{1+\Snorm{\hat\Sigma^{1/2}\betastar}}
	+\delta^2\cdot \snorm{\hat\Sigma^{1/2}\beta^*}}\,.
		\]
		
%
	\end{proof}
\end{theorem}

Next we show how to do bootstrapping for this general case. In this case we will assume to know  an upper bound $\Delta$  of $\norm{\bm X\betastar}$.

\begin{algorithm}[H]
  \caption{Multivariate Linear Regression via Median for Non-Spherical Design}
  \label{alg:linear-regression-via-bootstrapping-nonspherical}
	\mbox{}\\
	\textbf{Input:} $\Paren{y, X, \hat{\Sigma}, \Delta}$, where $X\in \R^{n \times d}$, 
	$y \in \R^n, \Delta\geq 3$, and 
	$\hat{\Sigma}\in \R^{d\times d}$ is a positive definite symmetric matrix. 
	\begin{enumerate}[1.]
				\item Randomly partition the samples $y_1,\ldots,y_n$ in 
		$t := t_1 + t_2$, sets $\bm \cS_1,\ldots, \bm\cS_{t}$, 
		where  $t_1 = \lceil\ln\Delta\rceil$ and $t_2 = \lceil \ln n \rceil$,
		such that all $\bm\cS_1,\ldots, \bm\cS_{t_1}$ have sizes
		$\Theta\Paren{\frac{n}{\log\Delta}}$ and $\bm\cS_{t_1+1},\ldots \bm\cS_{t_2}$ 
		have sizes $\Theta\Paren{\frac{n}{\log n}}$. 
		\item 	Denote $\bm \betahat^{(0)}=0\in \R^{d}$ and $\Delta_0 = \Delta$. 
		For $i \in [t]$, run \cref{alg:linear-regression-via-median-nonspherical} on input 
		\[
		\Paren{y_{\bm\cS_i} -  X_{\bm \cS_i}\bm\betahat^{(i-1)},  X_{\bm \cS_i}, \hat{\Sigma}}\,,
		\]
		and let $\bm \betahat^{(i)}$ be the resulting estimator.
		\item Return $\hat{\bm\beta}:=\underset{i \in [t]}{\sum}\bm\betahat^{(t)}.$
	\end{enumerate}
\end{algorithm}
\begin{remark}[Running time]
	Running time is $O(nT(d)/d)$, 
	where $T(d)$ is a time required for multiplication of two $d\times d$ matrices. 
\end{remark}

The theorem below extend \cref{thm:median-algorithm-informal} to non-spherical Gaussians.

\begin{theorem}\label{thm:median-algorithm-main-nonsperhical}
	Let $\bm y = \bm X\betastar+\eta$ for   $\betastar\in \R^d$,  $\bm X\in \R^{n \times d}$ with 
	iid rows $\bm x_i\sim N(0, \Sigma)$, 
	$\eta$ satisfying  \cref{assumption:noise-median} with parameter $\alpha$.
	Let $\hat{\Sigma}\in \R^{d\times d}$ be a positive definite
	symmetric matrix independent of $\bm X$.
	
	Denote by $\sigma_{\min}$, $\sigma_{\max}$ and $\kappa$ the smallest singular value, 
	the largest singular value, and the condition number of $\Sigma$.
	Suppose that  $\Delta\geq 3\Paren{1+\norm{\bm X\betastar}}$,
	$n \ge 
	C \cdot \Paren{\frac{d\ln n}{\alpha^2}\cdot \ln \Paren{\Delta\cdot  n} + 
		\Paren{d+\ln n} \kappa^2 \ln n}$ 
	for some large enough absolute constant $C$, 
	and 
	$\norm{\hat{\Sigma}-\Sigma} \le \frac{\sigma_{\min}}{C\sqrt{\ln n}}$.
	
	Let $\bm \betahat$ be the estimator computed by
	\cref{alg:linear-regression-via-bootstrapping-nonspherical} given 
	$(\bm y, \bm X, \hat{\Sigma}, \Delta)$ as input.
	Then
	\begin{align*}
	\frac{1}{n}\Snorm{\bm X\Paren{\betastar- \bm \betahat}}\leq 
	O\Paren{\frac{d\cdot\ln^2 n}{\alpha^2\cdot n}}
	\end{align*}
	with probability $1-o(1)$ as $n\to \infty$.
	
		\begin{proof}
		Let's show that $\norm{\Sigma^{1/2} \hat{\Sigma}^{-1/2} - \Id_d} \le \frac{1}{100\sqrt{\ln n}}$.
		Since $\norm{\hat{\Sigma}-\Sigma} \le \frac{\sigma_{\min}}{C\sqrt{\ln n}}$,
		\[
		\Norm{\Sigma^{-1}\hat{\Sigma} - \Id_d} \le 
		\frac{1}{C\sqrt{\ln n}}\,.
		\]
		So $\Sigma^{-1}\hat{\Sigma} = \Id_d +  E$, 
		where $\norm{E} \le \frac{1}{C\sqrt{\ln n}}$.
		Hence for large enough $C$,
		\[
		\Norm{\Sigma^{1/2} \hat{\Sigma}^{-1/2} - \Id_d}= 
		\Norm{\Paren{\Id_d +  E}^{-1/2}-\Id_d} \le \frac{1}{100\sqrt{\ln n}}\,.
		\]  
		
		Since $n \ge C\cdot \frac{d\ln n}{\alpha^2}\cdot \Paren{\ln d + \ln\Delta}$ and and $\delta < \frac{1}{C\sqrt{\ln n}}$,
		applying \cref{thm:median-iteration-technical-non-spherical} with $N=2n$,
		for each $i\in [t]$,
		\[
		\Snorm{\hat\Sigma^{1/2}\betastar - \sum_{j=1}^{i}\hat\Sigma^{1/2}\bm\betahat^{(j)}}\leq \frac{1}{10}\Paren{1+\Snorm{\hat\Sigma^{1/2}\betastar -
		\sum_{j=1}^{i-1}\hat\Sigma^{1/2}\bm\betahat^{(j)}}}\,,
		\]
		with probability at least $1-O\Paren{n^{-5}}$. 
		By union bound over $i\in[t_1]$, with probability at least $1-O\Paren{n^{-4}}$,
		\[
		\Snorm{\hat\Sigma^{1/2}\betastar - \sum_{j=1}^{t_1}\hat\Sigma^{1/2}\bm\betahat^{(j)}}
		\le 100\,.
		\]
		Hence by \cref{thm:median-iteration-technical-non-spherical}, by union bound over $i\in[t]\setminus[t_1]$,
		with probability  at least  $1-O\Paren{n^{-4}}$,
		\[
		\Snorm{\hat\Sigma^{1/2}\betastar - \sum_{j=1}^{t}\hat\Sigma^{1/2}\bm\betahat^{(j)}}\le 
		O\Paren{\frac{d\cdot \log n}{\alpha^2\cdot \Paren{n/\log n}} + \frac{1}{n}}
		\le
		O\Paren{\frac{d\cdot \log^2 n}{\alpha^2\cdot n}}
		\,.
		\]
		
		By \cref{fact:gaussian-sample-covariance}, with probability at least $1-O(n^{-4})$, for large enough $C$,
		\[
		\Norm{\frac{1}{n}\transpose{\bm X}\bm X - \hat{\Sigma}} \le 
		\Norm{\transpose{\frac{1}{n}\bm X}\bm X - \Sigma} + \Norm{\Sigma - \hat{\Sigma}}
		\le \frac{\sigma_{\min}}{100\sqrt{\ln n}} + \frac{\sigma_{\min}}{C\sqrt{\ln n}} \le
		\frac{\sigma_{\min}}{10}
		\le 
				\frac{\sigma_{\min}\Paren{\hat{\Sigma}}}{5}
		\,,
		\]
		where  $\sigma_{\min}\Paren{\hat{\Sigma}}$ is the smallest singular value of $\hat{\Sigma}$.
		Hence
		\[
		\frac{1}{n}\Snorm{\bm X\Paren{\betastar- \bm\betahat}} = 
		\transpose{\Paren{\betastar- \bm\betahat}}\Paren{\frac{1}{n}\transpose{\bm X}\bm X}
		\Paren{\betastar- \bm\betahat} \le 
		1.2	\Snorm{\hat{\Sigma}^{1/2}\Paren{\betastar- \bm\betahat}}
		\,.
		\]
	\end{proof}
\end{theorem}

\subsubsection{Estimating covariance matrix}\label{sec:estimate-covariance-matrix}
\cref{alg:linear-regression-via-bootstrapping-nonspherical} requires $\hat{\Sigma}\in \R^{d\times d}$ which is a symmetric matrix independent of $\bm X$ such that
$\norm{\hat{\Sigma}-\Sigma} \lesssim \frac{\sigma_{\min}}{\sqrt{\ln n}}$.
The same argument as in the proof of \cref{thm:median-algorithm-main-nonsperhical} shows that
if 	$n \ge C \cdot	\Paren{d+\ln n} \kappa^2\ln n$ for some large enough absolute constant $C>0$, then with probability at least $1-O\Paren{n^{-4}}$ 
the sample covariance matrix $\hat{\bm \Sigma}$ of 
$\bm x_1,\ldots,\bm x_{\lfloor n/2\rfloor}$ 
satisfies the desired property.
So we can use \cref{alg:linear-regression-via-bootstrapping} with design matrix 
 $\bm x_{\lfloor n/2\rfloor + 1}, \ldots, \bm x_n$ and covariance estimator 
 $\hat{\bm \Sigma}$. Computation of $\hat{\bm \Sigma}$ takes time $O(nd^2)$.

\section{Bounding the Huber-loss estimator via first-order conditions}\label{sec:huber_loss}
In this section we study the guarantees of the Huber loss estimator via first order optimality condition. Our analysis exploits the connection with high-dimensional median estimation as described in \cref{sec:techniques-connection-huber-median} and yields guarantees comparable to \cref{thm:huber-loss-informal} (up to logarithmic factors) for slightly more general noise assumptions.

We consider the  linear regression model
\begin{align}
	\bm y = X\betastar+\bm \eta
\end{align}  where $X\in \R^{n\times d}$ is a deterministic matrix, $\betastar\in \R^d$ is a deterministic vector and $\bm \in \R^n$ is a random vector satisfying the assumption below.

\begin{assumption}[Noise assumption]\label{assumption:noise-huber}
	Let $\bm\cR\subseteq [n]$ be a set chosen uniformly at random among all sets of size\footnote{For simplicity we assume that $\alpha n$ is integer.} $\alpha n$. Then $\bm \eta\in \R^n$ is a random vector  such that for all $i \in [n]$, 
	$\bm\eta_i$ satisfies:
	\begin{enumerate}
		\item $\bm \eta_1,\ldots, \bm\eta_n$ are mutually conditionally independent given $\bm \cR$.
		\item For all $i\in [n]$, $\bbP \Paren{\bm \eta_i\leq 0\given \bm \cR}=
		\bbP \Paren{\bm \eta_i\geq 0\given \bm\cR}$,
		\item  For all $i \in \bm \cR$, there exists a conditional density $\bm p_i$ of $\bm\eta_i$ given $\bm \cR$ such that $\bm p_i(t) \ge 0.1$ for all $t\in[-1,1]$.
	\end{enumerate}
\end{assumption}

We remark that \cref{assumption:noise-huber} is more general than the assumptions of \cref{thm:huber-loss-informal} (see \cref{sec:error-convergence-model-assumptions}).


We will also require the design matrix $X$ to be well-spread. We restate here the definition.
\begin{definition}
	Let $V\subseteq \R^n$ be a vector space. $V$ is called \emph{$(m,\rho)$-spread}, if for every $v\in V$ and every subset $S\subseteq [n]$ with $\Card{S} \ge n-m$,
	\[
	\norm{v_S} \ge \rho \norm{v}\,.
	\]
\end{definition}

Also recall the definition of Huber loss function.
\begin{definition}[Huber Loss Function]\label{def:huber-loss}
	Let $\bm y=X\betastar+\bm \eta$ for 
	$\betastar\in \R^d, \bm \eta\in \R^n, X\in\R^{n\times d}$. 
	For $h > 0$ and $\beta \in \R^d$, define
	\[
	\bm f_h(\beta) = 
	\sum_{i=1}^n\Phi_h\Paren{\iprod{x_i,\beta} - \bm y_i} =
	\sum_{i=1}^n\Phi_h\Paren{\iprod{x_i,\beta-\betastar} - \bm \eta_i} \,,
	\]
	where 
	\[
	\Phi_h(t) =
	\begin{cases}
	\frac{1}{2h}t^2 & \text{if $\abs{t} \le h$}\\
	\abs{t} - \frac{h}{2} & \text{otherwise.}
	\end{cases}
	\]
\end{definition}


We are now ready to state the main result of the section. We remark that for simplicity we do not optimize constants in the statement below.

\begin{theorem}\label{thm:huber-loss-technical}
	Let $\alpha=\alpha(n,d)\in (0,1)$ and let $\bm y = X\betastar+\bm\eta$ for  $\betastar\in \R^n$,  $X\in \R^{n \times d}$ and $\bm \eta\in \R^n$  satisfying  \cref{assumption:noise-huber}. Let
	\[
	\delta = \frac{10^7\cdot d\ln n}{\alpha^2 \cdot n}\,,
	\]
	and suppose the column span of $X$ is $(\delta\cdot n,1/2)$-spread. Then, for $h=1/n$, the Huber loss estimator  $\bm \betahat := \underset{\beta\in\R^d}{\argmin} \;\bm f_h(\beta)$
	satisfies
	\[
	\frac{1}{n}\snorm{X\Paren{\bm \betahat-\betastar}} \le \delta
	\]
	with probability at least $1-10n^{-d/2}$.
\end{theorem}

\begin{remark}
	If for all $i\in[n]$ the conditional distribution of $\bm \eta_i$ given $\bm \cR$ is symmetric about $0$ (this assumption is satisfied for the noise from considered in \cref{thm:huber-loss-informal}), the theorem is also true for Huber parameter $h = 1$.
\end{remark}

To prove \cref{thm:huber-loss-technical} we need the Lemmata below. We start showing a consequence of the $(\delta\cdot n,1/2)$-spread property of $X$.

\begin{lemma}\label{lem:l1-norm-uncorrupted-rows}
	Let $X\in\R^{n\times d}$, $\alpha$ and $\delta$ be as in \cref{thm:huber-loss-technical}, and let $\bm \cR$ be as in \cref{assumption:noise-huber}.
	With probability $1-2n^{-d/2}$, for any $u\in \R^d$,
	\[
	\sum_{i\in\bm\cR} \abs{\iprod{x_i,u}} \ge 
	 \frac{1}{2}\alpha\cdot\sqrt{\delta n}\cdot\norm{Xu}\,.
	\]
	\begin{proof}
		Let $\bm \zeta_1,\ldots,\bm \zeta_n$ be i.i.d. Bernoulli random variables such that 
		$\Pr\Paren{\bm \zeta_i = 1} = 1 - \Pr\Paren{\bm \zeta_i = 0} = \alpha$. 
		By \cref{lem:random-set-equivalence}, it is enough
		to show that with probability $1-n^{-d}$, for any $u\in \R^d$,
		\[
		\sum_{i=1}^n \bm \zeta_i \abs{\iprod{x_i,u}}
		 \ge \frac{1}{2}\alpha\cdot\sqrt{\delta n}\cdot\norm{Xu}\,.
		\]
		Note that the inequality is scale invariant, 
		hence it is enough to prove it for all $u\in \R^d$ such that $\norm{Xu} = 1$. 
		Consider arbitrary $u\in \R^d$ such that $\norm{Xu} = 1$. 
		Applying \cref{lem:l1vsl2} with $\cA=\emptyset$, 
		$m=\lfloor \delta n\rfloor$, $\gamma_1=0$, $\gamma_2 = 1/2$
		and $v = Xu$, we get 
		\[
		\sum_{i=1}^n  \abs{\iprod{x_i,u}} \ge \frac{3}{4}\sqrt{\delta n}\,.
		\]
		Hence
		\[
		\E_{\bm\zeta} \sum_{i=1}^n \bm \zeta_i \abs{\iprod{x_i,u}} \ge
		 \frac{3}{4}\alpha\cdot\sqrt{\delta n}\,.
		\] 
		Let \(\bracbb{\cdot}\) is the Iverson bracket (0/1 indicator).
		Applying \cref{lem:uniform-bound} with $g(x,y) = \bracbb{y = 1} \cdot \abs{x}$, $v = Xu$ and $\bm w = \bm \zeta = \transpose{(\bm \zeta_1,\ldots, \bm \zeta_n)}$, we get that with probability $1-n^{-d}$ for all $u$ such that $\norm{Xu} = 1$,
		\[
		\Abs{\sum_{i=1}^n \Paren{\bm \zeta_i\abs{\iprod{x_i,u}} - \E_{\bm \zeta} \bm\zeta_i \abs{\iprod{x_i,u}} }}
		\le 20\sqrt{d\ln n}\le 
		\frac{1}{5}\alpha \sqrt{\delta n}\,,
		\]
		which yields the desired bound.
	\end{proof}
\end{lemma}

Next we show that with high probability $\norm{X\Paren{\bm\betahat-\betastar}} < n$. 

\begin{lemma}\label{lem:initial-bound}
	Let $y\in \R^n,X\in \R^{n \times d}$ as in \cref{thm:huber-loss-technical}, and let $h\le 1$.
	With probability $1-4n^{-d/2}$, for any $\beta$ such that $\norm{X(\beta-\betastar)}\ge n$,
	\[
	\bm f_h(\beta) \ge \bm f_h(\beta^*) + 1\,.
	\]
	\begin{proof}
		Note that
		\[
		\bm f_h(\beta^*) = 
		\sum_{i=1}^n \Phi_h\Paren{\bm \eta_i} \le  \sum_{i=1}^n \abs{\bm \eta_i}\,.
		\]
		
		Consider some $\beta$ such that $\norm{X(\beta-\betastar)} = n$. 
		Denote $u = \beta - \betastar$. Since there exists a conditional density $\bm p_i(t)\ge 0.1$ (for  $t\in [-1,1]$), there exist $a, b\in [0,1]$ such that for all $i\in \bm\cR$,
		\[
		\Pr\Paren{  -a \le \bm \eta_i \le 0 \given \bm \cR}
		=
		\Pr\Paren{ 0 \le  \bm \eta_i \le b \given \bm \cR}
		\ge 0.1\,.
		\]
		Let $\bm\cS = \Set{ i\in [n] \suchthat -a \le \bm \eta_i \le b}$.
		We get
		\begin{align*}
		\bm f_h(\beta) = \sum_{i=1}^n \Phi_h\Paren{\iprod{x_i, u} - \bm \eta_i} 
		\ge&
		\sum_{i=1}^n \Abs{\iprod{x_i, u} - \bm \eta_i}  - hn
		\\
		\ge&
		\sum_{i\in \bm\cS\cap \bm \cR} \Abs{\iprod{x_i, u}} +
		\sum_{i\in [n]\setminus\bm\cS} \Abs{\iprod{x_i, u} - \bm \eta_i} - 2n.
		\end{align*}
		
		Denote $\bm \zeta_i = \bracbb{-a \le \bm \eta_i\le b }$.
		By \cref{lem:l1-norm-uncorrupted-rows},
		\[
		\E\Brac{\sum_{i\in\bm\cR}\bm \zeta_i\cdot \abs{\iprod{x_i,u}}\given \bm \cR} \ge 
	 \frac{1}{10}\alpha\cdot\sqrt{\delta n}\cdot\norm{Xu}\,.
		\]
		with probability $1-2n^{-d/2}$.

		By \cref{lem:uniform-bound} with $g(x,y) = \bracbb{y = 1} \cdot \abs{x}$, $v = X_{\bm\cR}u$, $R=n$ and $\bm w = \bm \zeta_{\bm\cR}$, we get that with probability $1-n^{-d}$ for all $u$ such that $\norm{X_{\bm\cR}u} \le n$,
		\[
		\sum_{i\in\bm\cS\cap \bm \cR} \abs{\iprod{x_i,u}} \ge 
		\frac{1}{10}\alpha\cdot\sqrt{\delta n}\cdot\norm{Xu} - 
		20\cdot \norm{X_{\bm\cR}u}\cdot\sqrt{d\ln n} - 1
		\ge  \frac{1}{20}\cdot \alpha\cdot\sqrt{\delta n}\cdot\norm{Xu}-1\,.
		\]
		Note that
		\begin{align*}
		\sum_{i\in [n]\setminus\bm\cS}\Abs{\iprod{x_i, u} - \bm \eta_i} 
		=& 
		\sum_{i\in [n]\setminus\bm\cS}
		\Abs{\abs{\bm\eta_i} - \sign\paren{\bm \eta_i}\iprod{x_i, u}}
		\ge
		\sum_{i\in [n]\setminus\bm\cS} \abs{\bm \eta_i} - 
		\sum_{i\in [n]\setminus\bm\cS} 
		\sign\paren{\bm \eta_i}\iprod{x_i, u}\,.
		\end{align*}
		
		Applying \cref{lem:uniform-bound} with $g(x,y) = \bracbb{-a \le y\le b }\sign\paren{y}\cdot  \abs{x}$, 
		$v = Xu$, $\bm w = \bm \eta$, 
		$R = n$, we get that with probability $1-n^{-d}$, 
		for all $u\in \R^d$ such that $\norm{Xu} = n$,
		\begin{align*}
		\Abs{\sum_{i\in [n]\setminus \bm\cS}
		\sign\paren{\bm \eta_i}\iprod{x_i, u}} 
	=\,&
		\Abs{\sum_{i=1}^n 
			\Paren{\bracbb{-a \le \bm \eta_i \le b} \sign\paren{\bm \eta_i}\iprod{x_i, u} - 
				\E\Brac{\bracbb{-a \le \bm \eta_i \le b}
					\sign\paren{\bm \eta_i}\iprod{x_i, u} 
					\given \bm \cR}}}
		\\
		\le\,& 
		20n\sqrt{d\ln n}\norm{Xu}  + 1\,.
		\end{align*}
		Therefore, with probability $1-4n^{-d/2}$, for any $\beta\in \R^d$ such that $\norm{X(\beta - \betastar)} = n$,
		\[
		\bm f_h(\beta) \ge 
	    \frac{1}{20}\alpha\cdot\sqrt{\delta n}\cdot\norm{Xu}
		+ \sum_{i=1}^n \abs{\bm \eta_i} - \sum_{i\in\bm\cS}\abs{\bm\eta_i} -
		 2n - 20n\sqrt{d\ln n}\norm{Xu}  -2
		\ge  \sum_{i=1}^n \abs{\bm \eta_i}  + 1 \ge \bm f_h(\betastar) + 1\,.
		\]
		Note that since $\bm f$ is convex,  with probability $1-4n^{-d/2}$, for any $\beta$ such that 
		$\norm{X(\beta - \betastar)} > n$, $\bm f_h(\beta) \ge \bm f_h(\betastar) + 1$.
	\end{proof}
\end{lemma}


Now  observe that $\bm  f_h(\beta)$ is differentiable and
\[
\nabla \bm f_h(\beta) = \sum_{i=1}^n \phi_h\Paren{\iprod{x_i,\beta} - \bm y_i}\cdot x_i = 
\sum_{i=1}^n \phi_h\Paren{\iprod{x_i,\beta-\betastar} - \bm \eta_i}\cdot x_i \,,
\]
where $\phi_h(t) = \Phi'_h(t) = \sign(t) \cdot \min\set{\abs{t}/h, 1}$, $t\in \R$.
We will need the following lemma. 

\begin{lemma}\label{lem:bound-expectation}
Let $\bm z$ be a random variable such that $\Pr\Paren{\bm z\le 0} = \Pr\Paren{\bm z \ge 0}$. 
Then for any $\tau$ such that $\abs{\tau} \ge 2h$,
\[
\tau \cdot \E_{\bm z} \phi_h\Paren{\tau- \bm z} \ge 
\abs{\tau}\cdot\Pr\Paren{0\le \sign\paren{\tau} \cdot \bm z\le \abs{\tau}/2}\,.
\]
\begin{proof}
	Note that
	\[
	\tau \cdot \E_{\bm z} \phi_h\Paren{\tau- \bm z} = 
	\abs{\tau}\cdot \E_{\bm z} \phi_h\Paren{\abs{\tau}-\sign\paren{\tau}\cdot \bm z}\,.
	\]
	 We get
	\begin{align*}
	\E_{\bm z} \phi_h\Paren{\abs{\tau}-\sign\paren{\tau} \bm z} 
	&=
	\Pr\Paren{\sign\paren{\tau}  \bm  z\le \abs{\tau} - h}
	+\E_{\bm z} \bracbb{\sign\paren{\tau}  \bm z > \abs{\tau}-h}\cdot \phi_h\Paren{\abs{\tau}-\sign\paren{\tau} \bm z}
	\\&\ge
	\Pr\Paren{0\le \sign\paren{\tau}  \bm z\le \abs{\tau} - h} + 
	\Pr\Paren{\sign\paren{\tau}   \bm z < 0} - \Pr\Paren{\sign\paren{\tau}   \bm z > 0}
	\\&\ge
	\Pr\Paren{0\le \sign\paren{\tau}   \bm z\le \abs{\tau}/2}\,.
	\end{align*}
\end{proof}
\end{lemma}

Using  point 3 of  \cref{assumption:noise-huber}, we get for all $i\in [n]$ and for all $\tau \ge 2h$,
\[
\tau \cdot \E\Brac{ \phi_h\Paren{\tau- \bm \eta_i} \given \bm \cR}\ge 
\frac{1}{20}\abs{\tau}\cdot\min\set{\abs{\tau},1}\,.
\]

Note that for $h=1$, if $\bm z$ is symmetric, we can also show it for $\tau \le 2h \le 2$. Indeed,
\begin{align*}
\abs{\tau}  \E_{\bm z} \phi_h\Paren{\abs{\tau}-\sign(\tau) \bm z} 
&= 
\abs{\tau}  \E_{\bm z} \bracbb{\abs{\bm z} \le h} \phi_h\Paren{\abs{\tau}- \sign(\tau)\bm z}
+
\abs{\tau}  \E_{\bm z} \bracbb{\abs{\bm z} > h} \phi_h\Paren{\abs{\tau}- \sign(\tau)\bm z}
\\&\ge
\abs{\tau}  \E_{\bm z} \bracbb{\abs{\bm z} \le h} \phi_h\Paren{\abs{\tau}-\sign(\tau) \bm z}
\\&=
\abs{\tau}  \E_{\bm z} \bracbb{\abs{\bm z} \le h} \Paren{\phi_h\Paren{\abs{\tau}- \bm z}-\phi_h\Paren{-\bm z}}\,,
\end{align*}
since for symmetric $\bm z$, $\E_{\bm z} \bracbb{\abs{\bm z} > h} \phi_h\Paren{\abs{\tau}-\sign(\tau) \bm z} \ge 0$.
Assuming the existence of density $p$ of $\bm z$ such that $p(t) \ge 0.1$ for all $t\in[-1,1]$, 
we get
\begin{align*}
\E_{\bm z} \bracbb{\abs{\bm z} \le h} \Paren{\phi_h\Paren{\abs{\tau}-\bm z}-\phi_h\Paren{-\bm z}}
\ge&
0.1\int_{-1}^1  \Paren{\phi_h\Paren{\abs{\tau}- z}-\phi_h\Paren{z}} dz
\\
\ge& \frac{1}{20}\min\set{\abs{\tau},1}\,.
\end{align*}

We are now ready to prove \cref{thm:huber-loss-technical}.

\begin{proof}[Proof of \cref{thm:huber-loss-technical}]
	Consider $\bm u = \bm \betahat - \betastar$. By \cref{lem:initial-bound}, with probability $1-3n^{-d/2}$, $\norm{X \bm u} \le n$. If $\norm{X \bm u} < 100$, we get the desired bound. 
	So further we assume that $100 \le \norm{X  \bm u} \le n$.
	
	Since $\nabla \bm f_h(\betahat) = 0$, 
	\[
	\sum_{i=1}^n \phi_h\Paren{\iprod{x_i, \bm u} - \bm \eta_i}\cdot \iprod{x_i, \bm u} = 0\,.
	\]
	
	For each $i\in [n]$, consider the function $F_i$ defined as follows:
	\[
	F_i(a) = 
	 \iprod{x_i, a}\E\Brac{\phi_h\Paren{\iprod{x_i,  a} - \bm \eta_i}\given \bm \cR}
	\]
	for any $a\in \R^d$. Applying \cref{lem:bound-expectation} with $\bm z = \bm \eta_i$, $\tau = \iprod{x_i, a}$ and $h = 1/n$, and using  point 3 of 
	\cref{assumption:noise-huber},
	we get for any $a\in \R^d$,
	\begin{align}
	\sum_{i=1}^n 
	F_i( a)
	&\ge 
	-\sum_{i=1}^n \bracbb{\abs{\iprod{x_i, a}} < 2h} \abs{\iprod{x_i, a}} +
	\sum_{i=1}^n \bracbb{\abs{\iprod{x_i, a}} \ge 2h} F_i(a)
	\\&\ge  
	-2 + 
	\frac{1}{20}\sum_{i\in \bm\cR} 
	\abs{\iprod{x_i,a}} \cdot\min\Set{\abs{\iprod{x_i,a}}, 1}\,.\label{eq:huber-expectation-bound}
	\end{align}
	
	Note that this is the only place in the proof where we use $h\le1/n$. By the observation described after \cref{lem:bound-expectation}, if for all $i\in[n]$, conditional distribution of $\bm \eta_i$ given $\bm \cR$ is symmetric about $0$, the proof also works for $h = 1$.
	
	For $x, y\in \R$, consider $g(x,y) = x\cdot \phi_h\Paren{x - y}$. For any $\Delta x\in \R$, 
	\[
	\abs{g(x+\Delta x,y) - g(x,y)} =  
	\Abs{\Paren{x+\Delta x} \cdot \phi_h\Paren{x + \Delta x - y} - x\cdot \phi_h\Paren{x - y}}
	\le \Delta x + \frac{x}{h}\Delta x\,.
	\]
	By \cref{lem:uniform-bound} with $v = X a$, $\bm w=\bm \eta$, $R = n$, 
	and $K = \Paren{1 + n^2}$, 	
	with probability $1-4n^{-d/2}$, for all $a\in \R^d$ such that $\norm{Xa}\le n$,
	\begin{equation}\label{eq:huber-gradient-upper-bound}
	\Abs{\sum_{i=1}^n \Paren{\iprod{x_i, a} \cdot \phi_h\Paren{\iprod{x_i,a} - \bm \eta_i}
			 - 
			F_i(a)}}
		\le
	25\sqrt{d\ln n} \cdot \norm{Xa} + 1/n\,.
	\end{equation}

	Let $\bm \zeta_1,\ldots,\bm \zeta_n$ be i.i.d. Bernoulli random variables such that 
	$\Pr\Paren{\bm \zeta_i = 1} = 1 - \Pr\Paren{\bm \zeta_i = 0} = \alpha$. 
	By \cref{lem:random-set-equivalence} and \cref{lem:uniform-bound} with $g(x,y) = \bracbb{y=1} |x| \cdot \min\set{|x|, 1}$, $v = Xa$, $\bm w_i = \bm\zeta_i$, $R=n$ and $K = 2$, 
	with probability $1-3n^{-d/2}$, for all $a\in \R^d$ such that $\norm{Xa}\le n$,
	\begin{equation}\label{eq:huber-gradient-lower-bound}
	\sum_{i\in \bm\cR} \abs{\iprod{x_i, a}} \cdot\min\Set{\abs{\iprod{x_i,a}}, 1}
	\ge
	\alpha \sum_{i=1}^{n} \abs{\iprod{x_i,a}} \cdot\min\Set{\abs{\iprod{x_i,a}}, 1}
	- 20\sqrt{d\ln n}\, \norm{Xa} - 1/n\,.
	\end{equation}
	
	Plugging $a=\bm u$ into inequalities \ref{eq:huber-expectation-bound},
	\ref{eq:huber-gradient-upper-bound} and \ref{eq:huber-gradient-lower-bound}, we get
	\[
	\sum_{\abs{\iprod{x_i,\bm u}} \le 1} \iprod{x_i,\bm u}^2 +
	\sum_{\abs{\iprod{x_i,\bm u}} > 1} \abs{\iprod{x_i,\bm u}} \le 
	\frac{1000}{\alpha}\sqrt{d\ln n} \cdot \norm{X\bm u}
	\]
	with probability $1-7n^{-d/2}$.
	
	If 
	\[
	\sum_{\abs{\iprod{x_i,\bm u}} \le 1} \iprod{x_i,\bm u}^2 < \frac{1}{3}\snorm{X\bm u}\,,
	\]
	we get
	\[
	\sum_{\abs{\iprod{x_i,\bm u}} > 1} \abs{\iprod{x_i,\bm u}} \le 
	\frac{1000}{\alpha}\sqrt{d\ln n} \cdot \norm{X\bm u}\,.
	\]
	Applying \cref{lem:l1vsl2} with $m =\lfloor \frac{10^7 d\ln n}{\alpha^2} \rfloor$, $\gamma_1=1/\sqrt{3}$, $\cA= \Set{i\in[n]: \abs{\iprod{x_i, \bm u}} \le 1}$, $\gamma_2 = 1/2$ and $v = X\bm u$, we get a contradiction.
	
	Hence
	\[
	\sum_{\abs{\iprod{x_i,\bm u}} \le 1} \iprod{x_i,\bm u}^2 \ge \frac{1}{3} \snorm{X\bm  u}
	\]
	and we get
	\[
	\norm{X\Paren{\bm\betahat-\betastar}} \le \frac{3000}{\alpha}\sqrt{d\ln n}\,,
	\]
	 with probability at least $1-10n^{-d/2}$, which yields the desired bound.
\end{proof}




\phantomsection
\addcontentsline{toc}{section}{Bibliography}
\bibliographystyle{amsalpha}
\bibliography{bib/mathreview,bib/dblp,bib/scholar,bib/custom}

\newcommand{\etalchar}[1]{$^{#1}$}
\providecommand{\bysame}{\leavevmode\hbox to3em{\hrulefill}\thinspace}
\providecommand{\MR}{\relax\ifhmode\unskip\space\fi MR }
\providecommand{\MRhref}[2]{%
  \href{http://www.ams.org/mathscinet-getitem?mr=#1}{#2}
}
\providecommand{\href}[2]{#2}
\begin{thebibliography}{WYG{\etalchar{+}}08}

\bibitem[BFP{\etalchar{+}}73]{blum1973time}
Manuel Blum, Robert~W. Floyd, Vaughan~R. Pratt, Ronald~L. Rivest, and
  Robert~Endre Tarjan, \emph{Time bounds for selection}, J. Comput. Syst. Sci.
  \textbf{7} (1973), no.~4, 448--461.

\bibitem[BJKK17a]{bathia_crr}
Kush Bhatia, Prateek Jain, Parameswaran Kamalaruban, and Purushottam Kar,
  \emph{Consistent robust regression}, Advances in Neural Information
  Processing Systems 30 (I.~Guyon, U.~V. Luxburg, S.~Bengio, H.~Wallach,
  R.~Fergus, S.~Vishwanathan, and R.~Garnett, eds.), Curran Associates, Inc.,
  2017, pp.~2110--2119.

\bibitem[BJKK17b]{DBLP:conf/nips/Bhatia0KK17}
Kush Bhatia, Prateek Jain, Parameswaran Kamalaruban, and Purushottam Kar,
  \emph{Consistent robust regression}, {NIPS}, 2017, pp.~2107--2116.

\bibitem[CRT05]{taoRIP}
Emmanuel Candes, Justin Romberg, and Terence Tao, \emph{Stable signal recovery
  from incomplete and inaccurate measurements}, 2005.

\bibitem[CSV17]{charikar2017learning}
Moses Charikar, Jacob Steinhardt, and Gregory Valiant, \emph{Learning from
  untrusted data}, Proceedings of the 49th Annual ACM SIGACT Symposium on
  Theory of Computing, 2017, pp.~47--60.

\bibitem[CT05]{taoRIP2005}
Emmanuel Candes and Terence Tao, \emph{Decoding by linear programming}, 2005.

\bibitem[DKK{\etalchar{+}}19]{diakonikolas2019robust}
Ilias Diakonikolas, Gautam Kamath, Daniel Kane, Jerry Li, Ankur Moitra, and
  Alistair Stewart, \emph{Robust estimators in high-dimensions without the
  computational intractability}, SIAM Journal on Computing \textbf{48} (2019),
  no.~2, 742--864.

\bibitem[DKS19]{DBLP:conf/soda/DiakonikolasKS19}
Ilias Diakonikolas, Weihao Kong, and Alistair Stewart, \emph{Efficient
  algorithms and lower bounds for robust linear regression}, Proceedings of the
  Thirtieth Annual {ACM-SIAM} Symposium on Discrete Algorithms, {SODA} 2019,
  San Diego, California, USA, January 6-9, 2019 (Timothy~M. Chan, ed.), {SIAM},
  2019, pp.~2745--2754.

\bibitem[Don06]{donoho2006compressed}
David~L Donoho, \emph{Compressed sensing}, IEEE Transactions on information
  theory \textbf{52} (2006), no.~4, 1289--1306.

\bibitem[DT19]{paristech}
Arnak Dalalyan and Philip Thompson, \emph{Outlier-robust estimation of a sparse
  linear model using $l_1 $-penalized huber's $ m $-estimator}, Advances in
  Neural Information Processing Systems, 2019, pp.~13188--13198.

\bibitem[EvdG{\etalchar{+}}18]{vandeGeer}
Andreas Elsener, Sara van~de Geer, et~al., \emph{Robust low-rank matrix
  estimation}, The Annals of Statistics \textbf{46} (2018), no.~6B, 3481--3509.

\bibitem[GLR10]{DBLP:journals/combinatorica/GuruswamiLR10}
Venkatesan Guruswami, James~R. Lee, and Alexander~A. Razborov, \emph{Almost
  euclidean subspaces of \emph{l} \({}_{\mbox{1}}\)\({}^{\mbox{\emph{n}}}\)
  {VIA} expander codes}, Combinatorica \textbf{30} (2010), no.~1, 47--68.

\bibitem[GLW08]{DBLP:conf/approx/GuruswamiLW08}
Venkatesan Guruswami, James~R. Lee, and Avi Wigderson, \emph{Euclidean sections
  of with sublinear randomness and error-correction over the reals},
  {APPROX-RANDOM}, Lecture Notes in Computer Science, vol. 5171, Springer,
  2008, pp.~444--454.

\bibitem[Gro11]{DBLP:journals/tit/Gross11}
David Gross, \emph{Recovering low-rank matrices from few coefficients in any
  basis}, {IEEE} Trans. Information Theory \textbf{57} (2011), no.~3,
  1548--1566.

\bibitem[HBRN08]{haupt2008compressed}
Jarvis Haupt, Waheed~U Bajwa, Michael Rabbat, and Robert Nowak,
  \emph{Compressed sensing for networked data}, IEEE Signal Processing Magazine
  \textbf{25} (2008), no.~2, 92--101.

\bibitem[Hoa61]{Hoare}
C.~A.~R. Hoare, \emph{Algorithm 65: Find}, Commun. ACM \textbf{4} (1961),
  no.~7, 321–322.

\bibitem[Hub64]{huber1964}
Peter~J. Huber, \emph{Robust estimation of a location parameter}, Ann. Math.
  Statist. \textbf{35} (1964), no.~1, 73--101.

\bibitem[Jen69]{jennrich1969}
Robert~I. Jennrich, \emph{Asymptotic properties of non-linear least squares
  estimators}, Ann. Math. Statist. \textbf{40} (1969), no.~2, 633--643.

\bibitem[KKK19]{conf/nips/KarmalkarKK19}
Sushrut Karmalkar, Adam~R. Klivans, and Pravesh Kothari, \emph{List-decodable
  linear regression}, Advances in Neural Information Processing Systems 32:
  Annual Conference on Neural Information Processing Systems 2019, NeurIPS
  2019, 8-14 December 2019, Vancouver, BC, Canada, 2019, pp.~7423--7432.

\bibitem[KKM18]{colt/KlivansKM18}
Adam~R. Klivans, Pravesh~K. Kothari, and Raghu Meka, \emph{Efficient algorithms
  for outlier-robust regression}, Conference On Learning Theory, {COLT} 2018,
  Stockholm, Sweden, 6-9 July 2018, 2018, pp.~1420--1430.

\bibitem[KP18]{karmalkar2018compressed}
Sushrut Karmalkar and Eric Price, \emph{Compressed sensing with adversarial
  sparse noise via l1 regression}, arXiv preprint arXiv:1809.08055 (2018).

\bibitem[KT07]{kashin2007remark}
Boris~S Kashin and Vladimir~N Temlyakov, \emph{A remark on compressed sensing},
  Mathematical notes \textbf{82} (2007), no.~5, 748--755.

\bibitem[LLC19]{caramanis1}
Liu Liu, Tianyang Li, and Constantine Caramanis, \emph{High dimensional robust
  $m$-estimation: Arbitrary corruption and heavy tails}, 2019.

\bibitem[LM00]{laurent2000}
B.~Laurent and P.~Massart, \emph{Adaptive estimation of a quadratic functional
  by model selection}, Ann. Statist. \textbf{28} (2000), no.~5, 1302--1338.

\bibitem[LSLC18]{caramanis2}
Liu Liu, Yanyao Shen, Tianyang Li, and Constantine Caramanis, \emph{High
  dimensional robust sparse regression}, arXiv preprint arXiv:1805.11643
  (2018).

\bibitem[NRWY09]{DBLP:conf/nips/NegahbanRWY09}
Sahand Negahban, Pradeep Ravikumar, Martin~J. Wainwright, and Bin Yu, \emph{A
  unified framework for high-dimensional analysis of
  {\textdollar}m{\textdollar}-estimators with decomposable regularizers},
  {NIPS}, Curran Associates, Inc., 2009, pp.~1348--1356.

\bibitem[NT13]{tractran}
Nam~H Nguyen and Trac~D Tran, \emph{Exact recoverability from dense corrupted
  observations via $l_1$-minimization}, IEEE transactions on information theory
  \textbf{59} (2013), no.~4.

\bibitem[Pol91]{pollard}
David Pollard, \emph{Asymptotics for least absolute deviation regression
  estimators}, Econometric Theory \textbf{7} (1991), no.~2, 186--199.

\bibitem[RH15]{rigollet2015high}
Phillippe Rigollet and Jan-Christian H{\"u}tter, \emph{High dimensional
  statistics}, Lecture notes for course 18S997 \textbf{813} (2015), 814.

\bibitem[RL05]{rousseeuw2005robust}
Peter~J Rousseeuw and Annick~M Leroy, \emph{Robust regression and outlier
  detection}, vol. 589, John wiley \& sons, 2005.

\bibitem[RY20]{yau20}
Prasad Raghavendra and Morris Yau, \emph{List decodable learning via sum of
  squares}, Proceedings of the 2020 {ACM-SIAM} Symposium on Discrete
  Algorithms, {SODA} 2020, Salt Lake City, UT, USA, January 5-8, 2020, 2020,
  pp.~161--180.

\bibitem[SBRJ19]{SuggalaBR019}
Arun~Sai Suggala, Kush Bhatia, Pradeep Ravikumar, and Prateek Jain,
  \emph{Adaptive hard thresholding for near-optimal consistent robust
  regression}, Conference on Learning Theory, {COLT} 2019, 25-28 June 2019,
  Phoenix, AZ, {USA}, 2019, pp.~2892--2897.

\bibitem[SZF19]{sun2019adaptive}
Qiang Sun, Wen-Xin Zhou, and Jianqing Fan, \emph{Adaptive huber regression},
  Journal of the American Statistical Association (2019), 1--24.

\bibitem[TJSO14]{tsakonas2014convergence}
Efthymios Tsakonas, Joakim Jald{\'e}n, Nicholas~D Sidiropoulos, and Bj{\"o}rn
  Ottersten, \emph{Convergence of the huber regression m-estimate in the
  presence of dense outliers}, IEEE Signal Processing Letters \textbf{21}
  (2014), no.~10, 1211--1214.

\bibitem[TSW18]{tanSuWitten}
Kean~Ming Tan, Qiang Sun, and Daniela Witten, \emph{Robust sparse reduced rank
  regression in high dimensions}, arXiv preprint arXiv:1810.07913 (2018).

\bibitem[Tuk75]{tukey}
John~W Tukey, \emph{Mathematics and the picturing of data}, Proceedings of the
  International Congress of Mathematicians, Vancouver, 1975, vol.~2, 1975,
  pp.~523--531.

\bibitem[Ver18]{vershynin_2018}
Roman Vershynin, \emph{High-dimensional probability: An introduction with
  applications in data science}, Cambridge Series in Statistical and
  Probabilistic Mathematics, Cambridge University Press, 2018.

\bibitem[Vis18]{vishnoi2018algorithms}
Nisheeth~K Vishnoi, \emph{Algorithms for convex optimization}, Cambridge
  University Press, 2018.

\bibitem[Wai19]{wainwright_2019}
Martin~J. Wainwright, \emph{High-dimensional statistics: A non-asymptotic
  viewpoint}, Cambridge Series in Statistical and Probabilistic Mathematics,
  Cambridge University Press, 2019.

\bibitem[WYG{\etalchar{+}}08]{wright2008robust}
John Wright, Allen~Y Yang, Arvind Ganesh, S~Shankar Sastry, and Yi~Ma,
  \emph{Robust face recognition via sparse representation}, IEEE transactions
  on pattern analysis and machine intelligence \textbf{31} (2008), no.~2,
  210--227.

\end{thebibliography}

\clearpage
\appendix


\section{Error convergence and model assumptions}\label{sec:error-convergence-model-assumptions}

In this section we discuss  the error convergence of our main theorems \cref{thm:huber-loss-gaussian-results}, \cref{thm:huber-loss-informal} as well as motivate our model assumptions.

\subsection{Lower bounds for consistent oblivious linear regression}\label{sec:bound-fraction-corruption}

We show here that no estimator can obtain expected squared error $o\Paren{d/(\alpha^2\cdot n)}$ for any $\alpha\in (0,1)$ and that no estimator can have expected error converging to zero for $\alpha \lesssim \sqrt{d/n}$.
The first claim is captured by  the following  statement.
\begin{fact}\label{fact:linear-regression-lower-bound}
  Let $X\in \R^{n\times d}$ be a matrix with linearly independent columns. Let $\bm \eta \sim N(0,\sigma^2\cdot \Id_n)$  with \(\sigma>0\) so that \(\alpha=\min_{i}\Pr\set{\abs{\bm\eta_i}\le 1}=\Theta(1/\sigma)\).
  
  Then there exists a distribution over $\bm \beta^*$ independent of $\bm \eta$ such that for every estimator $\hat{\beta}:\R^n  \rightarrow \R^d$, with probability at least $\Omega(1)$,
  \[
  	\frac{1}{n}\Snorm{X\hat{\beta}(X \bm\beta^*+ \bm \eta)-X\beta^*}\geq \Omega \Paren{ \frac{d}{\alpha^2\cdot n}}\,.
  \]
  
  In particular, for every estimator $\hat{\beta}:\R^n  \rightarrow \R^d$  there exists a vector $\beta^*$ such that for $\bm y = X \beta^*+ \bm \eta$, with probability at least $\Omega(1)$,
	\begin{align*}
		\frac{1}{n}\Snorm{X\hat{\beta}(\bm{y})-X\beta^*}\geq \Omega \Paren{ \frac{d}{\alpha^2\cdot n}}\,.
	\end{align*}

\end{fact}

\cref{fact:linear-regression-lower-bound} is well-known and  we omit the proof here (see for example \cite{rigollet2015high}).
The catch is that the vector $\bm \eta\sim N(0,\sigma^2\cdot \Id)$ satisfies the noise constraints of \cref{thm:huber-loss-informal} for $\alpha= \Theta(1/\sigma)$.
Hence, for $\alpha\lesssim \sqrt{d/n}$ we  obtain the second claim as an immediate  corollary.

\begin{corollary}
	Let $n,d \in \R$ and $\alpha\lesssim\sqrt{\frac{d}{n}}$. 
	Let $X\in \R^{n\times d}$ be a matrix with linearly independent columns and let $\bm \eta\sim N(0, 1/\alpha^2\cdot \Id_n)$. Then, for 
	every estimator $\hat{\beta}:\R^n  \rightarrow \R^d$ there exists a vector $\beta^*$ such that
	for $\bm y = X\betastar+\bm \eta$, with probability at least $\Omega(1)$,
	%
	\begin{align*}
		\frac{1}{n}\Snorm{X\hat{\beta}(\bm{y})-X\beta^*}\ge \Omega(1)\,.
	\end{align*}
\end{corollary}

In other words, in this regime no estimator obtains error converging to zero.

\subsection{On the design assumptions}

Recall our linear model with Gaussian design:  
\begin{align}\label{eq:gaussian_model}
	\bm y=\bm X\betastar+\eta
\end{align} for $\betastar\in \R^d$, $\bm X\in \R^{n\times d}$ with i.i.d entries $\bm X_\ij \sim N(0,1)$ and $\eta\in \R^n$ a deterministic vector with $\alpha\cdot n$ entries bounded by $1$ in absolute value.
We have mentioned in \cref{sec:results} how to extend these Gaussian design settings to deterministic design settings. 
We formally show here how to apply \cref{thm:huber-loss-informal} to reason about the Gaussian design model \cref{eq:gaussian_model}.
For this it suffices to to turn an instance of model \cref{eq:gaussian_model} into an instance of the model:
\begin{align}\label{eq:deterministic_model}
	\bm y'=  X'\betastar+\bm \eta'
\end{align}
where $\betastar\in \R^d$, $ X'$ is a $n$-by-$d$ matrix $(\Omega(n),\Omega(1))$-spread   and the noise vector  $\bm\eta$ has  independent, symmetrically distributed entries with $\alpha =\min_{i \in [n]}\bbP \Set{\Abs{\bm \eta_i}\leq 1}$.
 This  can be done resampling through the following procedure.
 \begin{algorithm}[H]
 	\caption{Resampling}
 	\label{alg:resampling}
 	\begin{algorithmic}
 		\STATE \textbf{Input: }$(y,X)$ where $y\in \R^n$, $X\in \R^{n\times d}$. 
 		\STATE Sample $n$ indices $\bm \gamma_1,\ldots,\bm \gamma_n$ independently from the uniform distribution over $[n]$.
 		\STATE Sample $n$ i.i.d Rademacher random variables $\bm \sigma_1,\ldots, \bm \sigma_n$.
 		\STATE Return pair $(\bm y', \bm X')$ where $\bm y'$ is an $n$-dimensional random vector and $\bm X'$ is an $n$-by-$d$ random matrix:
 		\begin{align*}
 			\bm y'_i &= \bm \sigma_i \cdot y_{\bm \gamma_i}\\
 			\bm X'_{i,-} &= \bm \sigma_i\cdot X_{\bm \gamma_i,-}\,.
 		\end{align*}
 	\end{algorithmic}
 \end{algorithm}

Note that $\bm X'$ and $\bm \eta$ are independent.
The next result shows that for $n\gtrsim d\log d$ with high probability over $\bm X'$ \cref{alg:resampling} outputs an instance of \cref{eq:deterministic_model} as desired.

\begin{theorem}\label{thm:gaussian-with-repetitions}
	Let $n\gtrsim d\ln d$.
	Let $\bm X'\in \R^{n\times d}$ 
	be a matrix obtained from a Gaussian matrix 
	$\bm X\in \R^{n\times d}$ by choosing (independently of $\bm X$) $n$ rows of  $\bm X$ with replacement.
	Then column span of $\bm X'$ is 
	$(\Omega\Paren{n}, \Omega(1))$-spread with probability $1-o(1)$ as $n\to \infty$.
	\begin{proof}
		Let $\bm c(i)$ be the row chosen at $i$-th step.
		Let's show that with high probability for all subsets $\cM\subseteq[n]$ of size $m$, 
		$\Card{\bm c^{-1}(\cM)} \le O\Paren{m\ln\Paren{n/m}}$. 
		By Chernoff bound (\cref{fact:chernoff}), for any $j\in [n]$ and any $\Delta \ge 1$,
		\begin{equation}\label{eq:chernoff-large-deviation}
		\Pr\Paren{\sum_{i=1}^n \bracbb{\bm c(i)=j} \ge 2\Delta} \le \Delta^{-\Delta}\,.
		\end{equation}
		Denote $a_m = en/m$. Let $A_j\Paren{2\Delta_j}$ be the event $\Set{\sum_{i=1}^n \ind{\bm c(i)=j} \ge 2\Delta_j}$.
		If $\sum_{j=1}^m \Delta_j = 2m\ln a_m$, then
		\begin{align*}
		\Pr\Brac{\bigcap_{j=1}^m A_j\Paren{2\Delta_j}}
		&= \Pr\Brac{A_m\Paren{2\Delta_m}} \cdot
		\Pr\Brac{\bigcap_{j=1}^{m-1} A_j\Paren{2\Delta_j} \given A_m\Paren{2\Delta_m}}
		\\& \le \Pr\Brac{A_m\Paren{2\Delta_m}} \cdot
		\Pr\Brac{\bigcap_{j=1}^{m-1} A_j\Paren{2\Delta_j}}
		\\& \le \exp\Paren{-\sum_{j=1}^m \Delta_j}
		\\& = a_m^{-2m}\,.
		\end{align*}
		By union bound,
		with probability at least $1-a_m^{-m}$, for all subsets $\cM\subseteq[n]$ of size $m$, 
		$\Card{\bm c^{-1}(\cM)} \le 4{m\ln\Paren{en/m}}$. 
		Let  $\bm z$ be the vector with entries  $\bm z_j = \sum_{i=1}^n \bracbb{\bm c(i)=j}$.
		With probability at least $1-1/n \le 1-\sum_{m=1}^{n} a_m^{-m}$, for any $\cM\subseteq [n]$,
		\[
		\sum_{j\in \cM} \abs{\bm z_j} \le 4\Card{\cM} \ln\Paren{\frac{en}{\Card{\cM}}}\,.
		\]
		
		Note that by  \cref{fact:k-sparse-norm-gaussian}, with probability at least $1 - 1/n$, for any vector unit $v$ from column span of $\bm X$ and any set $\cM\subseteq [n]$,
		\[
		\norm{v_{\cM}}^2 \le 2d + 20\Card{\cM}\ln\Paren{\frac{en}{\Card{\cM}}}\,.
		\]
		
		Now let $v'$  be arbitrary unit vector from column span of $\bm X'$. Let $\cS$ be the set of its top $m=c^{6}\cdot n$ entries for some small enough constant $c$. Then for some vector $v$ from column span of $\bm X$, with probability $1-o(1)$,
		\[
		\sum_{i\in\cS} (v'_i)^2 \le \sum_{j\in \bm c\Paren{\cS}} \abs{z_j} \cdot  v_j^2 = \iprod{z_{c\Paren{\cS}}, v^2_{c\Paren{\cS}}} 
		\le c^4\cdot \Paren{d\ln(en) +  6n} \cdot \norm{v}^2\,,
		\]
		where the last inequality follows from \cref{lem:iprod-of-log-spread-vectors}.
		
		Now let's bound $\norm{v}^2$.
		Let $u$ be a fixed unit vector, then $\bm Xu \sim N(0, \Id_n)$. By Chernoff bound, with probability at least $1-\exp\Paren{-cn}$, 
		number of entries of $\bm Xu$ bounded by $c$ is at most $2cn$. Note that with high probability $0.9\sqrt{n}\norm{u}\le \norm{\bm X u} \le 1.1\sqrt{n}\norm{u}$
		
		Let  $u'\in \R^d$ be a unit vector such that number of entries of  $\bm Xu'$  bounded by $c$ is at most $2cn$ and 
		let $u\in \R^d$ be a unit vector such that $\norm{\bm Xu' - \bm X u}^2 \le 1.1^2\cdot n\norm{u'-u}^2 \le c^2n/5$. 
		Then $Xu$ cannot have more than $3cn$ entries bounded by $c/2$.
		Hence by union bound over $c/3$-net 
		in the unit ball in $\R^d$,  
		with probability at least $1-\exp\Paren{-cn/2}$, 
		$v$ has at most $3cn$ entries of magnitude 
		smaller than $c/2$. 
		Hence $v'$ has at most $12cn\ln\Paren{\frac{e}{3c}}\le 0.9n$ 
		entries of magnitude 
		smaller than $\frac{c}{3\sqrt{n}}\norm{v}$, and 
		$\norm{v'}^2 \ge \frac{c^2}{100}\norm{v}^2$. Choosing $c = 0.01$,
		we get that column span of $\bm X'$ is $(10^{-12}\cdot n, 1/2)$-spread 
		with high probability.
	\end{proof}
\end{theorem}

\begin{remark} It is perhaps more evident how the model considered in  \cref{thm:huber-loss-technical} subsumes the one of \cref{thm:huber-loss-gaussian-results}. Given $(\bm y, \bm X)$ as in \cref{eq:gaussian_model} it suffices to multiply the instance by an independent random matrix $\bm U$ corresponding to flip of signs and permutation of the entries, then add an independent Gaussian vector $\bm w \sim N(0, \Id_n)$.
The model then becomes
\begin{align*}
	\bm U \bm y = \bm U \bm X\betastar + \bm U \eta  +\bm w\,,
\end{align*}
which can be rewritten as 
\begin{align*}
	\bm y' = \bm X'\betastar +  \bm U\eta+\bm w\,.
\end{align*}
Here $\bm X'\in \R^{n\times d}$ has i.i.d entries $\bm X_\ij \sim N(0,1)$,  $\bm U =\bm S \bm P$ where $\bm P\in \R^{n\times n}$ is a  permutation matrix chosen u.a.r. among all permutation matrices, $\bm S\in \R^{n\times n}$ is a diagonal random matrix with i.i.d. Rademacher variables on the diagonal and $\bm U\eta\in \R^n$ is a symmetrically distributed vector independent of $\bm X'$ such that $\bbP \Set{\Abs{\bm \eta_i}\leq 1}\geq \alpha/2$. Moreover the entries of $\bm U \eta$ are conditionally independent given $\bm P$. 
At this point, we can relax our Gaussian design assumption and consider
\begin{equation}
	\bm y = X\betastar + \bm w +  \bm U \eta \,,
\end{equation}
which corresponds to the model considered in \cref{thm:huber-loss-technical}.
\end{remark}

\subsubsection{Relaxing well-spread assumptions}\label{sec:spread-assumption}
	It is natural to ask if under weaker assumptions on $X$ we may design an efficient algorithm that correctly recovers $\betastar$. 
	While it is likely that the assumptions in \cref{thm:huber-loss-informal} are not tight, \textit{some} requirements are needed if one hopes to design an estimator with bounded error.
	Indeed, suppose  \textit{there exists a vector $\betastar\in \R^d$ and  a set $\cS\subseteq [n]$ of  cardinality $o(1/\alpha)$ such that $\Norm{X_\cS \betastar}= \Norm{X \betastar}>0$}. Consider an instance of linear regression  $\bm y=X\betastar +\bm \eta$ with $\bm \eta$ as in \cref{thm:huber-loss-informal}. Then with probability $1-o(1)$ any non-zero row containing information about $\betastar$ will be corrupted by (possibly unbounded) noise. More concretely:

\begin{lemma}\label{lem:lower-bound}
	Let $\sigma>0$ be arbitrarily chosen. For large enough absolute constant $C>0$,
	let $X\in\R^{n}$ be an $\frac{n}{C\alpha}$-sparse deterministic vector and let $\bm \eta$ be an $n$-dimensional random vector with i.i.d coordinates sampled as
	\begin{align*}
		\bm \eta_i &=0\quad  \text{with probability }\alpha\\
		\bm \eta_i&\sim N(0, \sigma^2) \quad \text{otherwise.}
	\end{align*}
	Then for  every estimator $\hat{ \beta}:\R^n\rightarrow \R$ there exists $\betastar\in \R$ such that for $\bm y = X\betastar+\bm \eta$, with probability at least $\Omega(1)$,
	\begin{align*}
		\tfrac{1}{n}\Snorm{X\Paren{\betahat(\bm y)-\betastar}}\geq \Omega\Paren{\frac{\sigma^2}{n}}\,.
	\end{align*}
	\begin{proof}
		Let $\cC\subseteq [n]$ be the set of zero entries of $X$ and $\overline{\cC}$ its complement. Notice  that with probability $1-\Omega(1)$ over $\bm \eta$ the set $\bm \cS=\Set{i\in [n]\suchthat i \in \overline{\cC} \text{ and } \bm \eta_i =0}$ is empty. 
		Conditioning on this event $\cE$, for any estimator $\betahat:\R^n\rightarrow\R$ and $\bm\eta_{\cC}$ define the function $g_{\bm\eta_{\cC}}:\R^{n-\Card{\cC}}\rightarrow\R$ such that $g_{\bm\eta_{\cC}}(\bm y_{\overline{\cC}})=\betahat(\bm y)$.
		Taking distribution over $\bm \betastar$  from \cref{fact:linear-regression-lower-bound} (independent of $\bm \eta$), we get with probability $\Omega(1)$
		\begin{align*}
					\tfrac{1}{n} {\Snorm{X\Paren{\betahat(\bm y)-\betastar}}}
					=
					\tfrac{1}{n}{\Snorm{X
							\Paren{g_{\bm\eta_{\cC}}
								(\bm y_{{\overline{\cC}}})-\betastar}}}
					\ge
					\Omega\Paren{\frac{\sigma^2}{n}}\,.
		\end{align*}
		Hence for any $\betahat$ there exists $\beta^*$ with desired property.
	\end{proof}

\end{lemma}

Notice that the noise vector $\bm \eta$ satisfies the premises of \cref{thm:huber-loss-informal}. Furthermore, since $\sigma>0$ can be arbitrarily large, no estimator can obtain bounded error.

\subsection{On the noise assumptions}\label{sec:tightness-noise-assumptions}

\paragraph{On the scaling of noise}
	Recall our main regression model,
	\begin{align}
		\label{eq:general-model}
		\bm y = X\betastar+\bm \eta
	\end{align}
	where we observe (a realization of) the random vector \(\bm y\), the matrix \(X\in\R^{n\times d}\) is a known design, the vector \(\betastar\in \R^n\) is the unknown parameter of interest, and the noise vector \(\bm \eta\) has independent, symmetrically distributed coordinates with \(\alpha=\min_{i\in[n]}\Pr\{\abs{\bm \eta_i}\le 1\}\).
	
	A slight generalization of  \cref{eq:general-model} can be obtained if we allow $\bm \eta$ to have independent, symmetrically distributed coordinates with \(\alpha=\min_{i\in[n]}\Pr\{\abs{\bm \eta_i}\le \sigma\}\).
	This parameter $\sigma$  is closely related to the subgaussian parameter 
	$\sigma$ from \cite{SuggalaBR019}. 
	If we assume as in \cite{SuggalaBR019} that some (good enough) estimator of this parameter is given, we could then simply divide each $\bm y_i$ by this estimator and 
	obtain bounds comparable to those of \cref{thm:huber-loss-informal}.  For unknown $\sigma \ll 1$ better error bounds can be obtained if we decrease Huber loss parameter $h$ (for example, it  was shown in the Huber loss minimization analysis of \cite{tsakonas2014convergence}). It is not difficult to see that our analysis (applied to small enough $h$) also shows similar effect. 
	However, as was also mentioned in \cite{SuggalaBR019}, it is not known whether in general $\sigma$
	can be estimated using only $\bm y$ and $X$. 
	So for simplicity we assume that $\sigma =1$.
\subsubsection{Tightness of noise assumptions}
We provide here a brief discussion concerning our assumptions on the noise vector $\pmb \eta$. We argue that, without further assumptions on $X$, the assumptions on the noise in \cref{thm:huber-loss-technical} are tight (notice that such model is more general than the one considered in \cref{thm:huber-loss-informal}).  \cref{fact:tightness-assumption-median} and \cref{fact:tightness-assumption-independence} 
provide simple arguments that we cannot relax median zero and independence noise assumptions.

\begin{fact}[Tightness of Zero Median Assumption]\label{fact:tightness-assumption-median}
	Let $\alpha \in (0,1)$ be such that $\alpha n$ is integer, and let $0< \eps < 1/100$. 
	There exist $\beta,\beta' \in \R^d$ and $X\in \R^{n \times d}$ satisfying
	\begin{itemize}
		\item the  column span of $X$ is $(n/2,1/2)$-spread,
		\item $\norm{X\Paren{\beta-\beta'}}\geq 
		\frac{\eps}{10}\Norm{X}\geq\frac{\eps}{10}\cdot \sqrt{n}$, 
	\end{itemize}
	and there exist distributions $D$ and $D'$ over vectors in $\R^n$ such that if
    $\pmb z \sim D$ or $\pmb z \sim D'$, then:
	\begin{enumerate}
		\item $\pmb z_1,\ldots, \pmb z_n$ are mutually inependent,
		\item For all $i\in[n]$, 
		$\bbP \Paren{\pmb z_i\ge  0} \ge \bbP \Paren{\pmb z_i\leq 0} 
		\ge \Paren{1-\eps}\cdot \bbP \Paren{\pmb z_i\geq 0}$,
		\item  For all 
		$i \in [n]$, 
		there exists a density $p_i$ of $\pmb{z_i}$ such that $p_i(t) \ge 0.1$
		for all $t\in [-1,1]$,
	\end{enumerate}	 
	and for $\pmb \eta \sim D$ and $\pmb \eta' \sim D'$, random variables
$X\beta+\pmb\eta$ and
$X\beta'+\pmb\eta'$ have the same distribution.
	\begin{proof}
		It suffices to consider the one dimensional case. 
		Let $X\in \R^n$ be a vector with all entries equal to $1$, 
		let $\beta=1$ and $\beta'=1-\frac{\eps}{10}$.
		Then $\pmb \eta \sim N(0, \Id_n)$ and $\pmb \eta' = N(\mu, \Id_n)$ with 
		$\mu = \transpose{\Paren{\frac{\eps}{10},\ldots,\frac{\eps}{10}}}\in \R^n$
		satisfy the assumptions,
		and random variables 
		$X\beta+\pmb\eta$ and $X\beta'+\pmb\eta'$ have the same distribution.
	\end{proof}
\end{fact}

\begin{fact}[Tightness of Independence Assumption]\label{fact:tightness-assumption-independence}
	Let $\alpha \in (0,1)$ be such that $\alpha n$ is integer.
	There exist $\beta,\beta' \in \R^d$ and $X\in \R^{n \times d}$ satisfying
	\begin{itemize}
		\item the  column span of $X$ is $(n/2,1/2)$-spread,
		\item $\norm{X\Paren{\beta-\beta'}}\geq 
		\Norm{X}\geq \sqrt{n}$, 
	\end{itemize}
	and 
	there exists a distribution $D$ over vectors in $\R^n$ such that
	$\pmb \eta \sim D$  satisfies:
	\begin{enumerate}
		\item For all $i\in[n]$, $\bbP \Paren{\pmb \eta_i\leq 0} \ge \bbP \Paren{\pmb \eta_i\geq 0}$,
		\item  For all 
$i \in [n]$, 
there exists a density $p_i$ of $\pmb{z_i}$ such that $p_i(t) \ge 0.1$
for all $t\in [-1,1]$,
	\end{enumerate}	 
    and for some $\pmb\eta' \sim D$, 
    with probability $1/2$, $X\beta  + \pmb\eta  = X\beta' + \pmb\eta'$.
	\begin{proof}
		Again it suffices to consider the one dimensional case. 
		Let $X\in \R^n$ be a vector with all entries equal to $1$, $\beta = 0$, $\beta' = 1$.
		Let $\pmb\sigma\sim U\Set{-1,1}$ 
		and let 
		$\pmb v\in \R^n$ be a vector independent of $\pmb\sigma$ such that  
		for all $i\in[n]$, the entries of $\pmb v$ are iid $\pmb v_i = U[0,1]$.
		Then  $\pmb \eta=\pmb \sigma\cdot \pmb v$ and 
		$\pmb\eta' = -\pmb\sigma \Paren{1 - \pmb v}$
		satisfy the assumptions,
		and if $\pmb \sigma = 1$,
		$X\beta+\pmb\eta = X\beta' + \pmb\eta'$.
	\end{proof}
\end{fact}
Note that the assumptions in both facts do not contain $\pmb \cR$ as opposed to the assumptions of \cref{thm:huber-loss-technical}. One can take $\pmb\cR$  to be a random subset of $[n]$ of size $\alpha n$ independent of $\pmb\eta$ and $\pmb\eta'$.

\section{Computing the Huber-loss estimator in polynomial time}\label{sec:computing-huber-loss}
\Tnote{I will add some more prose}

In this section we show that in the settings of \cref{thm:huber-loss-informal}, we can compute the Huber-loss estimator efficiently. 
For a vector $v\in \mathbb{Q}^N$ we denote by $\mathtt{b}[v]$ its bit complexity.
For $r>0$ we denote by $\cB(0,r)$ the Euclidean ball of radius $r$ centered at $0$.
We consider the Huber loss function as in \cref{def:huber-loss}  and for simplicity we will assume $\transpose{X}X=n\Id_d$.

\begin{theorem}
	\label{thm:computing-huber-loss}
	Let $X\in \Q^{n\times d}$ be a matrix such that $\transpose{X}X=n\Id_d$ and $y \in \Q^n$ be a vector. 
	Let 
	\begin{align*}
		\mathtt{B} &:=\mathtt{b}[y] + \mathtt{b}[X]
	\end{align*}
	Let $f$ be a Huber loss function $f(\beta)=\sum_{i=1}^n \Phi\Paren{\Paren{X\beta - y}_i}$.
	Then there exists an algorithm that given $ X$, \( y\) and positive $\eps\in \Q$, 
	computes a vector $\overline{ \beta}\in \R^d$ such that
	\begin{align*}
		f(\overline{ \beta})\leq 
		\underset{\beta \in \R^{d}}{\inf}  f(\beta)+ \eps\,,
	\end{align*}
	in time 
	\begin{align*}
		\mathtt{B}^{O(1)}\cdot \ln(1/\eps)\,.
	\end{align*}
	%
	%
	%
\end{theorem}

As an immediate corollary, the theorem implies that for design matrix $X$ with orthogonal columns we can  compute an $\eps$-close approximation of the Huber loss estimator in time polynomial in the input size.

To prove  \cref{thm:computing-huber-loss} we will rely on the following standard result concerning the Ellipsoid algorithm. 

\begin{theorem}[See \cite{vishnoi2018algorithms}]\label{thm:ellipsoid-algorithm}
	There is an algorithm that, given 
	\begin{enumerate}
		\item a first-order oracle for a convex function $g: \R^d \rightarrow R$,
		\item a  separation oracle for a convex set $K \subseteq R^d$,
		\item numbers $r > 0$ and $R > 0$ such that  $\cB(0,r)\subset K \subset \cB(0,R)$, 
		\item  bounds $\ell, u$ such that  $\forall v \in K$, $\ell \leq g(v)\leq u$
		\item $\eps > 0$,
	\end{enumerate}
	outputs a point $\overline{x}\in K$ such that
	\begin{align*}
		g(\overline{x})\leq g(\hat{x})+\eps\,,
	\end{align*}
	where $\hat{x}$ is any minimizer of $g$ over $K$. The running time of the algorithm is
	\begin{align*}
		O \Paren{\Paren{d^2+T_K+T_g}\cdot d^2\cdot \log \Paren{\frac{R}{r}\cdot \frac{u-\ell}{\eps}}}\,,
	\end{align*}
	where $T_k, T_g$ are the running times for the separation oracle for $K$ and the first-order oracle for $g$ respectively.
\end{theorem}

We only need to apply \cref{thm:ellipsoid-algorithm} to the settings of \cref{thm:computing-huber-loss}.
Let $\hat{\beta}$ be a (global) minimizer of $f$.
Our first step is to show that $\norm{\hat{\beta}}$ is bounded by $\exp(O(\mathtt{B}))$.
\begin{lemma}
	\label{lem:bounded-norm-beta}
	Consider the settings of \cref{thm:computing-huber-loss}. Then $\Norm{\betahat}\leq 2^{5\mathtt{B}}$.
	Moreover, for any $\beta\in \cB(0, 2^{10\mathtt{B}})$
	\begin{align*}
		0\leq f(\beta)\leq 2^{12\mathtt{B}}\,.
	\end{align*}
	\begin{proof}
		Let $M=2^{\mathtt{B}}$.
		By definition $f (0)\leq 2\normo{y}+\frac{1}{2}\snorm{y}\leq M^4$.
		On the other hand for any $v\in \R^d$ with $\norm{v}\geq M^5$ we have 
		\begin{align*}
			f(v) &= \underset{i \in [n]}{\sum} \Phi \Paren{y_i- \iprod{X_i, v}}\\
			&\geq  \underset{i \in [n]}{\sum} \Phi \Paren{\iprod{X_i, v}} - \Snorm{y}-\Normo{y}\\
			&\geq M^5 - M^4\\
			&\geq M^4\,.
		\end{align*}
		It follows that $\norm{\betahat}\leq M^5$.
	 	For the second inequality note that for any $v\in \R^d$ with $\norm{v}\leq M^{10}$
	 	\begin{align*}
	 		f(v) &= \underset{i \in [n]}{\sum} \Phi \Paren{y_i- \iprod{X_i, v}}\\
	 		&\leq  2\underset{i \in [n]}{\sum} \Phi \Paren{\iprod{X_i, v}} + \Snorm{y}+2\Normo{y}\\
	 		&\leq M^{12}\,.
	 	\end{align*}
	\end{proof}
\end{lemma}

Next we state a simple fact about the Huber loss function and separation oracles, which proof we omit.  Recall the formula for the gradient of the Huber loss function. $\nabla f(\betastar)=\frac{1}{n}\underset{i=1}{\overset{n}{\sum}}\phi'\brac{\eta_i}\cdot x_i$ with $\Phi'[t]=\sign(t)\cdot \min \Set{\Abs{t},h}$.

\begin{fact}
	\label{fact:computing-first-order-oracle}
	Consider the settings of \cref{thm:computing-huber-loss}. 
	Let   $R= 2^{10\mathtt{B}}$.
	Then
	\begin{enumerate}
		\item there exists an algorithm that given $v\in \mathbb{Q}^d$, computes $\nabla f(v)$ and $f(v)$ in time 
		$ \mathtt{B}^{O(1)}\cdot \mathtt{b}^{O(1)}[v]$.
		\item there exists an algorithm that
		given $v\in \mathbb{Q}^d$  outputs
		\begin{itemize}
			\item $YES$ if $v\in \cB(0, R)$
			\item otherwise outputs a hyperplane $\Set{x \in\R^d\suchthat \iprod{a,x}=b}$ with $a,b \in \mathbb{Q}^d$ separating $v$ from $\cB(0,R)$,
		\end{itemize}  
		in time $\mathtt{B}^{O(1)}\cdot \mathtt{b}^{O(1)}[v]$.
	\end{enumerate}
\end{fact}

We are now ready to prove \cref{thm:computing-huber-loss}.

\begin{proof}[Proof of \cref{thm:computing-huber-loss}]
	Let $M=2^{\mathtt{B}}$.
	By \cref{lem:bounded-norm-beta}  it suffices to set $K=\cB(0, M^{10})$, $R=M^{10}$ and $r=R/2$.
	Then
	for any $v\in K$, $0\leq f(v)\leq M^{12}$.
	By \cref{fact:computing-first-order-oracle} $T_f+T_K\leq B^{O(1)}$.
	Thus, puttings things together and plugging \cref{thm:ellipsoid-algorithm} it follows that
	there exists an algorithm computing $\overline{\beta}$ with $f(\overline{ \beta})\leq f( \betahat)+\eps$ in time
	\[
		\mathtt{B}^{O(1)}\cdot \ln(1/\eps)
	\]
	for positive $\eps\in \Q$.
\end{proof}

\section{Consistent estimators in high-dimensional settings}

In this section we discuss the generalization of the notion of consistency to the case when the dimension $d$ (and the fraction of inliers $\alpha$) can depend on $n$.

\begin{definition}[Estimator]\label{def:estimator}
	Let $\set{d(n)}_{n=1}^{\infty}$ be a sequence of positive integers. 
    We call a function $\hat{\beta} :\bigcup_{n=1}^\infty \R^{n}\times \R^{n\times d(n)} \to \bigcup_{n=1}^\infty \R^{d(n)}$ an \emph{estimator}, 
    if for all $n \in \N$, $\hat{\beta}\Paren{\R^{n}\times \R^{n\times d(n)}} \subseteq \R^{d(n)}$ and the restriction of $\hat{\beta}$ to $\R^{n}\times \R^{n\times d(n)}$ is a Borel function.
\end{definition}

For example, Huber loss defined at $y\in \R^n$ and $X\in \R^{n\times d(n)}$ as
\[
\hat{\beta}\Paren{y,X} = \underset{\beta \in \R^{d(n)}}{\argmin} \sum_{i=1}^n \Phi\Paren{\Paren{X\beta - y}_i}
\] 
is an estimator (see \cref{lem:measurability-of-huber} for formal statement and the proof).

\begin{definition}[Consistent estimator]\label{def:consistent-estimator}
	Let $\set{d(n)}_{n=1}^{\infty}$ be a sequence of positive integers and let $\hat{\beta} :\bigcup_{n=1}^\infty \R^{n}\times \R^{n\times d(n)} \to \bigcup_{n=1}^n \R^{d(n)}$ be an estimator. 
	Let $\set{\bm X_n}_{n=1}^{\infty}$ 
	be a sequence of (possibly random) matrices
	and let $\set{\bm \eta_n}_{n=1}^{\infty}$ 
	be a sequence of (possibly random) vectors such that $\forall n\in \N$,  $\bm X_n$ has dimensions $n\times d(n)$  and $\bm \eta_n$ has dimension $n$.

	We say that estimator $\hat{\beta}$ is \emph{consistent} for $\set{\bm X_n}_{n=1}^{\infty}$  and $\set{\bm \eta_n}_{n=1}^{\infty}$ if there exists a sequence of positive numbers $\set{\eps_n}_{n=1}^n$ such that 
	$\lim_{n\to \infty} \eps_n = 0$ and for all $n\in \N$,
	\[
	\sup_{\beta^*\in \R^{d(n)}}\Pr\Paren{\tfrac 1n\norm{ \bm X_n{\hat\beta\Paren{\bm X_n \beta^* + \bm \eta_n\,, \bm X_n} -  \bm X_n\beta^*}}^2\ge\eps_n}\le \eps_n\,.
	\]
\end{definition}

\cref{thm:huber-loss-gaussian-results} implies that if sequences $d(n)$ and $\alpha(n)$ satisfy $d(n) / \alpha^2(n) \le o(n)$ and $d(n) \to \infty$, 
then Huber loss estimator is consistent for a sequence $\set{\bm X_n}_{n=1}^{\infty}$ of standard Gaussian matrices $\bm X_n \sim N(0,1)^{n\times d(n)}$  and every sequence of vectors $\set{\eta_n}_{n=1}^{\infty}$ such that each $\eta_n \in \R^n$ is independent of $\bm X_n$ and has at least $\alpha(n)\cdot  n$ entries of magnitude at most $1$.

Similarly, \cref{thm:huber-loss-informal} implies that if sequences $d(n)$ and $\alpha(n)$ satisfy $d(n) / \alpha^2(n) \le o(n)$ and $d(n) \to \infty$, then Huber loss estimator is consistent for each sequence 
$\set{ X_n}_{n=1}^{\infty}$ of matrices $X_n \in \R^{n\times d(n)}$ whose column span is $\Paren{\omega\paren{ d(n)/\alpha^2(n)}, \Omega(1)}$-spread and every sequence of $n$-dimensional random vectors $\set{\bm \eta_n}_{n=1}^{\infty}$ such that each $\bm \eta_n$ is independent of $X_n$ and has mutually independent, symmetrically distributed entries whose magnitude does not exceed $1$ with probability at least $\alpha(n)$.

Note that the algorithm from \cref{thm:median-algorithm-informal} requires some bound on $\norm{\beta^*}$. So formally we cannot say that the estimator that is computed by this algorithm is consistent. However, if in \cref{def:consistent-estimator} we replace supremum over $\R^{d(n)}$ by supremum over some ball in $\R^{d(n)}$ centered at zero (say, of radius $n^{100}$), then we can say that this estimator is consistent. More precisely, it is consistent (according to modified definition with $\sup$ over ball of radius $n^{100}$) for sequence $\set{\bm X_n}_{n=1}^{\infty}$ of standard Gaussian matrices $\bm X_n \sim N(0,1)^{n\times d(n)}$ and every sequence of vectors $\set{\eta_n}_{n=1}^{\infty}$ such that each $\eta_n \in \R^n$ is independent of $\bm X_n$ and has at least $\alpha(n)\cdot  n$ entries of magnitude at most $1$, if $n \gtrsim d(n)\log^2(d(n))/\alpha^2(n)$ and $d(n)\to \infty$.
\vspace{1em}

To show that Huber loss minimizer is an estimator, we need the following fact:
\begin{fact}\cite{jennrich1969}
	\label{fact:measurability-of-minimizer}
	For $d, N\in \N$, let $\Theta \subset \R^d$ be compact and let $\cM\subseteq \R^N$ be a Borel set. Let $f:\cM\times \Theta \to \R$ be a function such that for each $\theta\in \Theta$, $f(x,\theta)$ is a Borel function of $x$ and for each $x\in \cM$, $f(x,\theta)$ is a continuous function of $\theta$. 
	Then there exists a Borel function  $\hat{\theta}: \cM \to \Theta$ such that for all $x\in \cM$,
	\[
	f\paren{x,\hat{\theta}(x)} = \underset{\theta \in \Theta}{\min} f(x,\theta)\,.
	\]
\end{fact}

The following lemma shows that Huber loss minimizer is an estimator.
\begin{lemma}\label{lem:measurability-of-huber}
	Let $\set{d(n)}_{n=1}^{\infty}$ be a sequence of positive integers. 
	There exists an estimator $\hat{\beta}$ such that for each $n\in \N$ and for all
	$y\in \R^n$ and $X\in \R^{n\times d(n)}$,
	\[
	\sum_{i=1}^n \Phi\Paren{\Paren{X\hat{\beta}\Paren{y,X} - y}_i} = \underset{\beta \in \R^{d(n)}}{\min} \sum_{i=1}^n \Phi\Paren{\Paren{X\beta - y}_i}\,.
	\]
\end{lemma}

\begin{proof}
	For $i\in \N$ denote  
	\[
	\cM^i_n = \Set{\Paren{y,X} \in \R^{n}\times \R^{n\times d(n)} \suchthat \norm{y} \le 2^{i}\,, \sigma_{\min}(X) \ge 2^{-i}}\,,
	\]
	where $\sigma_{\min}(X)$ is the smallest positive singular value of $X$. Note that for all $(y,X) \in \cM^i_n$, there exists a minimizer of Huber loss function at $(y, X)$ in the ball
	$\Set{\beta \in \R^{d(n)} \suchthat \norm{\beta} \le 2^{2i}n^{10}}$.
	By \cref{fact:measurability-of-minimizer},  there exists a Borel measurable Huber loss minimizer $\hat{\beta}^i_n(y, X)$ on $\cM^i_n$. 
	
	Denote $\cM^0_n= \Set{\Paren{y,X} \in \R^{n}\times \R^{n\times d(n)} \suchthat X=0}$ and let $\hat{\beta}^0_n(y,X) = 0$ for all $(y,X) \in \cM^0_n$. Note that $\bigcup_{i=0}^{\infty}\cM^i_n = \R^{n}\times \R^{n\times d(n)}$. 
	For $(y,X)\in  \R^{n}\times \R^{n\times d(n)}$, define $\hat{\beta}_n(y, X) = \hat{\beta}^i_n(y, X)$,
	where $i$ is a minimal index such that $\Paren{y,X} \in \cM^i_n$.
	Then for each Borel set $\cB\subseteq \R^{d(n)}$, 
	\[
	\Paren{\hat{\beta}_n}^{-1}\Paren{\cB} = 
	\bigcup_{i=0}^{\infty}\Paren{\hat{\beta}_n}^{-1}\Paren{\cB}\cap\Paren{\cM^i_n\setminus \cM^{i-1}_n} =
	\bigcup_{i=0}^{\infty}\Paren{\hat{\beta}^i_n}^{-1}\Paren{\cB}\cap\Paren{\cM^i_n\setminus \cM^{i-1}_n}
	\,.
	\]
	Hence $\hat{\beta}_n$ is a Borel function. 
	Now for each $n\in \N$ and for all
	$y\in \R^n$ and $X\in \R^{n\times d(n)}$ define an estimator $\hat{\beta}(y, X)=\hat{\beta}_n(y, X)$.
\end{proof}

\section{Concentration of measure}\label{sec:missing_proofs}
This section contains some technical results needed for the proofs of \cref{thm:huber-loss-informal} and \cref{thm:median-algorithm-informal}.
We start by proving a concentration bound for the empirical median.

\begin{fact}[\cite{vershynin_2018}]\label{fact:epsilon-net-ball}
	Let $0 < \varepsilon < 1$. Let $\cB = \Set{v\in \R^n\suchthat \norm{v}\le 1}$.
	Then $\cB$ has an $\varepsilon$-net of size $\Paren{\frac{2+\eps}{\varepsilon}}^n$. That is,
	there exists a set $\cN_\varepsilon\subseteq \cB$ 
	of size at most $\Paren{\frac{2+\eps}{\varepsilon}}^n$ such that for any  vector 
	$u\in \cB$ there exists some $v\in \cN_\varepsilon$ such that $\norm{v-u} \le \varepsilon$.
\end{fact}

\begin{fact}[Chernoff's inequality, \cite{vershynin_2018}]\label{fact:chernoff}
	Let $\bm \zeta_1,\ldots, \bm \zeta_n$ 
	be independent Bernoulli random variables such that 
	$\Pr\Paren{\bm \zeta_i = 1} = \Pr\Paren{\bm\zeta_i = 0} = p$. 
	Then for every $\Delta > 0$,
	\[
	\Pr\Paren{\sum_{i=1}^n \bm\zeta_i \ge pn\Paren{1+ \Delta} } 
	\le 
	\Paren{ \frac{e^{-\Delta} }{ \Paren{1+\Delta}^{1+\Delta} } }^{pn}\,.
	\]
	and for every $\Delta \in (0,1)$,
	\[
	\Pr\Paren{\sum_{i=1}^n \bm\zeta_i \le pn\Paren{1- \Delta} } 
	\le 
	\Paren{ \frac{e^{-\Delta} }{ \Paren{1-\Delta}^{1-\Delta} } }^{pn}\,.
	\]
\end{fact}

\begin{fact}[Hoeffding's inequality, \cite{wainwright_2019}]\label{fact:hoeffding}
	Let $\bm z_1,\ldots, \bm z_n$
	be mutually independent random variables such that for each $i\in[n]$,
	$\bm z_i$ is supported on $\brac{-c_i, c_i}$ for some $c_i \ge 0$. 
	Then for all $t\ge 0$,
	\[
	\Pr\Paren{\Abs{\sum_{i=1}^n \Paren{\bm z_i - \E \bm z_i}} \ge t} 
	\le 2\exp\Paren{-\frac{t^2}{2\sum_{i=1}^n c_i^2}}\,.
	\]
\end{fact}
\begin{fact}[Bernstein's inequality
	\cite{wainwright_2019}]\label{fact:bernstein}
	Let $\bm z_1,\ldots, \bm z_n$
	be mutually independent random variables such that for each $i\in[n]$,
	$\bm z_i$ is supported on $\brac{-B, B}$ for some $B\ge 0$. 
	Then for all $t\ge 0$,
	\[
	\Pr\Paren{{\sum_{i=1}^n \Paren{\bm z_i - \E \bm z_i}} \ge t} 
	\le \exp\Paren{-\frac{t^2}{2\sum_{i=1}^n \E \bm z_i^2 + \frac{2Bt}{3}}}\,.
	\]
	\end{fact}

\begin{lemma}[Restate of \cref{lem:median-meta}]
	Let $\cS\subseteq [n]$ be a set of size $\gamma n$ and let $\bm z_1,\ldots, \bm z_n\in \R$ be mutually independent  random variables satisfying 
	\begin{enumerate}
		\item For all $i\in [n]$, $\bbP \Paren{\bm z_i \ge 0} = \bbP \Paren{\bm z_i \le 0}$.
		\item For some $\eps\geq 0$, for all $i\in \cS$,  $\bbP \Paren{\bm z_i\in \Brac{0,\eps}}=
		\bbP \Paren{\bm z_i \in \Brac{-\eps,0}}\geq q$.
	\end{enumerate}
	Then with probability at least $1-2\exp\Set{-\Omega\Paren{q^2\gamma^2 n}}$ the median $\bm{\hat{z}}$ satisfies
	\begin{align*}
	\Abs{\bm{\hat{z}}}\leq \eps\,.
	\end{align*}
	\begin{proof}
		Let $\bm \cZ=\Set{\bm z_1,\ldots,\bm z_n}$.
		Consider the following set:
		\begin{align*}
		\bm \cA&:=\Set{\bm z\in \bm \cZ \given \Abs{\bm z}\leq \eps}\,.
		\end{align*}
		Denote $\bm \cZ^{+} = \bm \cZ\cap  \R_{\ge 0}$, $\bm \cA^{+} = \bm \cA\cap  \R_{\ge 0}$, $\bm \cZ^{-} = \bm \cZ\cap  \R_{\le 0}$, $\bm \cA^{-} = \bm \cA\cap  \R_{\le 0}$.
		Applying Chernoff bound for $\gamma_1,\gamma_2, \gamma_3 \in (0,1)$,
		\begin{align*}
		\bbP \Paren{\Card{\bm \cZ^+}\leq \Paren{\frac{1}{2}-\gamma_1}n} & \leq \exp \Set{-\frac{\gamma_1^2\cdot n}{10}}\,,\\
		\bbP \Paren{\Card{\bm \cA}\leq (1-\gamma_2)\cdot  q\cdot  \Card{\cS}}&\leq \exp \Set{-\frac{\gamma_2^2\cdot q\cdot \Card{\cS}}{10}}\,,\\
		\bbP \Paren{\Card{\bm \cA^+}\leq \Paren{\frac{1}{2}-\gamma_3}\cdot \card{\bm \cA}
		\given \card{\bm \cA}}
		&\leq \exp \Set{-\frac{\gamma_3^2\cdot \Card{\bm \cA}}{10}}\,.
		\end{align*}
		Similar bounds hold for $\bm \cZ^-,\bm \cA^-$.
		
		Now, the median is in $\bm \cA$ if $\Card{\bm \cZ^-}+\Card{\bm \cA^+}\geq n/2$ and $\Card{\bm \cZ^+}+\Card{\bm \cA^-}\geq n/2$.
		It is enough to prove one of the two inequalities, the proof for the other is analogous. A union bound then concludes the proof.
		
		So for $\gamma_2=\gamma_3=\frac{1}{4}$, with probability at least
		 $1-\exp \Set{-\frac{\gamma_1^2\cdot n}{10}}-
		 2\exp \Set{-\Omega\Paren{q\cdot  \Card{\cS}}}$, 
		\begin{align*}
		\Card{\bm \cZ^-}+\Card{\bm \cA^+}
		\geq \Paren{\frac{1}{2}-\gamma_1}n+\frac{q\cdot  \Card{\cS}}{10}\,.
		\end{align*}
		it follows that $\Card{\bm \cZ^-}+\Card{\bm \cA^+}\geq n/2$ for
		\begin{align*}
		\gamma_1 \leq \frac{q \cdot \Card{\cS}}{10n}.
		\end{align*}
	\end{proof}
\end{lemma}

\begin{lemma}\label{lem:uniform-bound}
Let $V$ be an $m$-dimensional vector subspace of $\R^n$. 
Let $\cB \subseteq \Set{v\in V\suchthat \norm{v}\le R}$ for some $R \ge 1$ .

Let $g: \R^2 \to \R$ be a function such that for all $y\in \R$ and $\abs{x} \le R$,
$\abs{g(x,y)} \le C\abs{x}$ for some $C \ge 1$ and for any $|\Delta x|\le 1$, 
$\abs{g(x + \Delta x,y) - g(x,y)} \le K\abs{\Delta x}$ for some $K \ge 1$.

Let $\bm w\in \R^n$ be a random vector such that $\bm w_1,\ldots, \bm w_n$ are mutually independent.
For any $N\ge n$, with probability at least $1-N^{-m}$, for all $v\in \cB$,
\[
\Abs{\sum_{i=1}^n\Paren{g(v_i,\bm w_i) - \E_{\bm w} g(v_i,\bm w_i)}}  \le 10C\sqrt{m\ln\Paren{RKN}} \cdot \norm{v} + 1/N\,.
\]
\begin{proof}
Consider some $v\in \R^n$. Since $\abs{g_i(v_i, \bm w_i)} \le C\abs{v_i}$,
 by Hoeffding's inequality,
\[
\Abs{\sum_{i=1}^n\Paren{g(v_i, \bm w_i) - \E_{\bm w} g(v_i,\bm w_i)}}  \le  \tau C\norm{v}
\]
with probability $1-2\exp\Paren{-\tau^2/2}$.

Let $N\ge n$ and $\varepsilon = \frac{1}{2KnN}$.  Denote by $\cN_\varepsilon$ some $\varepsilon$-net in $\cB$ such that $\card{\cN_\varepsilon} \le \Paren{6\frac{R}{\varepsilon}}^m$. 
By union bound, for any $v\in \cN_\varepsilon$, 
\[
\Abs{\sum_{i=1}^n\Paren{g(v_i, \bm w_i) - \E_{\bm w} g(v_i, \bm w_i)}}  
\le  10C \sqrt{m\ln\Paren{RKN}}\cdot \norm{v} 
\]
with probability at least $1-N^{-m}$.

Consider arbitrary $\Delta v \in V$ such that $\norm{\Delta v} \le \varepsilon$. 
For any $v\in \cN_\varepsilon$ and $w \in \R^n$,
\[
\Abs{\sum_{i=1}^n\Paren{g(v_i + \Delta v_i,  w_i) - g(v_i, w_i)}} 
\le \sum_{i=1}^n\Abs{g(v_i + \Delta v_i,  w_i) - g(v_i,  w_i)}
\le K \sum_{i=1}^n \abs{\Delta v_i} \le \frac{1}{2N}\,.
\]
Hence 
\[
\Abs{\E_{\bm w} \sum_{i=1}^n\Paren{g(v_i + \Delta v_i, \bm w_i) - g(v_i, \bm w_i)}} 
\le \E_{\bm w} \Abs{\sum_{i=1}^n\Paren{g(v_i + \Delta v_i, \bm w_i) - g(v_i, \bm w_i)}} 
\le \frac{1}{2N}\,,
\] 
and
\[
\Abs{\sum_{i=1}^n\Paren{g(v_i, \bm w_i) - \E_{\bm w} g(v_i, \bm w_i)} - 
	\sum_{i=1}^n\Paren{g(v_i + \Delta v_i, w_i) - \E_{\bm w} g(\Delta v_i, \bm w_i)}} \le 1/N\,.
\]

Therefore, with probability $1-N^{-m}$, for any $v\in \cB$,
\[
\Abs{\sum_{i=1}^n\Paren{g(v_i, \bm w_i) - \E_{\bm w} g(v_i,\bm w_i)}}  
\le 10C\sqrt{m\ln\Paren{RKN}} \cdot \norm{v} + 1/N\,.
\]
\end{proof}
\end{lemma}

\begin{lemma}\label{lem:random-set-equivalence}
	Let $\bm \zeta_1,\ldots,\bm \zeta_n$ be i.i.d. Bernoulli random variables such that 
	$\Pr\Paren{\bm \zeta_i = 1} = 1 - \Pr\Paren{\bm \zeta_i = 0} = m/n$ for some  integer  $m\le n$. 
	Denote $\bm\cS_1 = \Set{i\in [n]\suchthat \bm \zeta_i = 1}$. 
	Let $\bm \cS_2\subseteq [n]$ be a random set chosen uniformly from all subsets of $[n]$ of size exactly $m$.
	
	Let $P$ be an arbitrary property of subsets of $[n]$.
	If $\bm \cS_1$ satisfies $P$ with  probability at least $1-\varepsilon$ 
	(for some $0 \le \varepsilon \le 1$), 
	then $\bm \cS_2$ satisfies $P$ with probability at least $1-2\sqrt{n}\varepsilon$.
	\begin{proof}
		If $m=0$ or $m=n$, then $\bm \cS_1=\bm \cS_2$ with probability $1$. So it is enough to consider the case $0<m<n$.
		By Stirling's approximation, for any integer $k\ge 1$,
		\[
		\sqrt{2\pi k} \cdot\frac{k^k}{e^k}\le k! \le \sqrt{2\pi k} \cdot\frac{k^k}{e^{k-1/{(12k)}}} 
		\le 1.1\cdot\sqrt{2\pi k} \cdot\frac{k^k}{e^k}\,.
		\]
		Hence
		\[ 
		\Pr\Paren{\Card{\bm \cS_1} = m}=
		\binom{n}{m}\Paren{\frac{m}{n}}^{m}\Paren{\frac{n-m}{n}}^{n-m} \ge 
		\frac{\sqrt{n}}{1.1^2\cdot \sqrt{2\pi m(n-m)}}\ge 
		\frac{1}{2\sqrt{n}}\,.
		\]
		Therefore,
		\[
		\Pr\Paren{\bm\cS_2\notin P}=
		\Pr\Paren{\bm\cS_1\notin P \suchthat \Card{\bm \cS_1} = m}
		\le
	    \frac{\Pr\Paren{\bm\cS_1\notin P}}{\Pr\Paren{\Card{\bm \cS_1} = m}}\le
	    2\sqrt{n}\varepsilon\,.
		\]
	\end{proof}
\end{lemma}

\begin{fact}[Covariance estimation of Gaussian vectors, \cite{wainwright_2019}]\label{fact:gaussian-sample-covariance}
	Let $\bm x_1,\ldots,\bm x_n\in\R^d$ be iid $\bm x_i\sim N(0,\Sigma)$ for some positive definite $\Sigma\in\R^d$. Then, with probability at least $1-\delta$,
	\[
	\Norm{\frac{1}{n}\sum_{i=1}^n \bm x_i \transpose{\bm x_i} - \Sigma} 
	\le O\Paren{\sqrt{\frac{d + \log\Paren{1/\delta}}{n}} + \frac{d + \log\Paren{1/\delta}}{n}}
	\cdot \Norm{\Sigma}\,.
	\]
\end{fact}

\begin{lemma}\label{lem:gaussian-sample-covariance-sparse}
	Let $\bm x_1,\ldots,\bm x_n\in\R^d$ be iid $\bm x_i\sim N(0,\Sigma)$ 
	for some positive definite $\Sigma\in\R^d$. 
	Let $k\in[d]$.
	Then, with probability at least $1-\eps$, for any $k$-sparse unit vector $v\in \R^d$,
	\[
	\transpose{v}\Paren{\frac{1}{n}\sum_{i=1}^n \bm x_i \transpose{\bm x_i} - \Sigma}v
	\le O\Paren{\sqrt{\frac{k\log d + \log\Paren{1/\eps}}{n}} + \frac{k\log d + \log\Paren{1/\eps}}{n}}
	\cdot \Norm{\Sigma}\,.
	\]
\begin{proof}
	If $d=1$, $1-d^{-k} =0$ and the statement is true, so assume $d>1$.
	Consider some set $\cS\subseteq [d]$ of size at most $k$.
	By \cref{fact:gaussian-sample-covariance},  with probability at least $1-\delta$,
		for any $k$-sparse unit vector $v$ with support $\cS$,
	\[
	\transpose{v}\Paren{\frac{1}{n}\sum_{i=1}^n \bm x_i \transpose{\bm x_i} - \Sigma}v
	\le O\Paren{\sqrt{\frac{k + \log\Paren{1/\delta}}{n}} + \frac{k + \log\Paren{1/\delta}}{n}}
\cdot \Norm{\Sigma}\,.
\]
	Since there are at most $\exp\Paren{2k\ln d}$ subsets of $[d]$ of size at most $k$, the lemma follows from union bound with $\delta = \eps\exp\Paren{-3k\ln d}$.
\end{proof}
\end{lemma}

\begin{fact}[Chi-squared tail bounds, \cite{laurent2000}]\label{fact:chi-squared-tail-bounds}
	Let $X\sim \chi^2_m$ (that is, a squared norm of  standard $m$-dimensional Gaussian vector). Then for all $x>0$
 	\begin{align*}
	\bbP \Paren{X-m \geq 2x+2\sqrt{mx}} &\leq e^{-x}\\
	\bbP \Paren{m-X\geq 2\sqrt{xm}}&\leq e^{-x}
	\end{align*}
\end{fact}

\begin{fact}[Singular values of Gaussian matrix, \cite{wainwright_2019}]\label{fact:singular-values-gaussian}
	Let $W\sim N(0,1)^{n \times d}$, and assume $n\ge d$. Then for each $t\ge0$
	\[
	\Pr\Paren{\sigma_{\max}(W)\ge \sqrt{n}  + \sqrt{d} + \sqrt{t}} \le  \exp\Paren{-t/2}
	\]
	and
	\[
	\Pr\Paren{\sigma_{\min}(W)\le \sqrt{n}  - \sqrt{d} - \sqrt{t}} \le  \exp\Paren{-t/2}\,,
	\]
	where $\sigma_{\max}(W)$ and $\sigma_{\min}(W)$ are the largest and the smallest singular values of $W$.
\end{fact}

\begin{fact}[$k$-sparse norm of a Gaussian matrix]\label{fact:k-sparse-norm-gaussian}
	Let $W\sim N(0,1)^{n\times d}$ be a Gaussian matrix. Let $1\le k \le n$. Then for every $\delta > 0$
	with probability at least $1-\delta$,
	\[
	\max_{\substack{u\in\R^d\\ \norm{u}=1}}\;\;
	\max_{\substack{\text{$k$-sparse }v\in\R^n\\ \norm{v}=1}} \transpose{v}Wu \le
	\sqrt{d} + \sqrt{k} + \sqrt{2k\ln\Paren{\frac{en}{k}}} + \sqrt{2\ln(1/\delta)}
	\,.
	\]
	\begin{proof}
		Let $v$ be some $k$-sparse unit vector that maximizes the value, and let $S(v)$ be the set of nonzero coordinates of $v$.
		Consider some fixed (independend of $W$) unit $k$-sparse vector $x\in \R^n$ and the set $S(x)$ of nonzero coordinates of $x$. If we remove from $W$ all the rows with indices not from $S(x)$, we get an
		$k\times d$ Gaussian matrix $W_{S(x)}$. By \cref{fact:singular-values-gaussian}, norm of this matrix is bounded by $\sqrt{d}+\sqrt{k} + \sqrt{t}$
		with probability at least $\exp\Paren{-t/2}$. Number of all subsets $S\subseteq [n]$ of size $k$ is $\binom{n}{k}\le \Paren{\frac{en}{k}}^k$.
		By union bound, the probability that the
		norm of $W_{S(v)}$ is greater than $\sqrt{d}+\sqrt{k} + \sqrt{t}$ is at most
		\[
		\binom{n}{k}\cdot \exp\Paren{-t/2} \le \exp\Paren{k\ln\Paren{en/k} - t/2}\,.
		\]
		Taking $t = 2{k\ln\Paren{en/k} + 2\log(1/\delta)}$, we get the desired bound.
	\end{proof}
\end{fact}

\section{Spreadness notions of subspaces}

\begin{lemma}
  Let \(v\in\R^n\) be a vector with \(\tfrac 1 n \norm{v}^2=1\).
  Suppose
  \[
    \tfrac 1 n\sum_{i=1}^n \bracbb{v_i^2\le 1/\delta}\cdot v_i^2 \ge \kappa\,.
  \]
  Then, \(\tfrac 1 n\norm{v_S}^2 \ge \kappa/2\) for every subset \(S\subseteq [n]\) with \(\card{S}\ge (1-\delta\kappa/2)n\).
\end{lemma}

The above lemma is tight in the sense that there are vectors \(v\) that satisfy the premise but have \(\norm{v_S}=0\) for a subset \(S\subseteq [n]\) of size \(\card{S}=(1-\delta k)n\).

\begin{proof}
  Let \(T\subseteq [n]\) consist of the \(\delta \kappa/2\cdot n\) entries of \(v\).
  Let \(w_i=\bracbb{v_i^2\le 1/\delta}\cdot v_i^2\).
  Then, \(\tfrac 1 n \norm{w_T}^2\le 1/\delta \cdot \delta\kappa/2\le \kappa/2\).
  Thus, \(\tfrac 1 n \norm{v_S}^2 \ge \tfrac 1 n \norm{w}^2 - \tfrac 1n\norm{w_T}^2\ge \kappa/2\).
\end{proof}

\begin{lemma}
  Let \(v\in\R^n\) be a vector with \(\tfrac 1 n \norm{v}^2=1\).
  Suppose \(\tfrac 1 n\norm{v_S}^2 \ge \kappa\) for every subset \(S\subseteq [n]\) with \(\card{S}\ge (1-\delta)n\).
  Then
  \[
    \tfrac 1 n\sum_{i=1}^n \bracbb{v_i^2\le 1/\delta}\cdot v_i^2 \ge \kappa\,.
  \]
\end{lemma}

\begin{proof}
  The number of entries satisfying \(v_i^2 \le 1/\delta\) is at least \((1-\delta)n\).
\end{proof}

\begin{lemma}\label{lem:min-eigenvalue-implies-spread}
	Let $V\subseteq \R^n$ be a vector subspace. 
	Assume that for some $\rho, R\in (0,1)$, 
	for all $v\in V$,
	\[
	\sum_{i=1}^n \Bracbb{v_i^2 \le \tfrac 1 {R^2n}\norm{v}^2} \cdot v_i^2
	\ge \rho^2 ||v||^2\,.
	\]
	Then $V$ is 
	$\Paren{\frac{\rho^2}{4}R^2n, \frac{\rho}{2}}$-spread.

	\begin{proof}
		Let $m = \frac{\rho^2}{4}R^2n$.
		For vector $v\in V$ with $||v||^2=m$, 
		the set $M = \{i\in [n] \;|\;  v_i^2 > 1\}$ has size at most $m$, 
		so its complement is a disjoint union of three sets: 
		the set $S$ of $n- m$ smallest entries of $v$, 
		the set $M'\subseteq [n]\setminus S$ of entries  of magnitude $\le \rho /2$, 
		and the set of entries 
		$M''\subseteq [n]\setminus \Paren{S\cup M'}$  
		of magnitude between $\rho /2$ and $1$. 
		Note that since $|M'|\le m$, 
		$\sum_{i\in M'} v_i^2 \le  \frac{1}{4}\rho^2 m = 
		\frac{1}{4}\rho^2 ||v||^2 $. 
		
		Now consider $w = \frac{2}{\rho}v$. 
		The set $N$ of  entries of $w$ of magnitude at most one 
		is a subset of	$S\cup M'$. 
		By our assumption, 
		$\sum_{w_i^2\le 1} w_i^2\ge \rho^2||w||^2$. 
		Hence 
		\[
		\sum_{i\in S} w_i^2 
		\ge \sum_{w_i^2\le 1} w_i^2 - \sum_{i\in M'} v_i^2 
		\ge \frac{3}{4} \rho^2 ||w||^2
		\]
		Since this inequality is scale invariant, 
		$V$ is  $(m, \sqrt{3/4}\cdot\rho)$-spread.
	\end{proof}
\end{lemma}

\begin{fact}\cite{DBLP:journals/combinatorica/GuruswamiLR10}
	Let $n\in \N$ and let $V$ be a vector subspace of $\R^n$. Define
	$\Delta(V) = \underset{\substack{v\in V \\ \norm{v}=1}}{\sup}\sqrt{n}/\norm{v}_1$. Then
	\begin{enumerate}
		\item If $V$ is
		$\Paren{m,\rho}$-spread, then $\Delta(V)\le \frac{1}{\rho^2}\sqrt{n/m}$. 
		\item $V$ is $\Paren{\frac{n}{2\Delta(V)^2}, \frac{1}{4\Delta(V)}}$-spread.
	\end{enumerate}
\end{fact}

The lemma below relates $\Normo{\cdot}$ and $\Norm{\cdot}$ of vectors satisfying specific sparsity constraints.

\begin{lemma}\label{lem:l1vsl2}
	Let $m\in[n]$, $\cA\subset[n]$ , $\gamma_1 > 0$ and $\gamma_2 > 0$.
	Let $v\in \R^n$ be a vector such that
	\[
	\sum_{i\in \cA} v_i^2 \le \gamma_1^2\snorm{v}\,.
	\] 
	and for any set $\cM\subset[n]$ of size $m$,
	\[ 
	\sum_{i\in \cM}v_i^2 \le \gamma_2^2\snorm{v}\,.
	\]
	Then
	\[
	\sum_{i\in [n]\setminus \cA} \abs{v_i} \ge 
	\frac{1 - \gamma_1^2 - \gamma_2^2}{\gamma_2} \sqrt{m}\, \norm{v}\,.
	\]
	\begin{proof}
		Let $\cM$ be the set of $m$ largest coordinates of $v$ (by absolute value). Since the inequality 
		$\sum_{i\in [n]\setminus S} \abs{v_i} \ge 	
		\frac{1-\gamma_1^2 - \gamma_2^2}{\gamma_2} \sqrt{m}\, \norm{v}$ is scale invariant,
		assume without loss of generality that for all $i\in M$, $\abs{v_i} \ge 1$ and 
		for all $i\in [n]\setminus \cM$, $\abs{v_i} \le 1$.
		Then
		\[
		\snorm{v} \le \sum_{i\in \cM} v_i^2 + \sum_{i\in \cA} v_i^2 + 
		\sum_{i\in [n]\setminus \Paren{\cA\cup \cM}} v_i^2 \le  
		\Paren{\gamma_2^2+\gamma_1^2}\snorm{v}  + \sum_{i\in [n]\setminus \Paren{\cA\cup \cM}} v_i^2  \,.
		\]
		hence
		\[
		\Paren{1-\gamma_2^2 - \gamma_1^2}\snorm{v} \le \sum_{i\in [n]\setminus \Paren{\cA\cup \cM}} v_i^2 \le  
		\sum_{i\in [n]\setminus \Paren{\cA\cup \cM}} \abs{v_i}
		\le \sum_{i\in [n]\setminus \cA} \abs{v_i} \,.
		\]
		Note that
		\[
		\Paren{1-\gamma^2 - \gamma_1^2}\snorm{v} \ge 
		\Paren{\frac{1-\gamma_2^2 - \gamma_1^2}{\gamma_2^2}} \sum_{i\in \cM} v_i^2 
		\ge \Paren{\frac{1-\gamma_2^2 - \gamma_1^2}{\gamma_2^2}}  m\,.
		\]
		Therefore,
		\[
		\Paren{\sum_{i\in [n]\setminus \cA} \abs{v_i}}^2 \ge  
		\frac{\Paren{1-\gamma_2^2 - \gamma_1^2}^2}{\gamma_2^2} \cdot m\snorm{v}\,.
		\]
	\end{proof}
\end{lemma}

\begin{lemma}\label{lem:norm-of-log-spread-vectors}
	Suppose that vector $v \in \R^n$ satisfies the following property: for any $\cS\subseteq [n]$,
	\begin{equation}\label{eq:log-spread-vectors}
	\sum_{i\in \cS} \abs{v_i} \le \Card{\cS} \cdot \ln\Paren{\frac{en}{\Card{\cS}}}\,.
	\end{equation}
	Then 
	\[
	\norm{v} \le \sqrt{6 n}\,.
	\]
	\begin{proof}
		Note that the set of vectors which satisfy \cref{eq:log-spread-vectors} is compact. 
		Hence there exists a vector $v$ in this set with maximal $\norm{v}$.
		Without loss of generality we can assume that the entries of $v$ are nonnegative and sorted in descending order. Then for any $m\in[n]$,
		\[
		\sum_{i=1}^m {v_i} \le m \ln\Paren{\frac{en}{m}}\,.
		\] 
		Assume that for some $m$ the corresponding inequality is strict. Let's increase the last term ${v_m}$ by small enough $\eps > 0$. If there are no nonzero $v_{m'}$ for $m'>m$, all inequalities are still satisfied and $\norm{v}$ becomes larger, which contradicts our choice of $v$. So there exists the smallest $m'>m$ such that $v_{m'} > 0$, and after decreasing $v_{m'}$ by $\eps < v_{m'}$ all inequalities are still satisfied.
		$\norm{v}$ increases after this operation:
		\[
		\Paren{v_m+\eps}^2 + \Paren{v_{m'}- \eps}^2 = v_m^2 + v_{m'}^2 + 2\eps\Paren{v_m-v_{m'}} + \eps^2 
		> v_m^2 + v_{m'}^2\,.
		\] 
		Therefore, there are no strict inequalities. Hence  $v_1 = \ln(en)$ and for all $m > 1$,
		\[
		{v_m} =  m \ln\Paren{\frac{en}{m}} - (m-1)\ln\Paren{\frac{en}{m-1}} = m\ln(1-1/m) + \ln\Paren{\frac{en}{m-1}}\le \ln\Paren{\frac{en}{m-1}}\,.
		\]
		Since for any decreasing function $f:[1,n]\to \R$, $\sum_{j=2}^n f(j) \le \int_1^n f(x)dx$,
		\[
		\norm{v}^2\le \ln^2(en) + \sum_{j=1}^{n-1}\ln^2\Paren{\frac{en}{j}}\le 
		2\ln^2(en) + \int_{1}^{n} \ln^2\Paren{\frac{en}{x}} dx\,.
		\]
		Note that
		\[
		\int \ln^2\Paren{\frac{en}{x}} dx = 2x + 2x \ln\Paren{\frac{en}{x}} + x\ln^2\Paren{\frac{en}{x}}\,.
		\]
		Hence
		\[
		\int_{1}^{n} \ln^2\Paren{\frac{en}{x}} dx = 5n - \ln^2(en) - 2\ln(en) - 2\le 5n - \ln^2(en)\,,
		\]
		and we get the desired bound.
	\end{proof}
\end{lemma}

\begin{lemma}\label{lem:iprod-of-log-spread-vectors}
	Suppose that vectors $v \in \R^n$ and $w\in \R^n$ satisfy the following properties: for some $t_1\ge0$ and $t_2\ge 0$, for any $\cS\subseteq [n]$,
	\begin{equation}
	\sum_{i\in \cS} \abs{v_i} \le t_1 + \Card{\cS} \cdot \ln\Paren{\frac{en}{\Card{\cS}}}\,.
	\end{equation}
	and 
	\begin{equation}
	\sum_{i\in \cS} \abs{w_i} \le t_2 + \Card{\cS} \cdot \ln\Paren{\frac{en}{\Card{\cS}}}\,.
	\end{equation}
	Then 
	\[
	\abs{\iprod{v, w}} \le t_1t_2 + \Paren{t_1+t_2}\ln(en) + 6n\,.
	\]
	\begin{proof}
		Note that the set of pairs of vectors which satisfy 
		these properties is compact. 
		Hence there exist vectors $v$, $w$ that satisfy these properties such that $\abs{\iprod{v, w}}$ is maximal.
		Without loss of generality we can assume that the entries of $v$ and $w$ are nonnegative and sorted in descending order. Moreover, if some entry $v_i$ of $v$ is zero, we can increase $\abs{\iprod{v, w}}$ by assigning some small positive value to it without violating  conditions on $v$ (if  $w_i=0$, we can also assign some positive value to it without violating  conditions on $w$). Hence we can assume that all entries of $v$ and $w$ are strictly positive. 
		
		Now assume that for some $m$ the corresponding inequality for $v$ with the set $[m]$ is strict. Let's increase ${v_m}$ by small enough $\eps > 0$ and decrease $v_{m+1}$ by $\eps$.
		This operation does not decrease $\abs{\iprod{v, w}}$:
		\begin{equation}\label{eq:move-mass-to-large-entries}
		\Paren{v_m+\eps}w_m + \Paren{v_{m-1}-\eps}w_{m-1} = v_mw_m + v_{m-1}w_{m-1} + \eps\Paren{w_m-w_{m-1}} \ge 0
		\end{equation}
		Moreover, if $w_m = w_{m-1}$, the inequality for $w$ with a set $[m]$ is strict, so by adding $\eps$ to $w_m$ and subtracting $\eps$ from $w_{m-1}$  we can make  $w_m$ and $w_{m-1}$ different without violating constraints on $w$ and without decreasing $\abs{\iprod{v, w}}$. 
		Hence without loss of generality we can assume that all $v_i$ are different from each other and all $w_i$ are different from each other. Now, by \cref{eq:move-mass-to-large-entries}, there are no strict inequalities (otherwise there would be a contradiction).
		Hence  $v_1 = t_1 + \ln(en)$, $w_1 = t_2 + \ln(en)$ and for all $m > 1$,
		\[
		{v_m} =  w_m = m \ln\Paren{\frac{en}{m}} - (m-1)\ln\Paren{\frac{en}{m-1}} = m\ln(1-1/m) + \ln\Paren{\frac{en}{m-1}}\le \ln\Paren{\frac{en}{m-1}}\,.
		\]
		Since $v-t_1e_1$ satisfies conditions of \cref{lem:norm-of-log-spread-vectors},
		\[
		\abs{\iprod{v,w}}\le v_1w_1 +\norm{v-t_1e_1}^2 - \ln^2(en) \le t_1t_2 + \Paren{t_1+t_2}\ln(en) + 6n \,.
		\]
	\end{proof}
\end{lemma}

\end{document}